\documentclass[english,11pt, letter]{article}
\usepackage[T1]{fontenc}
\usepackage[latin9]{inputenc}
\usepackage{geometry}

\usepackage{amsmath}
\usepackage{picinpar}
\usepackage{multicol}
\usepackage{amsthm}
\usepackage{amssymb}
\usepackage{algorithm}
\usepackage{algorithmic}
\newlength\myindent
\usepackage{setspace}
\usepackage{amsmath,amsfonts,bm,amsthm}
\usepackage{xfrac}

\makeatletter
\numberwithin{equation}{section}
\numberwithin{figure}{section}
\theoremstyle{plain}
\newtheorem{thm}{\protect\theoremname}
\theoremstyle{definition}
\newtheorem{defn}[thm]{\protect\definitionname}
\theoremstyle{plain}
\newtheorem{lem}[thm]{\protect\lemmaname}
\theoremstyle{plain}

\pdfoutput=1

\pdfoutput=1

\usepackage{amsmath, amsthm, amssymb, amsfonts,bbold}
\usepackage[basic]{complexity}
\usepackage[pdftex,bookmarks,colorlinks]{hyperref}
\usepackage{enumitem}

\DeclareMathOperator*{\argmaxTex}{arg\,max}
\DeclareMathOperator*{\argminTex}{arg\,min}
\DeclareMathOperator*{\signTex}{sign}
\DeclareMathOperator*{\rankTex}{rank}
\DeclareMathOperator*{\diagTex}{diag}
\DeclareMathOperator*{\imTex}{im}

\hypersetup{colorlinks=true,citecolor=red}

\renewcommand{\varepsilon}{\epsilon}

\usepackage{thmtools}
\usepackage{thm-restate}
\usepackage[usenames]{color}

\usepackage[lined,boxed,ruled,norelsize,algo2e]{algorithm2e}
\usepackage{tikz}
\usetikzlibrary{plotmarks}
\usepackage[small,compact]{titlesec}
\usepackage{setspace}
\usepackage{subfigure}

\usepackage{tcolorbox}

\makeatother
\title{Accelerated Gradient Algorithms with Adaptive Subspace Search for Instance-Faster Optimization}
\usepackage{babel}
\providecommand{\corollaryname}{Corollary}
\providecommand{\definitionname}{Definition}
\providecommand{\lemmaname}{Lemma}
\providecommand{\theoremname}{Theorem}

\author{%
  Yuanshi Liu\\
    Peking University\\
   \texttt{liu$\_$yuanshi@pku.edu.cn}
   \and 
   Hanzhen Zhao\\
   Peking University\\
   \texttt{hzzhao@pku.edu.cn}
   \and Yang Xu\\
   Peking University\\
   \texttt{xuyang1014@pku.edu.cn}
   \and Pengyun Yue\\
   Peking University\\
   \texttt{yuepy@pku.edu.cn}
   \and Cong Fang\\
   Peking University\\
   \texttt{fangcong@pku.edu.cn}
}

\date{}
\begin{document}

\theoremstyle{plain}
\newtheorem{theorem}{Theorem}[section]
\newtheorem{proposition}[theorem]{Proposition}
\newtheorem{lemma}[theorem]{Lemma}
\newtheorem{corollary}[theorem]{Corollary}
\newtheorem{definition}[theorem]{Definition}
\newtheorem{assumption}[theorem]{Assumption}
\newtheorem{remark}[theorem]{Remark}

\newcommand{\ab}{\mathbf{a}}
\newcommand{\bbb}{\mathbf{b}}
\newcommand{\cbb}{\mathbf{c}}
\newcommand{\db}{\mathbf{d}}
\newcommand{\eb}{\mathbf{e}}
\newcommand{\fb}{\mathbf{f}}
\newcommand{\gb}{\mathbf{g}}
\newcommand{\hb}{\mathbf{h}}
\newcommand{\ib}{\mathbf{i}}
\newcommand{\jb}{\mathbf{j}}
\newcommand{\kb}{\mathbf{k}}
\newcommand{\lb}{\mathbf{l}}
\newcommand{\mb}{\mathbf{m}}
\newcommand{\nbb}{\mathbf{n}}
\newcommand{\ob}{\mathbf{o}}
\newcommand{\pb}{\mathbf{p}}
\newcommand{\qb}{\mathbf{q}}
\newcommand{\rb}{\mathbf{r}}
\newcommand{\sbb}{\mathbf{s}}
\newcommand{\tb}{\mathbf{t}}
\newcommand{\ub}{\mathbf{u}}
\newcommand{\vb}{\mathbf{v}}
\newcommand{\wb}{\mathbf{w}}
\newcommand{\xb}{\mathbf{x}}
\newcommand{\yb}{\mathbf{y}}
\newcommand{\zb}{\mathbf{z}}

\newcommand{\lambdab}{\mathbf{\lambda}}
\newcommand{\Lambdab}{\mathbf{\Lambda}}
\newcommand{\thetab}{\bm{\theta}}
\newcommand{\gammab}{\bm{\gamma}}

\newcommand{\bb}{\bm{b}}
\newcommand{\bc}{\bm{c}}
\newcommand{\bd}{\bm{d}}
\newcommand{\be}{\bm{e}}
\newcommand{\bbf}{\bm{f}}
\newcommand{\bg}{\bm{g}}
\newcommand{\bh}{\bm{h}}
\newcommand{\bi}{\bmf{i}}
\newcommand{\bj}{\bm{j}}
\newcommand{\bk}{\bm{k}}
\newcommand{\bbm}{\bm{m}}
\newcommand{\bn}{\bm{n}}
\newcommand{\bo}{\bm{o}}
\newcommand{\bp}{\bm{p}}
\newcommand{\bq}{\bm{q}}
\newcommand{\bs}{\bm{s}}
\newcommand{\bt}{\bm{t}}
\newcommand{\bu}{\bm{u}}
\newcommand{\bv}{\bm{v}}
\newcommand{\bw}{\bm{w}}
\newcommand{\bx}{\bm{x}}
\newcommand{\by}{\bm{y}}
\newcommand{\bz}{\bm{z}}

\newcommand{\Ab}{\mathbf{A}}
\newcommand{\Bb}{\mathbf{B}}
\newcommand{\Cb}{\mathbf{C}}
\newcommand{\Db}{\mathbf{D}}
\newcommand{\Eb}{\mathbf{E}}
\newcommand{\Fb}{\mathbf{F}}
\newcommand{\Gb}{\mathbf{G}}
\newcommand{\Hb}{\mathbf{H}}
\newcommand{\Ib}{\mathbf{I}}
\newcommand{\Jb}{\mathbf{J}}
\newcommand{\Kb}{\mathbf{K}}
\newcommand{\Lb}{\mathbf{L}}
\newcommand{\Mb}{\mathbf{M}}
\newcommand{\Nb}{\mathbf{N}}
\newcommand{\Ob}{\mathbf{O}}
\newcommand{\Pb}{\mathbf{P}}
\newcommand{\Qb}{\mathbf{Q}}
\newcommand{\Rb}{\mathbf{R}}
\newcommand{\Sbb}{\mathbf{S}}
\newcommand{\Tb}{\mathbf{T}}
\newcommand{\Ub}{\mathbf{U}}
\newcommand{\Vb}{\mathbf{V}}
\newcommand{\Wb}{\mathbf{W}}
\newcommand{\Xb}{\mathbf{X}}
\newcommand{\Yb}{\mathbf{Y}}
\newcommand{\Zb}{\mathbf{Z}}

\newcommand{\bA}{\bm{A}}
\newcommand{\bB}{\bm{B}}
\newcommand{\bC}{\bm{C}}
\newcommand{\bD}{\bm{D}}
\newcommand{\bE}{\bm{E}}
\newcommand{\bF}{\bm{F}}
\newcommand{\bG}{\bm{G}}
\newcommand{\bH}{\bm{H}}
\newcommand{\bI}{\bm{I}}
\newcommand{\bJ}{\bm{J}}
\newcommand{\bK}{\bm{K}}
\newcommand{\bL}{\bm{L}}
\newcommand{\bM}{\bm{M}}
\newcommand{\bN}{\bm{N}}
\newcommand{\bO}{\bm{O}}
\newcommand{\bP}{\bm{P}}
\newcommand{\bQ}{\bm{Q}}
\newcommand{\bR}{\bm{R}}
\newcommand{\bS}{\bm{S}}
\newcommand{\bT}{\bm{T}}
\newcommand{\bU}{\bm{U}}
\newcommand{\bV}{\bm{V}}
\newcommand{\bW}{\bm{W}}
\newcommand{\bX}{\bm{X}}
\newcommand{\bY}{\bm{Y}}
\newcommand{\bZ}{\bm{Z}}

\newcommand{\x}{{\mathbf{x}}}
\newcommand{\y}{{\mathbf{y}}}
\newcommand{\z}{{\mathbf{z}}}
\newcommand{\w}{{\mathbf{w}}}
\newcommand{\ba}{{\mathbf{a}}}
\newcommand{\beps}{{\boldsymbol{\epsilon}}}
\newcommand{\bepss}{{\boldsymbol{\epsilon}}^2}
\newcommand{\bphi}{{\boldsymbol{\phi}}}
\newcommand{\tmu}{{\tilde\mu}}
\newcommand{\mulambda}[1]{\mu_{\lambda, #1}}
\newcommand{\alambda}{a_\lambda}
\newcommand{\tepsilon}{\tilde\epsilon}

\newcommand{\gauss}{{\boldsymbol{\xi}}}
\newcommand{\gz}{{\boldsymbol{\zeta}}}
\newcommand{\argmin}{\mathop{\mathrm{argmin}}}

\newcommand{\I}{{\mathbf{I}}}
\newcommand{\A}{{\mathbf{A}}}
\newcommand{\B}{{\mathbf{B}}}
\newcommand{\U}{{\mathbf{U}}}
\newcommand{\D}{{\mathbf{D}}}
\newcommand{\bXi}{\boldsymbol\Xi}

\newcommand{\bbeta}{\bm{\beta}}

\newcommand{\N}{{\mathbb{N}}}
\newcommand{\tlam}{\tilde{\lambda}}

\newcommand{\tB}{{\tilde{\mathbf{B}}}}

\newcommand{\tr}{\mathrm{tr}}

\newcommand{\AGD}{\mathrm{AGD}}
\newcommand{\cO}{\mathcal{O}}
\newcommand{\hntf}{\hat\nabla_\rho\tilde f_\delta}

\newcommand{\effdim}{r}
\newcommand{\efftrace}{\mathrm{ET}}

\newcommand{\lys}[1]{{\color{blue} #1}}
\newcommand{\zhz}{\textcolor{red}}
\newcommand{\xyy}{\textcolor{cyan}}
\newcommand{\ypy}{\textcolor{purple}}
\newcommand{\congfang}{\textcolor{red}}


\global\long\def\R{\mathbb{R}}%

\global\long\def\Rn{\mathbb{R}^{n}}%

\global\long\def\Rm{\mathbb{R}^{m}}%

\global\long\def\Rd{\mathbb{R}^{d}}%

\global\long\def\Rr{\mathbb{R}^{r}}%

\global\long\def\Rmn{\mathbb{R}^{m \times n}}%

\global\long\def\Rnm{\mathbb{R}^{n \times m}}%

\global\long\def\Rmm{\mathbb{R}^{m \times m}}%

\global\long\def\Rnn{\mathbb{R}^{n \times n}}%

\global\long\def\Z{\mathbb{Z}}%

\global\long\def\Rp{\R_{> 0}}%

\global\long\def\dom{\mathrm{dom}}%

\global\long\def\dInterior{K}%

\global\long\def\Rpm{\R_{> 0}^{m}}%


\global\long\def\ellOne{\ell_{1}}%
 
\global\long\def\ellTwo{\ell_{2}}%
 
\global\long\def\ellInf{\ell_{\infty}}%
 
\global\long\def\ellP{\ell_{p}}%

\global\long\def\otilde{\widetilde{O}}%

\global\long\def\argmax{\argmaxTex}%

\global\long\def\argmin{\argminTex}%

\global\long\def\sign{\signTex}%

\global\long\def\rank{\rankTex}%

\global\long\def\diag{\diagTex}%

\global\long\def\im{\imTex}%

\global\long\def\enspace{\quad}%

\global\long\def\mvar#1{\mathbf{#1}}%

\global\long\def\vvar#1{#1}%



\global\long\def\defeq{\stackrel{\mathrm{{\scriptscriptstyle def}}}{=}}%

\global\long\def\diag{\mathrm{diag}}%

\global\long\def\mDiag{\mvar{Diag}}%
 
\global\long\def\ceil#1{\left\lceil #1 \right\rceil }%

\global\long\def\E{\mathbb{E}}%

\global\long\def\jacobian{\mvar J}%

\global\long\def\onesVec{1}%
 
\global\long\def\indicVec#1{1_{#1}}%
\global\long\def\cordVec#1{e_{#1}}%
\global\long\def\op{\mathrm{spe}}%

\global\long\def\va{\vvar a}%
 
\global\long\def\vb{\vvar b}%
 
\global\long\def\vc{\vvar c}%
 
\global\long\def\vd{\vvar d}%
 
\global\long\def\ve{\vvar e}%
 
\global\long\def\vf{\vvar f}%
 
\global\long\def\vg{\vvar g}%
 
\global\long\def\vh{\vvar h}%
 
\global\long\def\vl{\vvar l}%
 
\global\long\def\vm{\vvar m}%
 
\global\long\def\vn{\vvar n}%
 
\global\long\def\vo{\vvar o}%
 
\global\long\def\vp{\vvar p}%
 
\global\long\def\vs{\vvar s}%
 
\global\long\def\vu{\vvar u}%
 
\global\long\def\vv{\vvar v}%
 
\global\long\def\vx{\vvar x}%
 
\global\long\def\vy{\vvar y}%
 
\global\long\def\vz{\vvar z}%
 
\global\long\def\vxi{\vvar{\xi}}%
 
\global\long\def\valpha{\vvar{\alpha}}%
 
\global\long\def\veta{\vvar{\eta}}%
 
\global\long\def\vphi{\vvar{\phi}}%
\global\long\def\vpsi{\vvar{\psi}}%
 
\global\long\def\vsigma{\vvar{\sigma}}%
 
\global\long\def\vgamma{\vvar{\gamma}}%
 
\global\long\def\vphi{\vvar{\phi}}%
\global\long\def\vDelta{\vvar{\Delta}}%
\global\long\def\vzero{\vvar 0}%

\global\long\def\ma{\mvar A}%
 
\global\long\def\mb{\mvar B}%
 
\global\long\def\mc{\mvar C}%
 
\global\long\def\md{\mvar D}%
 
\global\long\def\mf{\mvar F}%
 
\global\long\def\mg{\mvar G}%
 
\global\long\def\mh{\mvar H}%
 
\global\long\def\mj{\mvar J}%
 
\global\long\def\mk{\mvar K}%
 
\global\long\def\mm{\mvar M}%
 
\global\long\def\mn{\mvar N}%

\global\long\def\mO{\mvar O}%

\global\long\def\mq{\mvar Q}%
 
\global\long\def\mr{\mvar R}%
 
\global\long\def\ms{\mvar S}%
 
\global\long\def\mt{\mvar T}%
 
\global\long\def\mU{\mvar U}%
 
\global\long\def\mv{\mvar V}%
 
\global\long\def\mx{\mvar X}%
 
\global\long\def\my{\mvar Y}%
 
\global\long\def\mz{\mvar Z}%
 
\global\long\def\mSigma{\mvar{\Sigma}}%
 
\global\long\def\mLambda{\mvar{\Lambda}}%
\global\long\def\mPhi{\mvar{\Phi}}%
 
\global\long\def\mZero{\mvar 0}%
 
\global\long\def\iMatrix{\mvar I}%
\global\long\def\mi{\mvar I}%
\global\long\def\mDelta{\mvar{\Delta}}%

\global\long\def\oracle{\mathcal{O}}%
 
\global\long\def\mw{\mvar W}%

\global\long\def\runtime{\mathcal{T}}%


\global\long\def\mProj{\mvar P}%

\global\long\def\vLever{\sigma}%
 
\global\long\def\mLever{\mSigma}%
 
\global\long\def\mLapProj{\mvar{\Lambda}}%

\global\long\def\penalizedObjective{f_{t}}%
 
\global\long\def\penalizedObjectiveWeight{{\color{red}f}}%

\global\long\def\fvWeight{\vg}%
 
\global\long\def\fmWeight{\mg}%

\global\long\def\vNewtonStep{\vh}%

\global\long\def\norm#1{\|#1\|}%
 
\global\long\def\normFull#1{\left\Vert #1\right\Vert }%
 
\global\long\def\normA#1{\norm{#1}_{\ma}}%
 
\global\long\def\normFullInf#1{\normFull{#1}_{\infty}}%

\global\long\def\normFullSquare#1{\normFull{#1}_{\square}}%
 
\global\long\def\normInf#1{\norm{#1}_{\infty}}%
 
\global\long\def\normOne#1{\norm{#1}_{1}}%
 
\global\long\def\normTwo#1{\norm{#1}_{2}}%
 
\global\long\def\normLeverage#1{\norm{#1}_{\mSigma}}%
 
\global\long\def\normWeight#1{\norm{#1}_{\fmWeight}}%

\global\long\def\cWeightSize{c_{1}}%
 
\global\long\def\cWeightStab{c_{\gamma}}%
 
\global\long\def\cWeightCons{{\color{red}c_{\delta}}}%

\global\long\def\TODO#1{{\color{red}TODO:\text{#1}}}%
\global\long\def\mixedNorm#1#2{\norm{#1}_{#2+\square}}%

\global\long\def\mixedNormFull#1#2{\normFull{#1}_{#2+\square}}%
\global\long\def\CNorm{C_{\mathrm{norm}}}%
\global\long\def\Pxw{\mProj_{\vx,\vWeight}}%
\global\long\def\vq{q}%
\global\long\def\cnorm{\CNorm}%

\global\long\def\next#1{#1^{\mathrm{(new)}}}%

\global\long\def\trInit{\vx^{(0)}}%
 
\global\long\def\trCurr{\vx^{(k)}}%
 
\global\long\def\trNext{\vx^{(k + 1)}}%
 
\global\long\def\trAdve{\vy^{(k)}}%
 
\global\long\def\trAfterAdve{\vy}%
 
\global\long\def\trMeas{\vz^{(k)}}%
 
\global\long\def\trAfterMeas{\vz}%
 
\global\long\def\trGradCurr{\grad\Phi_{\alpha}(\trCurr)}%
 
\global\long\def\trGradAdve{\grad\Phi_{\alpha}(\trAdve)}%
 
\global\long\def\trGradMeas{\grad\Phi_{\alpha}(\trMeas)}%
 
\global\long\def\trGradAfterAdve{\grad\Phi_{\alpha}(\trAfterAdve)}%
 
\global\long\def\trGradAfterMeas{\grad\Phi_{\alpha}(\trAfterMeas)}%
 
\global\long\def\trSetCurr{U^{(k)}}%
\global\long\def\vWeightError{\vvar{\Psi}}%
\global\long\def\code#1{\texttt{#1}}%

\global\long\def\nnz{\mathrm{nnz}}%
\global\long\def\tr{\mathrm{tr}}%
\global\long\def\vones{\vec{1}}%

\global\long\def\volPot{\mathcal{V}}%
\global\long\def\grad{\mathcal{\nabla}}%
\global\long\def\hess{\nabla^{2}}%
\global\long\def\hessian{\nabla^{2}}%

\global\long\def\shurProd{\circ}%
 
\global\long\def\shurSquared#1{{#1}^{(2)}}%
\global\long\def\solver{\mathrm{\mathtt{S}}}
\global\long\def\time{\mathrm{\mathcal{T}}}
\global\long\def\trans{\top}%

\global\long\def\lpweight{w_{p}}%
\global\long\def\mlpweight{\mw_{p}}%
\global\long\def\lqweight{w_{q}}%
\global\long\def\mlqweight{\mw_{q}}%

\newcommand{\bracket}[1]{[#1]}
\global\long\def\interiorPrimal{\Omega^{\circ}}%
\global\long\def\interiorDual{\Omega^{\circ}}%

\newcommand{\tx}{\tilde{\mathbf{x}}}
\newcommand{\errorA}{\epsilon_A}
\newcommand{\errorB}{\epsilon_A}
\newcommand{\errorC}{\epsilon_B}
\newcommand{\errorD}{\epsilon_C}
\newcommand{\BSa}{\mathop{\mathtt{CBinarySearch}}}
\newcommand{\bbv}{\mathbf{v}}
\newcommand{\BSb}{\mathop{\mathtt{CCubicBinarySearch}}}
\newcommand{\rtemp}{r_{\mathrm{temp}}}
\newcommand{\ltemp}{\gamma_{\mathrm{temp}}}
\newcommand{\tilder}{\tilde{r}}
\newcommand{\elltemp}{l_{\mathrm{temp}}}
\global\long\def\vol{\mathrm{vol}}%
\global\long\def\rPos{\R_{> 0}}%
\global\long\def\dWeight{\R_{>0}^{m}}%
\global\long\def\dWeights{\rPos^{m}}%

\global\long\def\cO{\mathcal{O}}%
\global\long\def\tO{\tilde{\cO}}%

\global\long\def\specGeq{\succeq}%

\global\long\def\specLeq{\preceq}%

\global\long\def\specGt{\succ}%

\global\long\def\specLt{\prec}%
\global\long\def\gradient{\nabla}%
\global\long\def\weight{w}%

\global\long\def\vWeight{\vvar{\weight}}%

\global\long\def\mWeight{\mvar W}%
\global\long\def\mNormProjLap{\bar{\mLambda}}%
\global\long\def\volPot{\mathcal{V}}%
\global\long\def\dInterior{\Omega^{\circ}}%
\global\long\def\dFull{\{\dInterior\times\R_{>0}^{m}\}}%
\global\long\def\vWeight{\vvar w}%

\global\long\def\mWeight{\mvar W}%
\global\long\def\polytope{\Omega}%
\global\long\def\interior{\Omega^{\circ}}%
\maketitle


\begin{abstract}
Gradient-based minimax optimal algorithms have greatly promoted the development of continuous optimization and machine learning, with the measurement of performance based on the hardest instance. One seminal work due to Yurii Nesterov \cite{Nesterov1983AMF} with successive follow-up works established $\tilde{\mathcal{O}}( \sqrt{L/\mu})$ gradient and computational complexities for minimizing an $L$-smooth $\mu$-strongly convex objective. However, an ideal algorithm would adapt to the explicit complexity of a particular objective function and incur faster rates for simpler problems, triggering our reconsideration of two defeats of existing optimization modeling and analysis. (i) The worst-case optimality is neither the instance optimality nor such one in reality. (ii) Traditional $L$-smoothness condition may not be the primary abstraction/characterization for modern practical problems. For example, for empirical risk minimization problems,  not merely the spectrum of Hessian is bounded from above by a constant, but even the nuclear norm \cite{zhang2005learning}. 

In this paper, we open up a new way to design and analyze gradient-based algorithms with direct applications in machine learning, including linear regression and beyond. We introduce two factors $(\alpha, \tau_{\alpha})$ to refine the description of the degenerated condition of the optimization problems based on the observation that the singular values of Hessian often drop sharply. We design adaptive algorithms that solve simpler problems without pre-known knowledge with reduced gradient or analogous oracle accesses. The algorithms also improve the state-of-art complexities for several problems in machine learning, thereby solving the open problem of how to design faster algorithms in light of the known complexity lower bounds. Specially, with the $\mathcal{O}(1)$-nuclear norm bounded, we achieve an \textit{optimal} ${\mathit{\tilde{\mathcal{O}}(\mu^{-1/3})}}$  (v.s. $\tilde{\mathcal{O}}(\mu^{-1/2})$) gradient complexity for linear regression. We hope this work could invoke the rethinking for understanding the difficulty of modern problems in optimization.

\end{abstract}

\section{Introduction}
Gradient-based algorithms play as workhorses in a large range of practical applications and have been developed vigorously for a long period. One common model that tailors to study the algorithms is to consider the $L$-smooth condition which requires the objective to have $L$-Lipschitz continuous gradients. It is well-known that during the period from 1983 to 1985, Yurii Nesterov \cite{nesterov1983method,nesterov1984one,nemirovskii1985optimal}    invented several accelerated gradient algorithms that achieve  $\tilde{\cO}( \sqrt{L/\mu})$ gradient and computational complexities for minimizing an $L$-smooth and $\mu$-strongly convex objective. Such kinds of algorithms are often called optimal gradient algorithms in the sense that they achieve min-max optimal gradient complexity. That is to say, there is a hard instance among the $L$-smooth and $\mu$-strongly convex functions for which any algorithm needs $\tilde{\Omega}( \sqrt{L/\mu})$ gradient evaluations to find an $\epsilon$-suboptimal solution.

Up to now, lots of accelerated algorithms have been successfully designed based on the pioneer framework 
under different or more specific settings. See e.g. accelerated coordinate method \cite{allen2016even, nesterov2012efficiency} and high-order acceleration under high-order Lipschitz conditions \cite{nesterov_cubic_2006,monteiro2013accelerated}. One more typical example  in  machine learning is  Empirical Risk Minimization (ERM), for which the objective is often an average of $n$ smooth functions.  And one is able to achieve $\tO(n+\sqrt{n/\mu})$ \cite{allen2017katyusha,lin2015universal,zhang2015stochastic} individual gradient costs  when each individual function is $\mu$-strongly convex.  Accordingly,  lower bound complexities are also established to show the optimality of algorithms using more involved techniques \cite{nesterov1998introductory}. Seeing the situation, we might feel that the research area is relatively mature and it remains open to further accelerate these plausibly unimprovable algorithms in light of the known complexity lower bounds. 

Though the progress of designing the optimal algorithms is breathtaking and worthy of a warm celebration,  an  introspection is -- \textit{are we still climbing the right hill?} One essential question that we are more concerned about is -- \textit{is the current default problem characterization and complexity analysis the best one for studying the real difficulty of modern practical problems, such as for machine learning?}  We argue that there
are two unconscious misunderstandings in modeling and analyzing the optimization problems that potentially hinder further progress for optimization on modern problems. 
\begin{itemize}
    \item[(i)]  Traditional $L$-smoothness condition may not always be the primary characterization for optimization problems. Taking the usual ERM problem as an instance, the goal of the task is to minimize the objective of the form:
\begin{equation}\label{eq:erm1}
     \min_\xb ~~F(\xb)=\frac{1}{n}\sum_{i=1}^n f_i(\ab_i^\top\xb ). 
    \end{equation} 
    To show the objective \eqref{eq:erm1} satisfies the smoothness conditions, we often assume that for each $i\in[n]$, the data $\ab_i$ is normalized to $\|\ab_i\|^2\leq R^2$ and $f\in\mathcal{C}^2$ is convex  and  $L_0$-smooth. Then we have
\begin{equation}\label{eq:estimate ell}
      \left\|\nabla^2 F \right\|_2 \leq      \left\|\nabla^2 F\right\|_*= \frac{1}{n} \sum_{i=1}^n f_i''(\ab_i^\top\xb) \tr(\ab_i\ab_i^\top) \leq  \sum_{i=1}^n\frac{L_0}{n} \|\ab_i\|^2  \leq L_0R^2.
    \end{equation}
    One can observe from \eqref{eq:estimate ell} that  not merely the spectrum of Hessian is bounded  by $L_0R^2$, but also the nuclear norm. The result  demonstrates a sharp drop in singular values of Hessian.
\item[(ii)]  Worst-case optimality is
not the instance optimality. The real practical problems may  usually not be as difficult
as the hardest case.  Consider the hard instance in convex optimization constructed by Nesterov \cite{Nesterov1983AMF}. It is a quadratic function with form: $ \frac{1}{2}(1-x_1)^2+\frac{1}{2}\sum_{i=1}^{d-1}(x_{i+1}-x_i)^2$. For any linear-span algorithm producing the iterations as  $ \xb_t \in \mathrm{Span}\left(\xb_0, \nabla f(\xb_0),\cdots,\nabla f(\xb_{t-1})\right) $ with initialization at  $\xb_0 =\mathbf{0}$, the objective constrains the algorithm per-iteratively solving only one entry of the variable. So once the function has a sufficiently high dimension, limited access to the gradient cannot solve the rest entries. Although for modern tasks,  the optimization problems are often high-dimensional, it is not clear that the number of effective entries that really need to be recovered step by step is large.  Let us take note of the evidence from preliminary experiments in machine learning. 
We conduct both linear regression and neural network training on two benchmark datasets: MNIST \cite{lecun1998gradient} and CIFAR10 \cite{krizhevsky2009learning}. The experiment description is deferred to Appendix \ref{app:experiment}. The experimental result, shown in Fig.~\ref{EXP1}, indicates that the convergence on MNIST is significantly faster than that on CIFAR10. 
For simpler problems,
reduced complexities can be obtained. 
\end{itemize}

\begin{figure*}[t]
	\centering
        \vspace{-0.35cm}
        \subfigtopskip = 5pt
        \subfigcapskip = -5pt
	\subfigure[Linear Regression]{
       \begin{minipage}[t]{0.5\linewidth}
			\centering
			\includegraphics[width=3.3in]{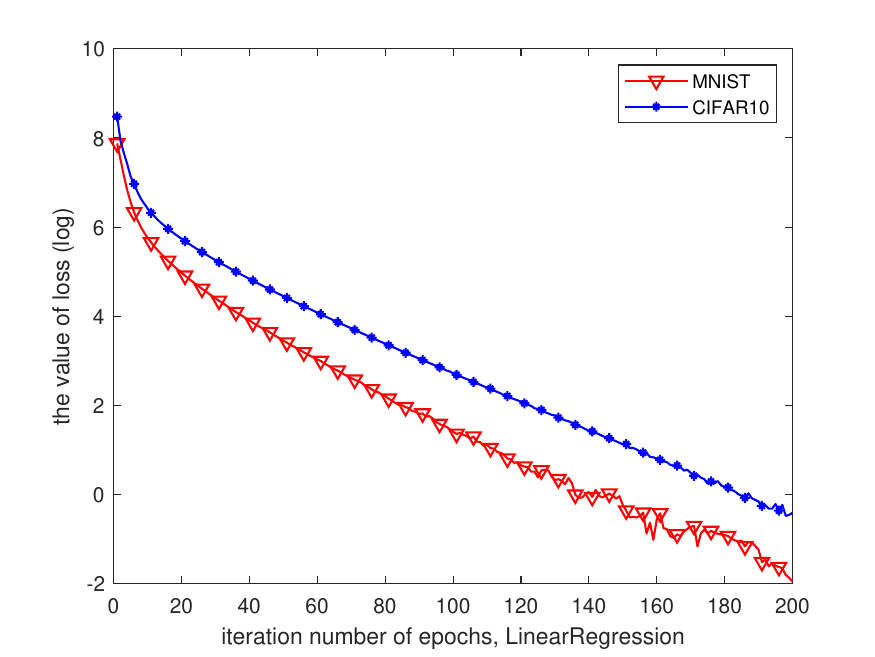}\\			\vspace{0.02cm}
		\end{minipage}%
	}%
       \subfigure[ResNet18]{
       \begin{minipage}[t]{0.5\linewidth}
			\centering
			\includegraphics[width=3.3in]{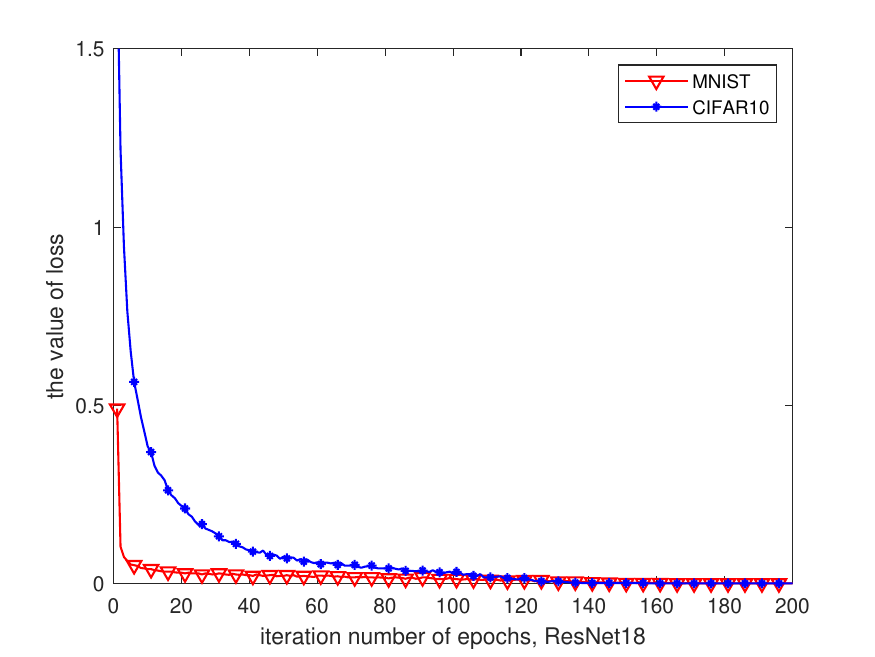}\\			\vspace{0.02cm}
		\end{minipage}%
	}%
	\centering
	\caption{The loss convergence speed of linear regression and neural networks. (a) plots the linear regression loss against the number of epochs on MNIST \cite{lecun1998gradient} and CIFAR10 \cite{krizhevsky2009learning}. (b) plots the ResNet18 with cross-entropy loss value against the number of epochs on MNIST and CIFAR10.}
 \label{EXP1}
	\vspace{-0.2cm}
	
\end{figure*}

\subsection{What We Do}
With this reconsideration, we attempt to answer the following  questions in this paper: 
\begin{itemize}
    \item[(A)]  \textit{Can we propose a more refined model/characterization of optimization problems to study the difficulty of solving practical problems closer to reality?}\label{labelA}
    \item[(B)]  \textit{Based on the model, is it possible to achieve  reduced complexities beyond those lower bounds under realizable and common conditions?} 
\end{itemize}

To answer the question   (\hyperref[labelA] {A}), we introduce two factors $(\alpha, \tau_{\alpha})$ to refine the characterization of an optimization problem. The main intuition is based on the commonly accepted observation that the singular values of Hessian for objectives always drop rapidly.  We use $\alpha\in(0, \infty]$ to describe the level of the  degeneracy of the Hessian matrices and define
\begin{equation}\label{eq: eff}
    \tau_{\alpha} = \sup_{\x\in \mathbb{R}^d} \left(\sum_{i=1}^d \left|\lambda_i(\nabla^2 f(\x))\right|^\alpha\right)^{\frac{1}{\alpha}},\qquad \alpha>0
\end{equation} 
to describe the magnitude of the degeneracy at the $\alpha$-level, where $\lambda_i(\cdot)$ is the $i$-th eigenvalue.
Compared with the smoothness and the dimension of the variables, we think the $(\alpha, \tau_{\alpha})$ description is a more accurate and fine-grained indicator to describe the difficulty of problems. When $\alpha =\infty$, $\tau_{\infty}$ reduces to the standard $L$-smoothness condition. For ERM  under the data normalization setting, we have for $\alpha\geq 1$, $\tau_\alpha \leq L_0 R^2$ and can pick $\alpha =1$. For simpler problems,  we might have $\tau_\alpha \approx 1$ for $\alpha \ll 1$. The description can also be understood as a structural condition for the optimization problems. When the structure is pre-known such as for the aforementioned ERM problems, one can design a specific algorithm under the structural assumption. However, when the structure is \textbf{not} pre-known, choosing the best pair $(\alpha^*, \tau_{\alpha^*})$ before applying the algorithm still requires non-negligible expenses, even for quadratic objectives. The general solution proposed by this paper is to  design adaptive algorithms  
that can automatically fit the structure of the optimization problem. 

In Section \ref{sec: mainquadratic}, we start our analysis from the quadratic objective. The quadratic functions class stands as a representative objective in the convex world. It appears to be the hard case for many gradient oracle models \cite{nesterov1998introductory} and many methods ground on quadratic optimization locally \cite{nesterov2017random,yue2023zeroth}. We propose an adaptive algorithm named AGMAS that uses at most $\tilde{\mathcal{O}}\left( \min_{k}\left\{ k + \sqrt{\frac{\lambda_k}{\mu}} \right\} \right)$ gradient calls to find an $\epsilon$-approximate minimizer. Under the $(\alpha,\tau_{\alpha})$ description, the algorithm does not need to acquire $\alpha$ or $\tau_{\alpha}$. The algorithm adaptively searches for the optimal degeneracy characterization of the problem. This will lead to an $\tilde{\mathcal{O}}\left(\min\{\mu^{-\frac{1}{2}}, \inf_{\alpha} \tau_{\alpha}^{\frac{\alpha}{1+2\alpha}}, d\}\right)$ complexity. For linear regression, the most common case in ERM,  we can obtain the   ${\cO(\mu^{-1/3})}$ gradient complexity, outperforming accelerated gradient descent \cite{Nesterov1983AMF} by a $\mu^{-\frac{1}{6}}$ factor. Note that the acceleration is not a consequence of the normalization of the problem but the fast drop of  eigenvalues. To see it,   a scale shift does not change the condition number of the problem and thus cannot make accelerated gradient descent convergence faster.  The reason for the opportunity to attain the faster convergence rate is  based on the fact that $\tau_1\approx \tau_{\infty}$.
We thereby solve the open problem in machine learning of how to design faster algorithms in light of the known complexity lower bound. The core idea of algorithm design is intuitive: by noting that the Hessian drops quickly, we can adaptively separate the space of the variable into two subspaces according to the magnitude of eigenvalues. At a high level, provided the rapid dropping eigenvalues under the Hessian degeneracy assumption, the eigenspace with large eigenvalue is low-dimensional. This leads to a prospect of acceleration via performing computation methods on large eigenspace. As for the remaining small eigenspace, the gradient method is applied since the leading eigenvalue is much smaller.

After proposing the refined and close-to-reality conditions, it is natural to come up with the following question: is the proposed adaptive algorithm min-max optimal? We provide an affirmative answer to this question. We show that the gradient complexity of AGMAS cannot be further improved. We construct a lower bound in terms of our degeneracy description. To be specific, following the seminal work from \cite{simchowitz2018randomized, braverman2020gradient}, we obtain an oracle lower bound of finding the largest eigenvalue using a distribution over a $3\times3$ block diagonal matrix, and establish an according lower bound for quadratic optimization using the shift-and-inverse paradigm. The lower bound states that for any randomized algorithm, it requires at least $\tilde{\mathcal{O}}\left( \min\{\mu^{-\frac{1}{2}}, \tau_{\alpha}^{\frac{\alpha}{1+2\alpha}}, d\} \right)$ gradient oracle calls to achieve an accurate approximate of the minimizer with a constant probability.

Our algorithm is simple and versatile for further improvement and generalization using techniques from optimization and numerical computation. In Section \ref{sec: genericfunc}, we consider the algorithm for optimizing the general objective under additional Hessian smoothness conditions. In the convex case, by extending Algorithm \ref{alg: esgd}, we achieve 
$\tilde{\mathcal{O}}\left(\tau_\alpha^{\frac{\alpha}{1+2\alpha}} D^{\frac{14\alpha+12}{14\alpha+7}} H^\frac{2}{14\alpha+7}\epsilon^{-\frac{7\alpha+2}{14\alpha+7}} +  D^{\frac{6}{7}}H^{\frac{2}{7}}\epsilon^{-\frac{2}{7}}\right)$ gradient oracle complexity,  and in non-convex case we achieve $\tilde\cO\left(H^{\frac{1+\alpha}{2+4\alpha}} \cdot \Delta \cdot
\tau_{\alpha}^{\frac{\alpha}{1+2\alpha}}\epsilon^{-\frac{3+7\alpha}{2+4\alpha}} \right)$ gradient oracle complexity to find an $\left(\epsilon, \sqrt{H\epsilon}\right)$-approximate second-order stationary point. Considering the common case when $\alpha = 1$, our results are beyond the state-of-the-art result $\tilde\cO(\epsilon^{-1/2})$ in convex case \cite{nesterov1998introductory} and $\tilde\cO(\epsilon^{-7/4})$ in non-convex setting \cite{jin2018accelerated}.

In Section \ref{sec: erm}, we extend our algorithm to solve the ERM problem. We consider two types of complexities: the number of data access and the computational costs. For data access complexity, we solve the problem using a mini-batch version of accelerated stochastic gradient method in primal space \cite{allen2017katyusha}, and perform our eigenvector extractor algorithm over each mini-batch. We obtain an $\tilde{\mathcal{O}}\left( n + n^{\frac{5}{6}}\epsilon^{-\frac{1}{3}} \right)$ data access complexity to obtain an $\epsilon$-approximate minimizer. And when $d\leq n$, using a leverage score sampling technique \cite{agarwal2017leverage}, the algorithm has data access complexity is $\tilde{\mathcal{O}}\left( n + d^{\frac{5}{6}}\epsilon^{-\frac{1}{3}} \right)$. For computational costs, we achieve different complexities in different regimes divided by $\alpha$, $d,n$, $\mu$ and $\tau$. For instance, we prove that given a block-weight-function, we can achieve complexity $\tilde\cO \left((nd)^{\frac{1-5\alpha}{1-3\alpha}}\left( \frac{\tau}{\mu} \right)^{\frac{\alpha}{1-3\alpha}}\right)$  in the regime when $\alpha\le 1/5, d\ge n^{3/2}$ and $ \frac{\tau}{\mu}\ge(nd)^{\frac{4\alpha+2}{5\alpha}} $ through combining our technique and interior point methods. We show that this surpasses accelerated stochastic variance reduction algorithms in some regimes.

\section{Previous Work}
In this section, we review the previous representative works that we think are most related. Some works form the basic of our work. And we discuss the relations and differences.

\textbf{Part I: Gradient-based Algorithms and  Results  on Classical  Setting.}  For the classical general $L$-smooth and $\mu$-stongly convex objective function, Nesterov proposed  the accelerated gradient descent (AGD) \cite{nesterov1983method,nesterov1984one,nemirovskii1985optimal} methods reaching $\tO\left(\sqrt{L/\max(\mu,\epsilon)}\right)$ gradient complexities. The proximal version of AGD is proposed later \cite{beck2009fast}. When the convex objective is an average of $n$ smooth  functions,  the individual gradient complexities can be upper bounded by 
$\tO\left(n+\sqrt{n/\max(\mu,\epsilon})\right)$ \cite{allen2017katyusha,lin2015universal,zhang2015stochastic} or  $\tO\left(n+n^{3/4}\sqrt{1/\max(\mu, \epsilon)}\right)$ \cite{allen2018katyusha,allen2017natasha,carmon2018accelerated}  depending on the convexity of each individual function. 

On the lower bound part,  Nesterov \cite{nesterov1998introductory} studied the iteration complexity of deterministic linear-span algorithms and showed that $\tilde{\Omega}\left(\sqrt{L/\max(\mu,\epsilon)}\right)$  iteration steps are  indispensable. Due to the restriction of  the algorithm class,  the iteration  complexity  cannot imply a  lower bound on computation costs for  more general algorithms, whereas,  implies a record-breaking  lower bound on gradient oracle accesses. Indeed,  Nemirovskij and Yudin \cite{nemirovskij1983problem} generalize the  gradient complexity lower bound for  any deterministic algorithm. The inspiring works from \cite{woodworth2017lower,braverman2020gradient,simchowitz2018randomized} consider any randomized algorithms and establish the same lower bound up to logarithmic  factors.

\textbf{Part II: Newton Method and Beyond.} Our algorithm is a kind of damped Newton method equipped with the limited-memory trick. Both the damped Newton methods (see e.g. \cite{robinson1994newton,ralph1994global})  and limited-memory trick (see e.g. \cite{liu1989limited,nocedal1980updating,nash1991numerical,byrd1994representations}), such as the commonly used quasi-Newton method L-BFGS \cite{liu1989limited,nocedal1980updating,nash1991numerical,byrd1994representations},   are prevalent and developed dating back to more than 50 years ago.   So our algorithms are not brand new.   The novel ingredient is to design the algorithms that dynamically choose the size of subspace and so attain adaptive and provably faster convergence rates under our setting. We also use the cubic regularization trick to extend the algorithm on generic optimization under Hessian smoothness conditions. Cubic Regularization Newton Method \cite{nesterov_cubic_2006} uses the second-order oracle to minimize a regularized objective function in each step. For non-smooth functions, the cubic regularization achieved an $\mathcal O\left(\epsilon^{-3/2}\right)$ convergence rate. For convex functions, cubic regularization method can be accelerated to an $\tilde{\mathcal{O}}\left(-1/3\right)$ convergence rate \cite{nesterov2008accelerating}. For convex functions, the optimal second-order algorithm is the large-step A-HPE method \cite{monteiro2013accelerated}, achieving an optimal $\tilde{\mathcal{O}}\left(\epsilon^{-2/7}\right)$ convergence rate. The large-step A-HPE method can be extended to higher-order algorithms, achieving an optimal convergence rate of $\tilde{\mathcal{O}}\left(\epsilon^{\frac{2}{3k+1}}\right)$ for $k$th-order algorithms. Many previous works studied the convergence rate of cubic regularization methods using first-order algorithms as the sub-problem solver \cite{cartis_adaptive_2011,carmon_gradient_2022,tripuraneni_stochastic_2018,yue2023zeroth}.

The modern interior-point methods are improvements of Newton method, which are designed usually for structural constrained problems such as linear programming of either the following form:
\begin{equation}\label{eq:linear-programming}
    (\text{P})= \min_{\xb\in\Rd_{\ge 0}:\ma^\top \xb=\bbb} \cbb^\top\xb \text{ and (D)}=\max_{\yb\in\Rr:\ma y\ge \cbb}\bbb^\top \yb,
\end{equation}
where $\ma\in\R^{d\times r}$, $b\in\Rr$ and $c\in\Rd$. 
Karmarkar is the first to prove that interior point methods can solve  linear programs in polynomial time \cite{karmarkar1984new}, and the interior methods were developed both in theory and efficiency by several works \cite{renegar1988polynomial,vaidya1989new,vaidya1989speeding,Nesterov1994}.
These years, the research has entered another prosperous age since Lee and Sidford \cite{leeS14,lee2015efficient,lsJournal19} innovatively developed weighted path finding and inverse maintenance technique and made important breakthroughs upon both iteration numbers and amortized cost per iteration to reduce the solving time to $\tilde\cO\left(( \nnz(\A) + r^2)\sqrt{r}\log\frac{1}{\epsilon}\right)$ and was recently improved to $\tilde\cO\left(dr+r^3\right) $ by \cite{van2020solving} in 2020. For the regime of $r = \Omega(d)$, \cite{cohen2019solving, jiang2020faster,LeeSZ19, van2020deterministic} improved the time to $\cO^*\left( (d^\omega)\log\frac{d}{\epsilon}\right)$ for (\ref{eq:linear-programming})(P) and ERM where $\omega\approx 2.37$ is the current matrix multiplication constant \cite{williams2023new}.

\textbf{Part III: Adaptive Algorithms and Works on Hessian  Trace Bounded Problems.} 
The idea of designing adaptive algorithms is widely considered in the field of statistical learning, where one often expect   algorithms adaptively fit an underlying structure of the  data. In the field of optimization, the idea  is also not new.  Representative  and  earlier works include the space dilation methods, (see e.g. \cite{shor1970convergence,shor1972utilization,shor1975convergence}) which compute the deflected
gradients by using a transformed metric based on the Hessian,  
and variable
metric methods, such as the BFGS family of algorithms (see e.g. earlier works \cite{broyden1970convergence,fletcher1970new,goldfarb1970family,shanno1970conditioning}). 
One notable online algorithms are adaptive sub-gradient methods, such as AdaGrad \cite{duchi2011adaptive}, Adam \cite{kingma2014adam} and AdamW \cite{loshchilov2017decoupled} which are the mainstream algorithms used for training deep neural networks. The main tuition of these algorithms is based on the observation that the data instance only has a few  non-zero features. Since  these features are often highly informative and discriminative, it is preferred to adopt  a large step size for these features.  Our work explicitly digs out the complexity advantages of algorithms under the setting where 
eigenvalues of Hessian drop fast and improve the state-of-art complexities for
several problems in machine learning.


There are also some works that study more efficient algorithms for Hessian trace bounded problems \cite{zhang2005learning,allen2016even,nesterov2012efficiency,lee2013efficient}. For example, the state-of-the-art accelerated coordinate descent achieves $\tilde\cO\left(\sum_{i=1}^d \sqrt{\Ab_{ii}}/\max\{\mu,\epsilon\}\right)$ oracle complexity where $\Ab$ is the Hessian matrix of the objective. Moreover, recent research on zeroth-order optimization \cite{yue2023zeroth}  and distributed optimization \cite{yue2023core} shows that an  $\tO\left(\tau_{1/2}/\sqrt{\max(\mu,\epsilon)}\right)$ zeroth-order oracle accesses is enough for quadratic functions. However, these researches only study the case when $\alpha=1$ or $\alpha=1/2$, which is a special case of $(\alpha,\tau_\alpha)$-degenerated functions and is not adaptive for $\alpha$. Besides, considering the gradient-based method proposed by \cite{yue2023core,yue2023zeroth}, to the best of our knowledge, the gradient oracle complexity does not break the existing lower bound $\tO\left(\max\{\mu,\epsilon\}^{-1/2}\right)$. In this paper, we achieve remarkable $\tO\left(\max\{\mu,\epsilon\}^{-1/3}\right)$ beyond the existing results.

\section{Notations}
\textbf{Orders Analysis: } Use the conventional notations $\mathcal{O}(\cdot), \Theta(\cdot), \Omega(\cdot)$ that ignore the absolute constants. $\tilde{\mathcal{O}}(\cdot)$, $\tilde{\Theta}(\cdot)$, $\tilde{\Omega}(\cdot)$ ignore the logarithmic factors and $ {\mathcal{O}}^*(\cdot)$, ${\Theta}^*(\cdot)$, ${\Omega}^*(\cdot)$ ignore the $n^{o(1)}$ factors. We adopt the notations that $f(x)\lesssim g(x)$ if $f(x) = \mathcal{O}(g(x))$ and $f(\xb) = \mathrm{poly}(\xb)$ if $f$ can be bounded by a polynomial of $\xb$. 

\textbf{Vectors Operations:} We let $\langle \x,\y \rangle$ denote the inner product of two vectors $\x$ and $\y$ in the Euclidean space. Besides we apply scalar operators to vectors with the interpretation that these operations should be applied to each coordinate of two vectors, e.g. for $\x, \y \in \Rd$, we denote $\x /\y \in \Rd$ with $[x/y]_i = x_i /y_i$.

\textbf{Matrices:} We call a matrix $\A$ non-degenerate if it has full column rank without zero rows. Use $\lambda_i(\Ab)$ to denote the $i$-th largest eigenvalue of symmetric matrix $\Ab$, and use $\lambda_{\min}(\Ab), \lambda_{\max}(\Ab)$ to denote its smallest, largest eigenvalue. We call a symmetric matrix $\B \in \mathbb{R^{d\times d}}$ positive definite if for all $\x \in \Rd$, $\x^\top \B \x >0$, and positive semidefinite 
if for all $\x \in \Rd$, $\x^\top \B \x \geq 0$. We define $\rm diag(\A) \in \Rd$ with $\rm diag(\A)_i = \A_{ii}$ for all $i \in [d]$ for $\A \in \mathbb{R}^{d \times d}$. 

\textbf{Matrix Operations:} Let $\A \preceq \B$ indicate that for all $\x \in \Rd$, $\x^\top \A\x \leq \x^\top \B\x$ for two symmetric matrices $\A, \B \in \mathbb{R^{d \times d}}$, and define $\prec$, $\succeq$ and $\succ$ analogously. We let $\rm nnz(\A)$ denote the number of nonzero entries in $\A$. We let $\mathbf{P}(\A) = \A(\A^\top \A)^{-1}\A^\top$ denote the orthogonal projection matrix onto a non-degenerate matrix $\A$'s image. We define $\sigma(\A) = \rm diag(\mathbf{P}(\A))$ the leverage scores of $\A$.

\textbf{Norms:} We let $\|\cdot \|$ denote the Euclidean norm of a vector. We let $\|\cdot \|_\A$ denote the Mahalanobia norm of a vector where $\|\x\|_\A = \sqrt{\x^\top \A\x}$ for all $\x \in \Rd$ and positive definite matrix $\A \in \mathbf{R^{d\times d}}$. For positive $\mathbf{w} \in \Rd_{>0}$ we let $\|\cdot\|_\mathbf{w}$ denote the norm where $\|\x\|_\mathbf{w} = \sqrt{ \sum_{i=1}^d w_i x_i^2 }$ for all $\x \in \Rd$.

\textbf{Calculus:} For a function $f(\x) \in C^2$, which means that $f$ is  second-order derivative. We use $\nabla f(\x)$ and $\nabla^2 f(\x)$ to denote the first-order and second-order derivative of $f$. For a function of two vectors $g(\x,\y)$ for all $\x \in \mathbb{R^{n_1}}$ and $\y \in \mathbb{R^{n_2}}$, we let $\nabla_\x g(\mathbf{a},\mathbf{b}) \in \mathbb{R^{n_1}}$ denote the gradient of $g$ as a function of $\x$ for fixed $\y$ at $(\mathbf{a},\mathbf{b})$, and define $\nabla_\y$, $\nabla_{\x\x}^2$ and $\nabla_{\y\y}^2$ analogously. For $h: \mathbb{R^n} \to \mathbb{R^m}$ and $\x \in \mathbb{R^n}$, we let $\mathbf{J}_h(\x) \in \mathbb{R^{m \times n}}$ denote the Jacobian of $h$.

\textbf{Optimization:} We use  $\x^*$ to denote  the minimizer, i.e. $\x^*\overset{\triangle}=\argmin_\x f(\x)$ and   $f^*$ to denote its minimum value, i.e. $f^*\overset{\triangle}=\min_\x f(\x)$. We say a function $f$ is $L$-smooth (or has $L$-Lipschitz continuous gradients), if 
$        \|\nabla f(\x)-\nabla f(\y)\| \le L\|\x-\y\|, \forall \x,\y \in \R^d.$ We say a function $f$ is convex if
$     f(\y)\ge f(\x) + \langle \nabla f(\x), \y-\x\rangle + \frac{\mu}{2}\|\x-\y\|^2, \forall \x,\y \in \R^d,$         where $\mu\ge 0$. Moreover,  if $\mu>0$,  $f$ is said to be $\mu$-strongly convex.  We say $f\in \mathcal C^2$ has $H$-Hessian Lipschitz   continuous Hessian matrices if $\|\nabla^2 f(\x)-\nabla^2 f(\y)\|\le H\|\x-\y\|,  \forall \x,\y \in \R^d.$  We say $\xb$ is an $\epsilon$-minimizer of $f$ if $f(\xb) - f(\xb^*)\leq \epsilon$. And we say $\x$ is an $\epsilon$-approximate first-order stationary point of $f$ if $ \|\nabla f(\x)\|\le\epsilon$ and we say $\x$ is an $(\epsilon,\delta)$-approximate second-order stationary point of $f$ if $ \|\nabla f(\x)\|\le\epsilon$ and $\nabla^2 f(\x) \succeq -\delta$. 

\textbf{Sets:} We call $U \subseteq \mathbb{R^k}$ convex if $t\cdot \x + (1-t)\cdot \y \in U$ for all $\x,\y \in U$ and $t \in [0,1]$ and symmetric if $\x \in \mathbb{R^k}$ if and only if $-\x \in \mathbb{R^k}$. We use $[n]$ to denote the set $\{1,2,\cdots,n\}$ for any $n\in\mathbb{N}_+$. Use $\mathbb{S}^{d\times d}$ to denote the set of $d$-dimensional symmetric matrices and $\mathbb{S}^{d\times d}_+$ to denote $d$-dimensional positive definite matrices.

\section{Quadratic Optimization Problems}\label{sec: mainquadratic}
We start the analysis from quadratic functions for the following reasons. The eigenspace of quadratic functions stays invariant over the iterations, which paves the way towards a space-specific optimization. Further, most gradient methods ground on the quadratic optimizations locally, and quadratic functions appear to be the hard case for many gradient-based methods \cite{nesterov1998introductory}. Formally, the problem we consider writes
\begin{align}\label{eq: quadraticobj}
    \min_{\xb\in\mathbb{R}^n}\frac{1}{2}\xb^{\top}\Ab\xb + \bbb^{\top}\xb.
\end{align}
We propose an algorithm that utilizes the degeneracy of $\Ab$ and finds the optimal degeneracy level and magnitude adaptively. If $\Ab$ satisfies $\tr(\Ab^{\alpha}) = \tau_{\alpha}$ and $\lambda_{\min}(\Ab) = \mu > 0$, the algorithm guarantees a 
$\tilde{\mathcal{O}}\left(\min\left\{ \mu^{-\frac{1}{2}}, \tau_{\alpha}^{\frac{\alpha}{1+2\alpha}}\mu^{-\frac{\alpha}{1+2\alpha}}, d \right\}\right)$ gradient oracle complexity. And this improves the vanilla analysis of $\tilde{\mathcal{O}}\left(\frac{1}{\sqrt{\mu}}\right)$ or $\mathcal{O}(d)$. We also emphasize that our result adaptively finds the optimal $(\alpha, \tau_\alpha)$. 

Correspondingly, we establish the algorithmic lower bound on our oracle model. We show that under the interested setting, our gradient oracle complexity nearly matches the algorithmic lower bound. The lower bound indicates that we have also explored the full region where improvement upon the classical methods can be obtained.



\subsection{Eigen Extractor}\label{sec: mainquadratic_upperbound}

Our method decomposes into two stages: first, we use gradient oracles to detect the large eigenvalue space, then we handle two parts appropriately.

The first stage iteratively searches the largest eigenvalue of a series of matrices, and adaptively finds the optimal proportion of dimension that is part of the large eigenvalue space. Our setting is different from the common eigendecomposition results \cite{allen2016lazysvd,musco2015randomized,shamir2016fast} since we require (1) the large eigenspace being low-rank, (2) the decomposition taking an additive form, (3) the smallest eigenvalue of the remaining matrix changing up to a constant multiplicative error and (4) the algorithm should be adaptive and the largest eigenvalue change slowly. Specifically, we have the following theorem for finding the large eigenvalue space.

\begin{tcolorbox}[colback = blue!5!white, colframe=blue!75!black, title = \textbf{Theorem 1} (Eigen Extractor)]

\begin{thm}\label{thm: eigenextractor}
For a given positive definite matrix $\Ab\in\mathbb{R}^{d\times d}$ and $l\in\mathbb{N}_+$. Denote the $l$-th largest eigenvalue of $\Ab$ by $\lambda_l$. 
With high probability, we can find a rank $r = \tilde{\mathcal{O}}(l)$ matrix $\Ab_1 = \sum_{i=1}^r a_i\mathbf{v}_i\mathbf{v}_i^{\top}$ using $ \tilde{\mathcal{O}}(l)$ gradient oracle calls. $\Ab_1$ satisfies $\Ab - \Ab_1\succeq\frac{\lambda_d}{2}\Ib$ and $\|\Ab - \Ab_1\| = \mathcal{O}(\lambda_k)$.
\end{thm}

\end{tcolorbox}



Leveraging Theorem~\ref{thm: eigenextractor}, we can perform a downstream optimization, which is based on (1) a specialized proximal accelerated gradient, (2) accelerated gradient descent, and (3) conjugate gradient depending on the result from the eigen extractor. We now state our gradient oracle bound for optimizing quadratic functions.

\begin{tcolorbox}[colback = blue!5!white, colframe=blue!75!black, title = \textbf{Theorem 2} (Gradient Complexity for Quadratic Functions)]
\begin{thm}\label{thm: quadratic}
For any accuracy $\epsilon>0$, with high probability, there is an algorithm that finds an $\epsilon$-approximate minimizer of problem~\eqref{eq: quadraticobj} with $\tilde{\mathcal{O}}\left(\min_{k\in [d]}\left\{  k + \sqrt{\frac{\lambda_k}{\max\{\mu,\epsilon\}}}\right\}\right)$ gradient oracle calls, where $\lambda_k$ is the $k$-th largest eigenvalue of $\mathcal{\Ab}$.

If the function class is confined to $(\alpha,\tau_{\alpha})$-degenerated functions, the gradient oracle complexity is $\tilde{\mathcal{O}}\left(\min\left\{ \mu^{-\frac{1}{2}}, \tau_{\alpha}^{\frac{\alpha}{1+2\alpha}}\mu^{-\frac{\alpha}{1+2\alpha}}, d \right\}\right)$. Specifically, with high probability,
\begin{itemize}
    \item[a.] When $\max\{\mu, \varepsilon\}^{-\frac{1}{2}}\leq \tau_\alpha^{\alpha}$, one can find $\xb$ such that $f(\xb) \leq f(\xb^*) + \epsilon$ using $\tilde{\mathcal{O}}\left(\mu^{-1/2}\right)$ gradient oracle calls. 
    \item[b.] When $\max\{\mu,\varepsilon\}^{-\frac{1}{2}}\geq \tau_\alpha^{\alpha}$ and $\tau_{\alpha}^{\frac{\alpha}{1+2\alpha}}\max\{\mu,\varepsilon\}^{-\frac{\alpha}{1+2\alpha}} \leq d$, one can find $\xb$ such that $f(\xb) \leq f(\xb^*) + \epsilon$ using $\tilde{\mathcal{O}}\left(\tau_{\alpha}^{\frac{\alpha}{1+2\alpha}}\max\{\mu,\varepsilon\}^{-\frac{\alpha}{1+2\alpha}}\right)$ gradient oracle calls. \label{item: quadratic}
    \item[c.] When $\tau_{\alpha}^{\frac{\alpha}{1+2\alpha}}\max\{\mu,\varepsilon\}^{-\frac{\alpha}{1+2\alpha}} \geq d$, one can find $\xb$ such that $f(\xb) \leq f(\xb^*) + \epsilon$ using $\tilde{\mathcal{O}}(d)$ gradient oracle calls.
\end{itemize}
\end{thm}
\end{tcolorbox}

\begin{window}[0,l,\includegraphics[width=3in]{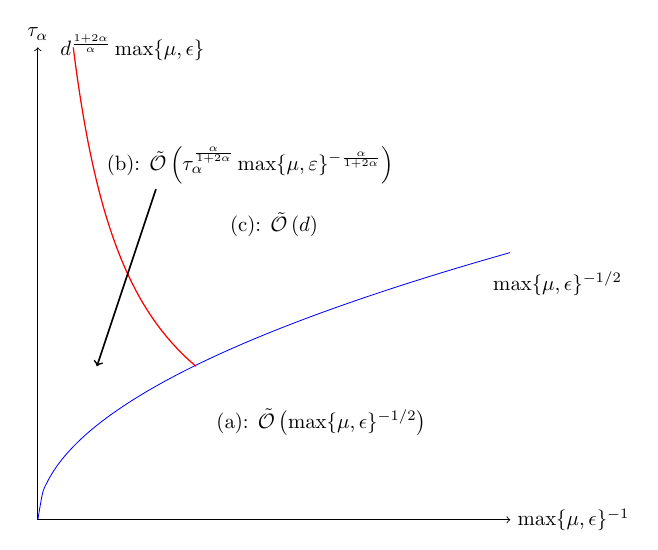},{}]
We emphasize that the algorithm does \textbf{not} need to acquire $(\alpha, \tau_{\alpha})$. Theorem~\ref{thm: quadratic} gives a thorough discussion on how the degeneracy of Hessian makes it possible to break the barrier of the classical lower bounds. If we confine the function class to be $(\alpha,\tau_{\alpha})$-degenerated, taking the high accuracy strongly convex example, when $\tau_{\alpha} \leq \mu^{-1/2}$, it is profitable to apply Theorem~\ref{thm: eigenextractor} to detect large eigenvalue space until the number of the gradient oracle calls come to $d$, entering the regime dominated by the conjugate gradient method \cite{hestenes1952methods}. At the wide middle regime, the improved convergence guarantees can be obtained and we will show that the gradient oracle complexity nearly matches the algorithmic lower bound. As a particular example, when $\alpha = 1$ and $\tau_{\alpha} = 1$, the theorem can be translated into a $\tilde{\mathcal{O}}\left(\mu^{-1/3}\right)$ convergence guarantee, which would be strictly faster than classical methods when $\frac{1}{d^3}\leq\mu\leq 1$. Again, we strengthen that the acceleration is not a consequence of the normalization of the problem but the fast drop of eigenvalues. The left picture demonstrates the three regions and the corresponding upper bound.

\end{window}

\subsection{Lower Bound}
Regarding the lower bound, we consider the oracle model consisting of randomized and adaptive calls of function gradient. The lower bound is confined to the function class of $(\alpha, \tau_{\alpha})$-degeneracy quadratic functions. To demonstrate that we achieve optimal oracle complexity in the entire region where improvement upon the classical methods can be obtained, we require a lower bound framework with a delicate dimension-accuracy relation.

We extend the seminal result of \cite{braverman2020gradient} to adapt to our framework. \cite{braverman2020gradient} reduces the oracle lower bound of optimization to the one of principle component analysis and constructs the randomized lower bound upon a Wishart distribution. Inspired by this, we construct a $3\times 3$ block diagonal random matrix for PCA lower bound, which can be reduced to a lower bound for solving quadratic problems. The theorem is stated below.

\begin{tcolorbox}[colback = blue!5!white, colframe=blue!75!black, title = \textbf{Theorem 3} (Lower Bound)]
\begin{thm}\label{thm: quadraticlowerbound}
Let $C$, $\mu_0$, $d_0$, $\delta<1$ be universal constants. For any $\mu\leq\mu_0$ and $d\geq d_0$, suppose the algorithm $\mathrm{Alg}$ outputs $\hat{\xb}$ such that
\begin{align*}
    \mathbb{P}_{\mathrm{Alg},\xb_0}\left( \|\hat{\xb} - \Ab^{-1}\bbb\|_{\Ab}^2 \leq C\mu\|\xb_0 - \Ab^{-1}\bbb\|^2  \right)\geq \delta
\end{align*}
for any objective $f(\xb) = \frac{1}{2}\xb^{\top}\Ab\xb + \bbb^{\top}\xb$ and initial point $\xb_0$ such that $\mathrm{tr}(\Ab^{\alpha}) \lesssim \tau_{\alpha}^{\alpha}$ and $\lambda_{\min}(\Ab) \gtrsim \mu$.
Then the gradient oracle calls of $\mathrm{Alg}$ should satisfy the following claims.
\begin{itemize}
    \item When $\mu^{-\frac{1}{2}} \leq \tau_{\alpha}^{\alpha}$, $\mathrm{Alg}$ requires at least $\tilde{\Omega}\left( \mu^{-\frac{1}{2}} \right)$ gradient oracle calls.
    \item  When $\mu^{-\frac{1}{2}} \geq \tau_{\alpha}^{\alpha}$ and $ \tau_{\alpha}^{\frac{\alpha}{1+2\alpha}}\mu^{-\frac{\alpha}{1+2\alpha}} \leq d $,  $\mathrm{Alg}$ requires at least $\tilde{\Omega}\left( \tau_{\alpha}^{\frac{\alpha}{1+2\alpha}}\mu^{-\frac{\alpha}{1+2\alpha}} \right)$ gradient oracle calls.
    \item When $ \tau_{\alpha}^{\frac{\alpha}{1+2\alpha}}\mu^{-\frac{\alpha}{1+2\alpha}} \geq d $, $\mathrm{Alg}$ requires at least $\tilde{\Omega}\left( d \right)$ gradient oracle calls.
\end{itemize}
\end{thm}
\end{tcolorbox}

\section{Generic Optimization Problems in Convex and Non-convex Setting}\label{sec: genericfunc}


Compared with linear regression or least square problems, generic convex and non-convex optimization problems occupy more important positions. For example, in machine learning,  training of  deep neural networks in general is a non-convex optimization problem. 

Concretely speaking, we extend our analysis of $(\alpha,\tau_\alpha)$-degenerated quadratic functions to generic convex and non-convex optimization. To be specific, we restrict the objective functions to have $H$-continuous Hessian matrices. 

We combine the analysis for quadratic function with the A-NPE framework \cite{monteiro2013accelerated}. 
With a binary search routine for hyper-parameters, optimizing the general convex functions reduces to optimizing a series of quadratic sub-problems. Similarly, we use Cubic Regularization Newton's Method \cite{nesterov_cubic_2006} and related techniques to find an $\left(\epsilon,\sqrt{H\epsilon}\right)$-approximate second-order stationary point for non-convex problems. For general convex objectives, we achieve $\tilde{\mathcal{O}}\left(\tau_\alpha^{\frac{\alpha}{1+2\alpha}} D^{\frac{14\alpha+12}{14\alpha+7}} H^\frac{2}{14\alpha+7}\epsilon^{-\frac{7\alpha+2}{14\alpha+7}} +  D^{\frac{6}{7}}H^{\frac{2}{7}}\epsilon^{-\frac{2}{7}}\right)$ gradient oracle complexity, and for non-convex optimization we can find an $\left(\epsilon,\sqrt{H\epsilon}\right)$-approximate second-order stationary point with $\tilde\cO\left(H^{\frac{1+\alpha}{2+4\alpha}}\tau_{\alpha}^{\frac{\alpha}{1+2\alpha}}\epsilon^{-\frac{3+7\alpha}{2+4\alpha}}  \right)$ gradient oracle calls. In this paper, we mainly focus on the gradient complexity improvement  based on our analysis of $(\alpha,\tau_\alpha)$-degenerated functions. We will consider designing simpler algorithms to solve such problems in future work. Moreover, to solve these problems it is not necessary to know the exact $\alpha$ and $\tau_\alpha$. Our designed Algorithm \ref{alg:A-HPEesgd} and Algorithm \ref{alg:cubicesgd} can adaptively solve the problem only based on local $(\alpha,\tau_\alpha)$.

We propose our result of finding an $\epsilon$-approximate solution of the general convex problems in Theorem \ref{thm:convex-m} and the result of finding an $\left(\epsilon,\sqrt{H\epsilon}\right)$-approximate second-order stationary point for the general non-convex problems in Theorem \ref{thm:non-convex-m} as below.
\subsection{Convex Objective Functions}
\begin{tcolorbox}[colback = blue!5!white, colframe=blue!75!black, title = \textbf{Theorem 4} (General Convex Setting)]
\begin{thm}
    Assume the $(\alpha,\tau_\alpha)$-degenerated objective function $f$ is convex and has $H$-continuous Hessian matrices. Under the same hyper-parameters setting in Algorithm \ref{alg:A-HPEesgd}, it  requires 
    \begin{equation}
        \tilde{\mathcal{O}}\left(\tau_\alpha^{\frac{\alpha}{1+2\alpha}} D^{\frac{14\alpha+12}{14\alpha+7}} H^\frac{2}{14\alpha+7}\epsilon^{-\frac{7\alpha+2}{14\alpha+7}} +  D^{\frac{6}{7}}H^{\frac{2}{7}}\epsilon^{-\frac{2}{7}}\right)
    \end{equation}
    gradient oracle calls to find an $\epsilon$-approximate solution, where 
    \begin{equation}
        D = \inf_{\x^* \in \Xb^*} \sup\left\{\|\x-\x^*\|: f(\x) \leq f(\x_0)\right\}.
    \end{equation}
    \label{thm:convex-m}
\end{thm}
\end{tcolorbox}
\begin{remark}
Based on the Hessian degeneracy and our analysis of the $(\alpha,\tau_\alpha)$-degenerated quadratic functions, we improve the gradient oracle complexity of solving a general convex optimization problem from $\tilde\cO\left(\epsilon^{-1/2} \right)$ to $\tilde{\mathcal{O}}\left(\tau_\alpha^{\frac{\alpha}{1+2\alpha}} D^{\frac{14\alpha+12}{14\alpha+7}} H^\frac{2}{14\alpha+7}\epsilon^{-\frac{7\alpha+2}{14\alpha+7}} +  D^{\frac{6}{7}}H^{\frac{2}{7}}\epsilon^{-\frac{2}{7}}\right)$. 
In a representative case where $\alpha = 1$, we obtain $\tilde\cO\left( \tau_{\alpha}^{1/3}\epsilon^{-3/7} \right)$ gradient oracle complexity.  Another vital fact we notice is that when $\alpha \to \infty$, the complexity of our algorithm matches the lower bound of the traditional result $\cO\left(\epsilon^{-1/2}\right)$, which means that our algorithm is strictly faster than the conclusions in classical convex optimization.
\end{remark}

\subsection{Non-convex Objective Functions}
\begin{tcolorbox}[colback = blue!5!white, colframe=blue!75!black, title = \textbf{Theorem 5} (Non-convex Setting)]
\begin{thm}
    Assume the $(\alpha,\tau_\alpha)$-degenerated objective function $f$ has $H$-continuous Hessian matrices. Under the corresponding hyper-parameters setting in Algorithm \ref{alg:cubicesgd}, it  requires 
    \begin{equation}
        \tilde\cO\left(H^{\frac{1+\alpha}{2+4\alpha}} \cdot \Delta \cdot
\tau_{\alpha}^{\frac{\alpha}{1+2\alpha}}\epsilon^{-\frac{3+7\alpha}{2+4\alpha}} \right)
    \end{equation}
    gradient oracle calls to find an $\left(\epsilon,\sqrt{H\epsilon}\right)$-approximate second-order stationary point, where 
    \begin{equation}
        D = \inf_{\x^* \in \Xb^*} \sup\left\{\|\x-\x^*\|: f(\x) \leq f(\x_0)\right\}, \quad \Delta = f(\x_0) - f^*.
    \end{equation}
    \label{thm:non-convex-m}
\end{thm}
\end{tcolorbox}
\begin{remark}
    Theorem \ref{thm:non-convex-m} indicates that our Algorithm \ref{alg:cubicesgd} can find an $\left(\epsilon,\sqrt{H\epsilon}\right)$-approximate second-order stationary point in $\tilde\cO\left(\epsilon^{-5/3}\right)$ in most cases when $\alpha = 1$, which is faster than the best-known result $\tilde\cO\left(\epsilon^{-7/4}\right)$ in \cite{jin2018accelerated}. When $\alpha \to \infty$, the lower bound of our complexity matches $\tilde\cO\left(\epsilon^{-7/4}\right)$, which means our algorithm is strictly faster than the result in \cite{jin2018accelerated}. Moreover, in Section \ref{sec:generic}, we give an example and prove that a two-layer neural network has $H$-continuous Hessian matrices. 
\end{remark}

\section{Empirical Risk Minimization}\label{sec: erm}
Empirical risk minimization (ERM) problems occur in many machine learning problems and typically take the form of \footnote{The regularization term does not affect the analysis since the magnitude of each eigenvalue only changes by $\mu$. The $\mu$-level change does not affect our results.}
\begin{align}\label{eq: regerm}
\min_\xb \frac{1}{n}\sum_{i=1}^n f_i(\ab_i^\top \xb) + \frac{\mu}{2}\|\xb\|^2. 
\end{align}
Note that the 
We seek methods that solve ERM problems with mild conditions and gain improvement from the degenerated Hessian assumption. Specifically, we are into two different aspects, data access complexity, and computation time.

The data access complexity is one of the bottlenecks for problems such as privacy \cite{dwork2006differential,dwork2014algorithmic} and distributed optimization \cite{scaman2017optimal,yue2023core}. We propose a mini-batch accelerated stochastic gradient method combined with the large eigenspace finding and achieve a state-of-the-art data access oracle for linear regression. For computation complexity, we achieve improved computation time based on the Hessian degenerated assumption. The method applies a similar routine as in the quadratic case to the dual problem and solves the proximal operator via the interior point method (IPM).


\subsection{Data Access Reduction}\label{sec: dataaccess}

When $f_i(\xb) = \frac{1}{2}(\ab_i^{\top}\xb - b_i)^2$, where $\ab_i,\xb\in\mathbb{R}^d$, we can step along the aforementioned idea, using Theorem~\ref{thm: eigenextractor} to find large effective dimensions. To leverage the finite-sum structure, we propose a mini-batch version of the accelerated stochastic gradient method and a batch-wise effective dimension finding. Our theorem on gradient oracle calls for finite sum setting is as follows.





\begin{tcolorbox}[colback = blue!5!white, colframe=blue!75!black, title = \textbf{Theorem 6} (Data Access Complexity for ERM)]
\begin{thm}\label{thm: dataaccessoracle}
    Consider optimizing problem~\ref{eq: regerm} with $f_i(\xb) = \frac{1}{2}(\ab_i^\top\xb - b_i)^2$. With normalized data $\|\ab_i\|\leq 1$, there is an algorithm that generates an $\varepsilon$-approximate minimizer of the problem with high probability, using $\tilde{\mathcal{O}}\left( n + n^{\frac{5}{6}}\mu^{-\frac{1}{3}} \right)$ data accesses and $\tilde{\mathcal{O}}\left( n + d^{\frac{5}{6}}\mu^{-\frac{1}{3}} \right)$ data accesses when $n\geq d$.
\end{thm}
\end{tcolorbox}
The data access complexity surpasses accelerated stochastic gradient methods \cite{allen2017katyusha,lin2015universal,zhang2015stochastic} when $\mu\geq \min\{n,d\}^{-2}$.

\subsection{Solve ERM with IPM Subroutine}


We discuss the prospect of combining our framework and interior point methods. Given ERM problem \eqref{eq: regerm}, for simplicity of the notation, define $\Ab = \left( \ab_1,\ab_2,\cdots,\ab_n \right)^\top$, and we impose a $1$-smoothness condition on $f_i(\xb)$. With a slight abuse of the notations, its dual problem writes
\begin{align}\label{eq: ERM-dual}
    \min_{\xb} &\sum_{i=1}^n f^*_i(\xb_i) + \frac{1}{2n\mu}\xb^\top\Ab\Ab^\top\xb,
\end{align}
where $\xb\in\mathbb{R}^n$.

Leveraging Theorem~\ref{alg: eigenextrator}, we separate $\ma\ma^\top$ into $\Ab_1$ and $\Ab_2$, where $\Ab_1 = \sum_{i = 1}^r p_i\ub_i\ub_i^\top$ is a $r \leq d$ rank matrix. We apply proximal accelerated gradient method \cite{beck2009fast} to problem~\eqref{eq: ERM-dual} via letting $g(\xb) = \frac{1}{\mu}\xb^\top\Ab_2\xb$ and $h(\xb) = \sum_{i=1}^n f^*_i(\xb_i) + \xb^\top\bbb + \frac{1}{2n\mu}\xb^\top\Ab_1\xb_i$. Given the assumption on the Lipschitz smoothness and convexity of $f_i$, $\sum_{i=1}^n f_i^*(\xb_i)$ is $1$-strongly convex and thus the iteration complexity of accelerated gradient descent is $\tilde{\mathcal{O}}\left( \sqrt{\frac{\tau_{\alpha}}{r^{\frac{1}{\alpha}}\mu}} \right)$. The proximal operator writes
\begin{align}\label{eq: ermagdproxmain}
    \arg\min_{\xb} \left\{ \sum_{i = 1}^n f^*(\xb_i) + \frac{1}{\mu}\sum_{i=1}^r p_i\left(\langle \ub_i, \xb \rangle\right)^2 + \frac{1}{2\mu r}\|\xb - \gammab\|^2 \right\},
\end{align}
where $\gamma \in \mathbb{R}^n$ is the one-step gradient descent from the previous iteration. A key observation on \eqref{eq: ermagdproxmain} is that the objective function is element-wise separable except for the $\frac{1}{\epsilon}\sum_{i=1}^r p_i\left(\langle \ub_i, \lambdab \rangle\right)^2$ component, which is low-rank quadratic form when $r \ll d$. One can leverage the interior point method \cite{leeS14,lsJournal19} to solve \eqref{eq: ermagdproxmain}. This leads to a $\tilde{\mathcal{O}}(\sqrt{r})$ iteration, amortised $ \tilde{\mathcal{O}}(r(n+d)+r^2)$-cost IPM. Optimizing $r$ to trade off the iteration number and amortized cost, we will demonstrate that our method improves upon the previous accelerated variance reduction  methods \cite{allen2017katyusha, zhang2015stochastic} in a wide regime. Our formal result for ERM is as follows.
\begin{tcolorbox}[colback = blue!5!white, colframe=blue!75!black, title = \textbf{Theorem 7} (Computation Time for ERM)]
\begin{thm}[]\label{thm: ERM-regu-complexity}
    Given a block-weight-function for IPM subroutine (Section~\ref{sec:ipm}), there is an algorithm optimizing (\ref{eq: regerm}) to an $\epsilon$-approximate minimizer with high probability in  
    \begin{align*}
        \tilde{\mathcal{O}}\left(ndr +\sqrt{\frac{\tau_{\alpha}}{\mu r^{\frac{1}{\alpha}}}}\left( nr^{1.5} + r^{2.5} + nd \right) \right)
    \end{align*}
    total computation time.
\end{thm}
\end{tcolorbox}
\begin{remark}

    When $\alpha\le \frac{1}{5}$, setting $r=\left(\sqrt{\frac{\tau}{\mu}}\cdot \frac1d\right)^\frac{2\alpha}{1-\alpha}
    \vee 
    \left(\sqrt{\frac{\tau}{\mu}}\cdot \frac1{nd}\right)^\frac{2\alpha}{1-3\alpha}
    \vee 
    \left(\sqrt{\frac{\tau}{\mu}}\right)^{\frac{2\alpha}{2\alpha+1}}$,
    the total computation complexity is 
    $$
    \tilde\cO\left(nd\left(\left(\sqrt{\frac{\tau}{\mu}}\cdot \frac1d\right)^\frac{2\alpha}{1-\alpha}
    \vee 
    \left(\sqrt{\frac{\tau}{\mu}}\cdot \frac1{nd}\right)^\frac{2\alpha}{1-3\alpha}
    \vee 
    \left(\sqrt{\frac{\tau}{\mu}}\right)^{\frac{2\alpha}{2\alpha+1}}\right)\right).
    $$

    When $\alpha>\frac{1}{5}$, in the regime of $d\ge n^{3/2}$, setting $ r=\left(\frac{\tau}{\mu}\right)^{\frac{\alpha}{2\alpha+1}} \wedge (nd)^{2/5}$, the total computation complexity is 
    $$
    \tilde \cO\left( nd\sqrt{\frac{\tau}{\mu}}\left( \sqrt{\frac{\tau}{\mu}}^{-\frac{1}{2\alpha+1}} \vee (nd)^{-\frac{1}{5\alpha}} \right) \right).
    $$

    Specifically, when $\alpha\le 1/5, d\ge n^{3/2}$ and $ \frac{\tau}{\mu}\ge(nd)^{\frac{4\alpha+2}{5\alpha}} $, we reach $\tilde\cO \left((nd)^{\frac{1-5\alpha}{1-3\alpha}}\left( \frac{\tau}{\mu} \right)^{\frac{\alpha}{1-3\alpha}}\right)$ complexity.
    




\end{remark}

\section{Conclusion}
It is always our ultimate goal to design provably faster algorithms that work on practical problems. An ideal algorithm is expected to adapt to the complexity of a particular objective function and incur faster rates for simpler problems. This paper studies the potential based on the degeneracy of Hessian matrices of the objective function by introducing the factors  $(\alpha, \tau_{\alpha})$. We then propose several provably better and adaptive algorithms that can fit the underlying structure without pre-known the knowledge.  The algorithms also improve the state-of-art complexities for
several problems in the field of machine learning under suitable settings.

Several directions remain to be studied in the future. (A) Based on the pioneering work of \cite{carmon2017convex},  can we  simplify our algorithm for the generic objective function into a single loop? (B) Can we extend our algorithm to the online setting where data is given on the fly? (C) For the task of neural network training, can we propose a more efficient algorithm? (D) More broadly, it remains open to provide a further refined description of the optimization. 

We would say that understanding the difficulties of
modern problems, such as training deep neural networks, is still in its infancy. One promising potential is to study the instance complexity based on the intrinsic structure of the problems. We hope this work could invoke rethinking and inspire new design and analysis for more efficient optimization algorithms.

\bibliographystyle{alpha}
\bibliography{main}

\newpage
\tableofcontents
\newpage

\section{Details of Quadratic Optimization Problems}\label{sec: quadratic}
This section centers around the quadratic optimization problem, which appears to be the hard case for multiple gradient-based optimization problems \cite{nesterov1998introductory,simchowitz2018randomized,yue2023zeroth}. In this section, we present our quadratic optimization algorithm and complete the proof in Section~\ref{sec: mainquadratic}. The problem formally writes as $\min_{\xb\in{\mathbb{R}^d}} f(\xb) $ where $f(\xb) = \frac{1}{2}\xb^\top\Ab\xb + \mathbf{b}^\top\xb$. Here $\Ab\in\mathbb{R}^{d\times d}$. 

\subsection{Proof of Theorem~\ref{thm: eigenextractor}}
We prove an extended version of Theorem~\ref{thm: eigenextractor} stated as follows.

\begin{thm}[Eigen Extractor, Generalization of Theorem~\ref{thm: eigenextractor}]\label{thm: eigenextractorgen}
For a given positive definite matrix $\Ab\in\mathbb{R}^{d\times d}$ and $l\in[d]$. Denote the $l$-th largest eigenvalue of $\Ab$ by $\lambda_l$. For any accuracy $\epsilon = \mathcal{O}(a_l)$, with high probability, we find a rank $r = \Omega(l)$ matrix $\Ab_1 = \sum_{i=1}^r a_i\mathbf{v}_i\mathbf{v}_i^{\top}$ using $\tilde{\mathcal{O}}(l)$ gradient oracle calls. $\Ab_1$ satisfies $\Ab - \Ab_1\succeq\left( \lambda_d(\Ab) - \epsilon\right)\Ib$ and $\|\Ab - \Ab_1\| = \mathcal{O}(\lambda_k)$.

\end{thm}


To unveil our eigen extractor theorem, we introduce a classical 1-PCA algorithm as a preliminary. We leverage the shift-and-invert paradigm \cite{garber2015fast}, a well-known reduction from finding the leading eigenvector to solving a linear system. It approximates the leading eigenvector of $\Mb$ by applying power method on $(\gamma \Ib - \Mb)^{-1}$, where $\gamma \Ib - \Mb \succeq \mathbf{0}$. Power method with iterations $\wb_{k+1} = (\gamma \Ib - \Mb)^{-1}\wb_k$ can be implemented by solving a linear system (equal to minimizing $\frac{1}{2}\xb^\top(\gamma \Ib - \Mb)\xb - \wb_k^\top\xb$). We use a key result regarding the shift-and-inverse algorithm which is proved in \cite{garber2015fast,allen2016lazysvd}. The lemma shows an error analysis of the shift-and-inverse algorithm on matrix $\Ab$ that takes argument $\delta, \epsilon$, and we will present it without proof as in Lemma~\ref{lm: shiftninverse}. 




\begin{lem}[Shift-and-inverse guarantee, Thoerem 1.1 of \cite{garber2015fast}, Theorem 3.1 of \cite{allen2016lazysvd} (gap-free)]\label{lm: shiftninverse}
    Let $\Ab\in\mathbb{R}^{n\times n}$ be a symmetric matrix with eigenvalues $\lambda_1\geq\lambda_2\geq\cdots\geq \lambda_n$ and corresponding eigenvector $\ub_1, \ub_2, \cdots, \ub_n$. 
    
    Then with high probability, the shift-and-inverse algorithm taking argument $\delta, \epsilon$ produces a $\wb_f$ such that
    \begin{align}\label{eq: tempsni}
        \sum_{\lambda_i\leq \lambda_1(1-\delta)}(\wb_f^\top\ub_i)^2 \leq \epsilon, \quad \text{and}\quad \wb^\top\Ab\wb\geq(1-\delta)(1-\epsilon)\lambda_1
    \end{align}
    with $\tilde{\mathcal{O}}(\sqrt{\frac{1}{\delta}})$ gradient oracle calls.

    Further, suppose there exists an algorithm solving quadratic function with quadratic term $\mathcal{\Ab}$ and accuracy $\epsilon$ within gradient oracle calls $\mathcal{C}\left( \Ab, \epsilon \right)$, then there exists an algorithm find $\ub$ in time $\mathcal{C}\left( \Ab, \delta \right)$.
\end{lem}

\begin{lemma}[Byproduct of Theorem 3.1 of \cite{allen2016lazysvd}]\label{lm: findmu}
If there exists an algorithm that solves the quadratic optimization problem to an $\epsilon$-approximate minimizer of a quadratic problem with quadratic term $\Ab$ using gradient oracle calls $\mathcal{C}(\Ab,\epsilon)$, then there exists an algorithm that finds the smallest eigenvalue $\hat{\mu}$ of matrix $\Ab$ up to a constant level multiplicative error using $\tilde{\mathcal{O}}(\mathcal{C}(\Ab,\mu))$ gradient oracle calls.
\end{lemma}

Provided the 1-PCA analysis, we design an algorithm that iteratively searches the largest eigenvalue of a series of matrices and adaptively finds the optimal proportion of dimension that is part of the large eigenvalue space. Our method is presented as Algorithm~\ref{alg: eigenextrator}.


\begin{algorithm}[!h]
    \caption{Eigen Extractor}
    \label{alg: eigenextrator}
    {\bfseries Input: }$\Ab$, accuracy/ strongly convex parameter $\epsilon$, $\lambda_l (\text{in option 1})$\\
        $\Ab_1 \leftarrow \Ab$;\\
        $k \leftarrow 0$;\\
        Set: $\epsilon_0 = \tilde{\mathcal{O}}(\frac{\epsilon}{d^2})$ and $\delta \leftarrow \frac{1}{2}$;\\
        \While{True}{
            $k \leftarrow k+1$;\\
            Apply shift-and-inverse with $\delta$ and $\epsilon_0$ to find the approximated eigenvector $\mathbf{v}_k$;\\
            $a_k \leftarrow \mathbf{v}_k^\top\Ab\mathbf{v}_k$;\\
            $\Ab_{k+1} \leftarrow \Ab_k - \frac{a_k}{5}\mathbf{v}_k\mathbf{v}_k^\top $.
            \If{$a_k \leq \mathcal{O}(\lambda_l) $ (criteria 1)}{
                {\bfseries Break;} \quad\quad\quad\quad\quad\quad \textit{$\#$(option 1,\ in the proof on Theorem~\ref{thm: eigenextractor})}
            }
            \If{$k = \Omega\left(\sqrt{\frac{a_k}{\epsilon}}\right)$ or $k = \Omega(d)$ or $k = \Omega\left(\epsilon^{-\frac{1}{2}}\right)$ (criteria 2)}{
                {\bfseries Break;}\quad\quad\quad\quad\quad\quad \textit{$\#$(option 2,\ adaptive algorithm)}
            }
        }
        {\bfseries Output:} $a_1,a_2,\cdots,a_k$, $\mathbf{v}_1,\mathbf{v}_2,\cdots,\mathbf{v}_k$ and $\Ab_1 = \sum_{i=1}^k \frac{a_i}{5}\mathbf{v}_i\mathbf{v}_i^{\top}$.
\end{algorithm}

We demonstrate that Algorithm~\ref{alg: eigenextrator} achieves the claimed result in Theorem~\ref{thm: eigenextractorgen}.

\begin{proof}
We show the invoking Algorithm~\ref{alg: eigenextrator} for $\Ab$ with option 1 can obtain the claimed result.

We begin by proving the strong convexity of the output. Denote the eigenvalue decomposition of $\Ab_k$ by $\Ab_k = \Ub_k\Lambdab_k\Ub_k^{\top}$ and writes $a_k = \mathbf{v}_k^{\top}\Ab_k\mathbf{v}_k$ for simplicity. For any $\gammab = \Ub_k\thetab$, where $\thetab,\gammab\in\mathbb{R}^d$, we have
\begin{align*}
    \gammab^{\top}\left( \Ab_k - \frac{1}{5}(\mathbf{v}_k^{\top}\Ab_k\mathbf{v}_k)\mathbf{v}_k\mathbf{v}_k^{\top} \right) \gammab = \thetab^{\top}\Lambda_k\thetab - \frac{a_k}{5} \left(\mathbf{v}_k^{\top}\Ub\thetab\right)^2.
\end{align*}

Lemma~\ref{lm: shiftninverse} indicates that we can acquire an approximation of the leading vector that (1) brings an arbitrarily small perturbation to the space the corresponding eigenvalues of which are below a multiplicative error,(2) only increases the burden by a logarithmic factor. We leverage this to analyze the eigenvalue change. 

Let $\Lambdab_k^{(1)}$ be a diagonal matrix that contains the eigenvalue that is greater than $\frac{a_k}{2}$, and $\Lambdab_k^{(2)}$ be the diagonal matrix that contains the remaining eigenvalue. Let $\Ub_k = (\Ub^{(1)}_k, \Ub^{(2)}_k)$ with each is the corresponding eigenvalue of $\Lambdab_k^{(1)}$ and $\Lambdab_k^{(2)}$,  respectively. Further, we divide $\thetab = (\thetab_1^{\top},\thetab_2^{\top})^{\top}$ accordingly. Then we have the control 
\begin{align*}
    \gammab^{\top}\left( \Ab_k - \frac{1}{5}(\mathbf{v}_k^{\top}\Ab_k\mathbf{v}_k)\mathbf{v}_k\mathbf{v}_k^{\top} \right) \gammab = & \thetab_1^{\top}\Lambdab_k^{(1)}\thetab_1 + \thetab_2^{\top}\Lambdab_k^{(2)}\thetab_2 - \frac{a_k}{5}\left( \mathbf{v}_k^\top\Ub_1\thetab_1 + \mathbf{v}_k^{\top}\Ub_2\thetab_2  \right)^2\\
    \geq & \frac{a_k}{2}\|\thetab_1\|^2 + \lambda_{\min}(\Ab_k)\|\thetab_2\|^2 - \frac{2a_k}{5}\left( \mathbf{v}_k^{\top}\Ub_1\thetab_1  \right)^2 - \frac{2a_k}{5}\left( \mathbf{v}_k^{\top}\Ub_2\thetab_2  \right)^2\\
    \geq & \left(\frac{a_k}{2} - \frac{2a_k}{5} \right)\| \thetab_1 \|^2 + \left(\lambda_{\min}(\Ab_k) - \|\mathbf{v}_k^\top\Ub_2\|^2\right)\|\thetab_2\|^2\\
    \geq & \min\left\{ \frac{a_k}{10},  \lambda_{\min}(\Ab_k) - \|\mathbf{v}_k^\top\Ub_2\|^2\right\}\|\thetab\|^2\\
    \geq & \left(\lambda_{\min}(\Ab_k) - \|\mathbf{v}_k^\top\Ub_2\|^2\right)\|\theta\|^2.
\end{align*}
Thus we have $\lambda_{\min}(\Ab_k) \geq \lambda_{\min}(\Ab) - k\varepsilon^2_0 \geq \lambda_{\min}(\Ab) - \epsilon$.

Then we show that the algorithm uses $\tilde{\mathcal{O}}(l)$ gradient oracles to find the $\Ab_k$ such that $\|\Ab_k\| = \mathcal{O}(\lambda_l)$. We define $\rho = \frac{2}{(1-\delta)(1-\epsilon_0)}$, which is a constant given our choice of $\varepsilon_0$ and $\delta$. Find $r$ such that $\lambda_1\in[\lambda_l\rho^{r-1}, \lambda_l\rho^r)$, then $r = \tilde{\mathcal{O}}(1)$. Provided the above terms, we prove by induction that with at most $\tilde{\mathcal{O}}\left(l\right)$ calls of shift-and-inverse algorithm, we can obtain the claimed result.

Concretely speaking, we inductively prove the following statement for $i = r,r-1,\cdots,2$
\begin{itemize}
    \item[1.] $\lambda_{\max}(\Ab_{k_i}) \leq \lambda_l\rho^{i}$,
    \item[2.] $k_{i} - k_{i+1}=\tilde{\mathcal{O}}(l)$. 
\end{itemize}
Note that $r = \tilde{\mathcal{O}}(1)$, and each iteration we call the shift-and-inverse algorithm once, which only consumes $\tilde{\mathcal{O}}(1)$ gradient oracles. Thus the gradient oracle bound immediate follows from the above claim.

To prove the claim inductively, base case $i=r$ is obvious given that $\lambda_1 < \lambda_l\rho^r$. Suppose that the claims establish for $i+1$. We track the following positive term along the iterations


\begin{align*}
    \mathcal{S}_i^{(k)} = \sum_{\lambda_j^{(k)} \geq \rho^{i-1}\lambda_l} \left[\lambda_j^{(k)} - \rho^{i-1}\lambda_l\right]_+,
\end{align*}
where $[\cdot]_+ = \max\{\cdot,0\}$ and $\lambda_j^{(k)}$ is the $j$-th eigenvalue of $\Ab_k$. Suppose $\Vb$ being a column orthogonal matrix. Its columns $\Vb_i$ are eigenvectors of $\Ab_k$ such that corresponding eigenvalue $\Vb_i^\top\Ab_k\Vb_i\geq\rho^{i-1}\lambda_l$. Denote $\Pb_{\Vb}$ be the projection matrix on $\Vb$, then
\begin{align*}
    \Ab_k - \frac{a_k}{5}\mathbf{v}_k^\top\mathbf{v}_k = & \Ab_k -\frac{a_k}{5}\frac{\Pb_{\Vb}\mathbf{v}_k \mathbf{v}_k^\top\Pb_{\Vb}}{\|\Pb_{\Vb}\mathbf{v}_k\|^2} + \frac{a_k}{5}\frac{\Pb_{\Vb}\mathbf{v}_k \mathbf{v}_k^\top\Pb_{\Vb}}{\|\Pb_{\Vb}\mathbf{v}_k\|^2} - \frac{a_k} {5}\mathbf{v}_k^\top\mathbf{v}_k\\
    \leq & \Ab_k -\frac{a_k}{5}\frac{\Pb_{\Vb}\mathbf{v}_k \mathbf{v}_k^\top\Pb_{\Vb}}{\|\Pb_{\Vb}\mathbf{v}_k\|^2} + \frac{a_k}{5}\max_{\ub\in\mathbb{R}^d}\left( \left\langle \frac{\Pb_{\Vb}\mathbf{v}_k}{\|\Pb_{\Vb}\mathbf{v}_k\|}, \ub\right\rangle^2 - \langle \mathbf{v}_k, \ub\rangle^2 \right)\Ib.
\end{align*}

And the latter maximum terms can be bounded by
\begin{align*}
    \left\langle \frac{\Pb_{\Vb}\mathbf{v}_k}{\|\Pb_{\Vb}\mathbf{v}_k\|}, \ub\right\rangle^2 - \langle \mathbf{v}_k, \ub\rangle^2 \leq & \left\|\frac{\Pb_{\Vb}\mathbf{v}_k}{\|\Pb_{\Vb}\mathbf{v}_k\|} + \mathbf{v}_k\right\|\ \ \left\| \frac{\Pb_{\Vb}\mathbf{v}_k}{\|\Pb_{\Vb}\mathbf{v}_k\|} - \mathbf{v}_k \right\|\\
    \leq  & 2\left\|\frac{\Pb_{\Vb}\mathbf{v}_k}{\|\Pb_{\Vb}\mathbf{v}_k\|} - \mathbf{v}_k\right\|\\
    \leq & 2\left(1 - \left\| \Pb_{\Vb}\mathbf{v}_k \right\|\right) + 2\|\Pb_{\Vb^\perp}\mathbf{v}_k\|\\
    \leq & 2\varepsilon_0 + 2\sqrt{\varepsilon_0}.
\end{align*}
By Weyl's inequality, there are at most $l$ eigenvalues lying on the intervals $[\rho^{i-2}\lambda_l,\lambda_l\rho^{i+1})$. Combining the above two controls,
and thus when $\lambda_{\max}(\Ab_k) \geq \lambda_l\rho^i$,  the change of $\mathcal{S}_i^{(k)}$ over iterations can be bounded by
\begin{align*}
    \mathcal{S}_i^{(k+1)} \leq & \mathcal{S}_i^{(k)} + \frac{2la_k}{5}(\varepsilon_0 + \sqrt{\varepsilon_0}) - \frac{a_k}{5}\\
    \leq & \mathcal{S}_i^{(k)} + \frac{4l}{5}\sqrt{\varepsilon_0}\lambda_{\max}(\Ab_k) - \frac{2}{5\rho}\lambda_{\max}(\Ab_k)\\
    \leq & \mathcal{S}_i^{(k)} + \frac{2l}{5}\sqrt{\varepsilon_0}\lambda_l\rho^{i+1} - \frac{2}{5\rho}\lambda_l\rho^i\\
    \leq & \mathcal{S}_i^{(k)} -\frac{1}{5} \lambda_l\rho^{i-1}.
\end{align*}
Given that $\mathcal{S}_i^{(k_i)} \leq l\lambda_l(\rho^{i+1} - \rho^{i-1})$ and $\rho$ is a constant, after $\tilde{k} = \mathcal{O}\left( l \right)$ iterations, $\lambda_{\max}(\Ab_{k_i+\tilde{k}}) \leq \lambda_l\rho^{i}$, otherwise  $\mathcal{S}_{i}^{(k_i+\tilde{k})}$ would be negative.

\end{proof}

Theorem~\ref{thm: quadratic} follows from combining Theorem~\ref{thm: eigenextractor} and an eigenvalue control under the Hessian degeneracy condition. We formally write the our algorithm in Section~\ref{sec: mainquadratic_upperbound} as Algorithm~\ref{alg: esgd} and restate Theorem~\ref{thm: quadratic}.

\begin{thm}[Gradient Complexity for Quadratic Functions, Theorem~\ref{thm: quadratic} restated]
For any accuracy $\epsilon>0$, with high probability, there is an algorithm that finds an $\epsilon$-approximate minimizer of problem~\eqref{eq: quadraticobj} with $\tilde{\mathcal{O}}\left(\min_{k\in [d]}\left\{  k + \sqrt{\frac{\lambda_k}{\max\{\mu,\epsilon\}}}\right\}\right)$ gradient oracle calls, where $\lambda_k$ is the $k$-th largest eigenvalue of $\mathcal{\Ab}$.

If the function class is confined to $(\alpha,\tau_{\alpha})$-degenerated functions, the gradient oracle complexity is $\tilde{\mathcal{O}}\left(\min\left\{ \mu^{-\frac{1}{2}}, \tau_{\alpha}^{\frac{\alpha}{1+2\alpha}}\mu^{-\frac{\alpha}{1+2\alpha}}, d \right\}\right)$. Specifically, with high probability,
\begin{itemize}
    \item[a.] When $\max\{\mu, \varepsilon\}^{-\frac{1}{2}}\leq \tau_\alpha^{\alpha}$, one can find $\xb$ such that $f(\xb) \leq f(\xb^*) + \epsilon$ using $\tilde{\mathcal{O}}\left(\mu^{-1/2}\right)$ gradient oracle calls. 
    \item[b.] When $\max\{\mu,\varepsilon\}^{-\frac{1}{2}}\geq \tau_\alpha^{\alpha}$ and $\tau_{\alpha}^{\frac{\alpha}{1+2\alpha}}\max\{\mu,\varepsilon\}^{-\frac{\alpha}{1+2\alpha}} \leq d$, one can find $\xb$ such that $f(\xb) \leq f(\xb^*) + \epsilon$ using $\tilde{\mathcal{O}}\left(\tau_{\alpha}^{\frac{\alpha}{1+2\alpha}}\max\{\mu,\varepsilon\}^{-\frac{\alpha}{1+2\alpha}}\right)$ gradient oracle calls. 
    \item[c.] When $\tau_{\alpha}^{\frac{\alpha}{1+2\alpha}}\max\{\mu,\varepsilon\}^{-\frac{\alpha}{1+2\alpha}} \geq d$, one can find $\xb$ such that $f(\xb) \leq f(\xb^*) + \epsilon$ using $\tilde{\mathcal{O}}(d)$ gradient oracle calls.
\end{itemize}
\end{thm}

\begin{proof}
We only show that Algorithm~\ref{alg: esgd} can achieve the claimed result for high precision case (i.e. result depending on $\mu$). And the high precision case result can be extended to a convergence guarantee depending on $\max{\epsilon, \mu}$ using the proximal point method.

First, we have that for any $l\in[d]$,
\begin{align*}
    \tau_{\alpha}^{\alpha} = \tr(\Ab^{\alpha}) \geq \sum_{i=1}^l\lambda_i^\alpha \geq l\lambda_l^{\alpha},
\end{align*}
and thus $\lambda_l\leq\frac{\tau_{\alpha}}{l^{\frac{1}{\alpha}}}$. Therefore by Theorem~\ref{thm: eigenextractor}, when the iteration $k$ of Algorithm~\ref{alg: eigenextrator} satisfies that $k = \tilde{\Omega}\left( \tau_{\alpha}^{\frac{\alpha}{1+2\alpha}}\mu^{-\frac{\alpha}{1+2\alpha}} \right)$, then there exists $l = \tilde{\Omega}(k)$
\begin{align*}
    \sqrt{\frac{a_k}{\mu}} = \tilde{\mathcal{O}}\left( \sqrt{\frac{\lambda_l}{\mu}}\right) \leq \tilde{\mathcal{O}}\left(\sqrt{\frac{\tau_{\alpha}}{l^{\frac{1}{\alpha}}\mu}}\right) = \tilde{\mathcal{O}}\left( \sqrt{\frac{\tau_{\alpha}}{k^{\frac{1}{\alpha}}\mu}} \right)
\end{align*}
This indicates that $k = \tilde{\Omega}\left(\sqrt{\frac{\tau_{\alpha}}{k^{\frac{1}{\alpha}}\mu}} \right) = \tilde{\Omega}\left(\sqrt{\frac{a_k}{\mu}}\right)$ when $k = \tilde{\Omega}\left( \tau_{\alpha}^{\frac{\alpha}{1+2\alpha}}\mu^{-\frac{\alpha}{1+2\alpha}} \right)$. Thus Algorithm~\ref{alg: eigenextrator} stops with no more than $\tilde{\mathcal{O}}\left( \tau_{\alpha}^{\frac{\alpha}{1+2\alpha}}\mu^{-\frac{\alpha}{1+2\alpha}} \right)$ iterations, with each iteration consuming $\tilde{\mathcal{O}}(1)$ gradient oracle calls. The stopping criteria $k = \Omega\left( \sqrt{\frac{a_k}{\mu}} \right)$ and classical accelerated proximal gradient descent method analysis indicate an overall $\tilde{\mathcal{O}}\left( \tau_{\alpha}^{\frac{\alpha}{1+2\alpha}}\mu^{-\frac{\alpha}{1+2\alpha}} \right)$ gradient complexity upper bound.

Further, stopping criteria $k = \Omega\left( \mu^{-\frac{1}{2}} \right)$ and classical accelerated gradient analysis guarantee a $\tilde{\mathcal{O}}\left( \mu^{-\frac{1}{2}} \right)$ gradient oracle complexity upper bound. Stopping criteria $k = \Omega\left( d \right)$ and conjugate gradient analysis guarantee a $\tilde{\mathcal{O}}\left( d \right)$ gradient oracle complexity upper bound.

Combining these bounds yields the claimed overall $\tilde{\mathcal{O}}\left( \min\left\{ \mu^{-\frac{1}{2}}, \tau_{\alpha}^{\frac{\alpha}{1+2\alpha}}\mu^{-\frac{\alpha}{1+2\alpha}}, d \right\}\right)$ gradient complexity guarantee.



\end{proof}


\begin{algorithm}
    \caption{Accelerated Gradient Method with Adaptive Subspace Search (AGMAS)}
    \label{alg: esgd}
    {\bfseries Input:} $f(\xb) = \frac{1}{2}\xb^{\top}\Ab\xb + \bbb^{\top}\xb$, accuracy $\epsilon$, strongly convex parameter $\mu$;\\
    Invoke option 2 of Algorithm~\ref{alg: eigenextrator} with $\Ab$ and strongly convex parameter $\mu/2$ to obtain $\Ab_1$;\\
    \If{Algorithm~\ref{alg: eigenextrator} stop with $k = \Omega\left( \sqrt{\frac{a_k}{\mu}} \right)$}{
    Obtain $\hat{\xb}$ by proximal AGD with $f(\xb) = \frac{1}{2}\xb^{\top}(\Ab - \Ab_1)\xb + \bbb^{\top}\xb$ and $h(\xb) = \frac{1}{2}\xb^{\top}\Ab_1\xb$ to an $\varepsilon$-approximate minimizer;}
    \If{Algorithm~\ref{alg: eigenextrator} stop with $k = \Omega\left( \mu^{-\frac{1}{2}} \right)$}{
    Obtain $\hat{\xb}$ by accelerated gradient descent on $f$ to an $\varepsilon$-approximate minimizer;
    }
    \If{Algorithm~\ref{alg: eigenextrator} stop with $k = \Omega\left( d \right)$}{
    Obtain $\hat{\xb}$ by conjugate gradient method on $f$ to an $\varepsilon$-approximate minimizer;
    }
    {\bfseries Output:} $\hat{\xb}$.
\end{algorithm}

\subsection{Lower Bound for Quadratic Functions}

In this section, we formalize our lower bound setting and prove Theorem~\ref{thm: quadraticlowerbound}. We investigate the algorithm class using adaptive and randomized gradient oracles and the objective function is constrained to quadratic functions. Specifically, we consider the randomized algorithms that maps function $f:\mathbb{R}^d\rightarrow \mathbb{R}$ to a sequence of iterations
\begin{align*}
    \xb^k = \mathcal{A}^{k-1}\left(\mathrm{\xi}, \nabla f(\xb_0),\cdots,\nabla f(\xb_{k-1}) \right),
\end{align*}
where $\xi$ is a distribution over $[0,1]$ representing the randomness of the algorithm. And the function class $\mathcal{F}_{\tau_{\alpha}, \mu} = \{f(\xb) = \frac{1}{2}\xb^{\top}\Ab\xb + \bbb^{\top}\xb: \ \left(\tr(\Ab^{\alpha})\right)^{\frac{1}{\alpha}} = \tau_{\alpha}, \lambda_{\min}(\Ab) = \mu \}$.

Our lower bound construction follows a series of seminal works \cite{simchowitz2018randomized, braverman2020gradient}, utilizing the shift-and-inverse reduction and a Wishart random matrix construction. We step along the path of \cite{braverman2020gradient} in our specific setting, where Lemma~\ref{prop: reductiontoop} states that one can reduce the problem of approximating the leading eigenvector to solving $\mathcal{O}(1)$ quadratic functions with each $\mu = \mathcal{O}(\frac{1}{\mathrm{gap}})$, and Theorem~\ref{thm: pcalowerbound} establish the algorithmic lower bound of finding leading eigenvector. We adapt the analysis in two major extensions, (1) accommodating to the analysis in terms of effective dimensions, (2) making the lower bound flexible to any ratio of $\tau_{\alpha}, \mu$ and $d$. Throughout the proof, $c_1,c_2,\cdots$ are universal constants, and they may represent different constants as in different contexts.

\begin{lemma}\label{prop: reductiontoop}
Let $(d,\tau_{\alpha},\mu)$ satisfies that $d\geq d_0, \mu\leq \mu_0$ and $\tau_{\alpha}^{\alpha}\geq d\mu^{\alpha}$ (eigenvalue constrain) for some universal constants $d_0$ and $\mu_0$. $C_0\leq 1$ and $\delta_0$ are universal constants. Suppose that $\mathrm{Alg}$ solving a linear system with output $\hat{\xb}$ satisfies
\begin{align*}
    \mathbb{P}_{\mathrm{Alg}, \xb_0}\left( \|\hat{\xb}_0 -\Ab^{-1}\bbb\|_{\Ab}^2 \leq C_0\mu\|\xb_0 - \Ab^{-1}\bbb\|^2 \right) \geq 1 - \delta_0,
\end{align*}
for any $\Ab\in\mathbb{S}_+^{d\times d}$ satisfying $\tr(\Ab^{\alpha}) \lesssim \tau_{\alpha}^{\alpha}$, $\lambda_{\min}(\Ab)\gtrsim \mu$ and starting point $\xb_0\in\mathbb{R}^d$. And $\mathrm{Alg}$ uses $\mathcal{T}$ gradient oracle calls. Then for any $\delta\in(0,\delta_0)$ and constant $c$, there exists an algorithm $\mathrm{Alg}_{\mathrm{eig}}$ with output $\hat{b}$ satisfying
\begin{align}\label{eq: reduct2}
    \mathbb{P}\left(\hat{\bbb}^{\top}\Mb  \hat{\bbb} \geq (1-c\mathrm{gap})\lambda_1(\Mb))\right) \geq 1-\delta
\end{align}
for any $\Mb$ satisfying that (a) $\mathrm{gap}(\Mb)\gtrsim c_1\mu$, (b) $\mathrm{tr}\big( ((1+c\mathrm{gap})\Ib - \Mb)^{\alpha} \big) \lesssim \tau_{\alpha}^{\alpha}$ and (c) $\mathbf{0}\preceq \Mb \preceq (1-c_1\mathrm{gap})\Ib$. And its gradient oracle call number is bounded by $\tilde{O}\left( \mathcal{T} \right)$.

\end{lemma}
We omit the proof Lemma~\ref{prop: reductiontoop} since it is a direct extension of Proposition~8 in \cite{braverman2020gradient}. It is an application of the shift-and-inverse method and noise power method \cite{hardt2014noisy}. Apart from the notation change, the main difference to Proposition~8 in \cite{braverman2020gradient} is that we adapt the range of the matrix where the claim holds to our interested one.

\begin{thm}\label{thm: pcalowerbound}
For any constant $\beta\in(0,1)$,
let $(d,\tau_{\alpha},\mu)$ satisfy that $d\geq d_0(\beta), \mu\leq \mu_0(\beta)$ and $\tau_{\alpha}^{\alpha}\geq d\mu^{\alpha}$ (eigenvalue constrain) for some universal constants $d_0(\beta)$ and $\mu_0(\beta)$.

For any algorithm finding the leading eigenvalue of a positive definite matrix $\mathrm{Alg}$, denote $\mathcal{T}_{\mathcal{D}}$ the gradient oracle call number for finding $\hat{\lambda}$ such that
\begin{align}\label{eq: pcalowerbound0}
\mathbb{P}_{\Mb\sim\mathcal{D},\mathrm{Alg}}\left( \left| \hat{\lambda} - \lambda_1(\Mb) \right|\leq C\mathrm{gap}(\Mb) \right) \geq 1 - \Theta(\sqrt{\beta}).
\end{align}
Assume $\mathcal{D}$ satisfies that any $\Mb\sim\mathcal{D}$, (a) $\mathrm{gap}(\Mb)\gtrsim c_1\mu$, (b) $\mathrm{tr}\big( ((1+c\mathrm{gap})\Ib - \Mb)^{\alpha} \big) \lesssim \tau_{\alpha}^{\alpha}$ and (c) $\mathbf{0}\preceq \Mb \preceq (1-c_1\mathrm{gap})\Ib$. Then for each following case, there exists a distribution $\mathcal{D}$ such that

\begin{itemize}
    \item[1.] $\mathcal{T}_{\mathcal{D}} = \Omega\left(\mu^{-\frac{1}{2}}\right)$ if $\mu^{-\frac{1}{2}} \leq \tau_{\alpha}^{\alpha}$.
    \item[2.] $\mathcal{T}_{\mathcal{D}} = \Omega\left( \tau_{\alpha}^{\frac{\alpha}{1+2\alpha}}\mu^{-\frac{\alpha}{1+2\alpha}} \right)$ if $\mu^{-\frac{1}{2}} \geq \tau_{\alpha}^{\alpha}$ and $\tau_{\alpha}^{\frac{\alpha}{1+2\alpha}}\mu^{-\frac{\alpha}{1+2\alpha}} \leq d$.
    \item[3.] $\mathcal{T}_{\mathcal{D}} = \Omega(d)$ if $\tau_{\alpha}^{\frac{\alpha}{1+2\alpha}}\mu^{-\frac{\alpha}{1+2\alpha}} \geq d$.
\end{itemize}

\end{thm}

We prove Theorem~\ref{thm: pcalowerbound} by constructing a distribution over a parameterized $3\times 3$ block diagonal matrix, and we show that by balancing the order of the parameter, we can obtain the lower bound construction in the near complete regime. The detailed construction is postponed to Appendix~\ref{sec: pcalowerbound}. Theorem~\ref{thm: quadraticlowerbound} immediately follows from combining Lemma~\ref{prop: reductiontoop} and Theorem~\ref{thm: pcalowerbound}.


\begin{proof}
For any algorithm $\mathrm{Alg}$ finding the leading eigenvector of matrix $\Ab$, if $\mathrm{Alg}$ satisfies \eqref{eq: reduct2} with $\delta = \Theta(\sqrt{\beta})$ (We can select $\beta$ to make sure $\Theta(\sqrt{\beta}) \leq \delta_0$) and $c$ sufficiently small, it meets condition of \eqref{eq: pcalowerbound0}. Provided the algorithmic lower bound in Theorem~\ref{thm: pcalowerbound}, we conclude the claimed results.
\end{proof}

\section{Details of Generic Optimization Problems in Convex and Non-convex Setting}\label{sec:generic}

In this section,  we consider the general optimization problems to move forward a single step from our analysis of quadratic optimizing problems. We combine the analysis for Algorithm \ref{alg: esgd} with the Cubic Regularization Newton's Method \cite{nesterov_cubic_2006} and related technologies \cite{monteiro2013accelerated}. For general convex objectives, we achieve $\tilde{\mathcal{O}}\left(\tau_\alpha^{\frac{\alpha}{1+2\alpha}} D^{\frac{14\alpha+12}{14\alpha+7}} H^\frac{2}{14\alpha+7}\epsilon^{-\frac{7\alpha+2}{14\alpha+7}} +  D^{\frac{6}{7}}H^{\frac{2}{7}}\epsilon^{-\frac{2}{7}}\right)$ gradient oracle complexity and for non-convex optimization we can find an $\left(\epsilon,\sqrt{H\epsilon}\right)$-approximate second-order stationary point with $\tilde\cO\left(H^{\frac{1+\alpha}{2+4\alpha}}\tau_{\alpha}^{\frac{\alpha}{1+2\alpha}}\epsilon^{-\frac{3+7\alpha}{2+4\alpha}}  \right)$ gradient oracles.

\subsection{Convex Objective Functions}
For generic convex objective functions, we propose a method to reduce one iteration of solving the general convex optimization problem to logarithmic times of solving another quadratic sub-problem and achieve improved oracle complexity. Specifically, we consider the large-step A-NPE framework proposed by \cite{monteiro2013accelerated} but turn the search
of hyper-parameters process from a solution of a proximal Newton equation to a binary search with a solution of a quadratic function in each step. The detail of the algorithm is shown in Appendix~\ref{app:cubicalgs}, combining with our Algorithm~\ref{alg: esgd}. We achieve the improved gradient oracle complexity $\tilde{\mathcal{O}}\left(\tau_\alpha^{\frac{\alpha}{1+2\alpha}} D^{\frac{14\alpha+12}{14\alpha+7}} H^\frac{2}{14\alpha+7}\epsilon^{-\frac{7\alpha+2}{14\alpha+7}} +  D^{\frac{6}{7}}H^{\frac{2}{7}}\epsilon^{-\frac{2}{7}}\right)$. The result is shown in Theorem \ref{thm:convex} and the proof of Theorem \ref{thm:convex} is shown in Appendix \ref{app:cubicalgs}.

\begin{thm}
    Assume the $(\alpha,\tau_\alpha)$-degenerated objective function $f$ is convex and has $H$-continuous Hessian matrices. Under the same hyper-parameters setting in Algorithm \ref{alg:A-HPEesgd}, it  requires 
    \begin{equation}
        \tilde{\mathcal{O}}\left(\tau_\alpha^{\frac{\alpha}{1+2\alpha}} D^{\frac{14\alpha+12}{14\alpha+7}} H^\frac{2}{14\alpha+7}\epsilon^{-\frac{7\alpha+2}{14\alpha+7}} +  D^{\frac{6}{7}}H^{\frac{2}{7}}\epsilon^{-\frac{2}{7}}\right)
    \end{equation}
    gradient oracle calls to find an $\epsilon$-approximate solution, where 
    \begin{equation}
        D = \inf_{\x^* \in \Xb^*} \sup\left\{\|\x-\x^*\|: f(\x) \leq f(\x_0)\right\}.
    \end{equation}
    \label{thm:convex}
\end{thm}
\subsection{Non-convex Objective Functions}
We consider finding an $\left(\epsilon, \cO(\sqrt{\epsilon}\right))$-approximate second-order stationary point (SSP) for a second-order smooth objective function in the general non-convex setting. We propose our Algorithm \ref{alg:cubicesgd} by combining the Cubic Regularization Newton's Method designed by Nesterov \cite{nesterov_cubic_2006} with our Algorithm \ref{alg: esgd} to solve the quadratic sub-problem in Appendix \ref{app:cubicalgs}. We obtain the improved gradient oracle complexity $\tilde\cO\left(H^{\frac{1+\alpha}{2+4\alpha}}\tau_{\alpha}^{\frac{\alpha}{1+2\alpha}}\epsilon^{-\frac{3+7\alpha}{2+4\alpha}}  \right)$. We show our results in Theorem \ref{thm:non-convex} and the proof of Theorem \ref{thm:non-convex} is shown in Appendix \ref{app:cubicalgs}.

\begin{thm}
    Assume the $(\alpha,\tau_\alpha)$-degenerated objective function $f$ has $H$-continuous Hessian matrices. Under the corresponding hyper-parameters setting in Algorithm \ref{alg:cubicesgd}, it  requires 
    \begin{equation}
        \tilde\cO\left(H^{\frac{1+\alpha}{2+4\alpha}} \cdot \Delta \cdot
\tau_{\alpha}^{\frac{\alpha}{1+2\alpha}}\epsilon^{-\frac{3+7\alpha}{2+4\alpha}} \right)
    \end{equation}
    gradient oracle calls to find an $\left(\epsilon,\sqrt{H\epsilon}\right)$-approximate second-order stationary point, where 
    \begin{equation}
        D = \inf_{\x^* \in \Xb^*} \sup\left\{\|\x-\x^*\|: f(\x) \leq f(\x_0)\right\}, \quad \Delta = f(\x_0) - f^*.
    \end{equation}
    \label{thm:non-convex}
\end{thm}
\section{Details of Data Access Complexities of Linear Regression}

We restate the data access oracle complexities in Section~\ref{sec: dataaccess}.
\begin{thm}[Data Access Oracle Bound, Theorem~\ref{thm: dataaccessoracle} restated]
    Consider optimizing problem~\ref{eq: regerm} with $f_i(\xb) = \frac{1}{2}(\ab_i^\top\xb - b_i)^2$. With normalized data $\|\ab_i\|\leq 1$, there is an algorithm that generates an $\varepsilon$-approximate minimizer of the problem with high probability, using $\tilde{\mathcal{O}}\left( n + n^{\frac{5}{6}}\mu^{-\frac{1}{3}} \right)$ data accesses and $\tilde{\mathcal{O}}\left( n + d^{\frac{5}{6}}\mu^{-\frac{1}{3}} \right)$ data accesses when $n\geq d$.
\end{thm}

We propose a mini-batch accelerated stochastic gradient method combined with the large eigenspace finding. To leverage the large eigenspace finding techniques, instead of using a single stochastic gradient per iteration, we adopt a $\Theta(\sqrt{n})$-mini-batch setting with an according mini-batch size $\Theta(\sqrt{n})$. Next, we invoke Algorithm~\ref{alg: eigenextrator} to perform large eigenspace finding on each mini-batch. We perform accelerated stochastic gradient descent on primal space \cite{allen2017katyusha} of the mini-batch optimization problem. We summarize our algorithm as Algorithm~\ref{alg: finitesum}. And we can make use of the adaptive leverage score computation procedure in \cite{agarwal2017leverage}, to improve the data access complexity in the large sample regime.

\begin{algorithm}[H]
    \caption{Regression Solver}
    \label{alg: finitesum}
    \begin{algorithmic}
        \STATE{\bfseries Input:} accuracy $\epsilon$, mini batch number $m$, $f(\xb) = \frac{1}{2n}\sum_{i=1}^n(\ab_i^\top\xb - b_i)^2$.
        \STATE Divide the finite-sum problem into $m$ mini batches $f(\xb) = \frac{1}{m}\sum_{i=1}^m f_i(\xb_i)$.
        
        \For{$l = 1,2,\cdots,m$}{
        \STATE Apply Algorithm~\ref{alg: eigenextrator} to $f_i(\xb)+\frac{\mu}{2}\|\xb\|^2$ and obtain $\Ab_i$.
        }
        \STATE Apply Katyusha with $f(\xb) = \frac{1}{m}\sum_{i=1}^m\left( f_i(\xb) - \frac{1}{2}\xb^{\top}\Ab_i\xb + \frac{\mu}{2}\|\xb\|^2 \right)$ and $h(\xb) = \frac{1}{2} \xb^\top\left(\sum_{i=1}^m\Ab_i\right)\xb$.
    \end{algorithmic}
\end{algorithm}


\subsection{Proof of Theorem~\ref{thm: dataaccessoracle}}
We first state two results regarding the accelerated stochastic gradient and leverage score sampling.
\begin{lemma}[Theorem 5 of \cite{agarwal2017leverage}, with nonspecific regression solver]\label{lm: lsforreg}
    Suppose an algorithm solving linear regression problem for $A\in\mathbb{n,d}$ using data access complexty $\mathcal{C}(n)$. There is an algorithm solving the linear regression problem using data access $\tilde{\mathcal{O}}(n + \mathcal{C}(d)))$, with high probability.
\end{lemma}
\begin{lemma}[Theorem 2.1 of \cite{allen2017katyusha}]\label{lm: katyusha}
Consider optimizing function $f(\xb) = \frac{1}{n}\sum_{i=1}^nf_i(\xb) + h(\xb)$, with each $f_i(\xb)$ convex and $L$-smooth, and $h(\xb)$ $\mu$-strongly convex. Katyusha algorithm achieves an $\varepsilon$-approximated minimizer using at most $\tilde{\mathcal{O}}\left( n + \sqrt{nL/\mu} \right)$ data accesses.
\end{lemma}

\begin{proof}
We invoke Algorithm~\ref{alg: finitesum} with $m = \Theta(\sqrt{n})$. Denote that $g_i(\xb) = f_i(\xb) - \frac{1}{2}\xb^{\top}\Ab_i\xb + \frac{\mu}{2}\|\xb\|^2$. Note that each $g_i(x)$ is convex given by Theorem~\ref{thm: eigenextractor}, Katyusha algorithm converges to an $\epsilon$-approximate minimizer with $\tilde{\mathcal{O}}\left( m + \sqrt{m\left(\max_{i\in[m]}\lambda_{\max}(\nabla^2 g_i(\xb))\right)/\mu} \right)$ gradient calls of $g_i(\xb)$. And each gradient call of $g_i(\xb)$ or  $f_i(\xb)$ accesses $\mathcal{O}\left(\frac{n}{m}\right)$ data.

Since we assume that the data is normalized, for each $f_i(\xb)$, we have that $\tr\left(\nabla^2f_i(\xb)\right)\leq 1$. This indicates that $\max_{i\in[m]}\lambda_{r}(\nabla^2 f_i(\xb)) \leq \frac{1}{r}$. Combining Theorem~\ref{thm: eigenextractor}, we can use $\tilde{\mathcal{O}}\left( n^{-\frac{1}{6}}\mu^{-\frac{1}{3}} \right)$ gradient oracle calls of $f_i(\xb)$ to find $\lambda_k\left(\nabla^2 g_i(\xb) \right)\leq n^{\frac{1}{6}}\mu^{\frac{1}{3}}$.
Thus the overall data access complexity is $\tilde{\mathcal{O}}\left( \frac{n}{\sqrt{n}}\left( \sqrt{n} + n^{\frac{3}{4}}\sqrt{n^{1/6}\mu^{-2/3}} \right) + n n^{-\frac{1}{6}} + \mu^{-\frac{1}{3}}\right) = \tilde{\mathcal{O}}\left( n + n^{-\frac{5}{6}}\mu^{-\frac{1}{3}} \right)$.

And that claim under $d<n$ follows immediately follows from the previous analysis and Lemma~\ref{lm: lsforreg}.
\end{proof}

\section{Details of Interior Point Methods for ERM}\label{sec:ipm}

\subsection{From Empirical Risk Minimization to IPM}

In this section, our goal is to discuss the prospect of leveraging IPM (specifically, weighted path finding) to solve the proximal operator of ERM, which writes
\begin{align}\label{eq: ermagdprox}
    \arg\min_{\xb} \left\{ \sum_{i = 1}^n f_i^*(\xb_i) + \bbb^\top\xb + \frac{1}{\mu}\sum_{i=1}^r p_i\left(\langle \ub_i, \xb \rangle\right)^2 + \frac{1}{2\mu r}\|\xb - \gammab\|^2 \right\}.
\end{align}
Let $\zeta_i\defeq \langle \ub_i, \xb \rangle$ for $ i\in [r]$, then ~\eqref{eq: ermagdprox} is equivalent to
\begin{align}
    \min_{\xb, \zeta} \left\{ \sum_{i = 1}^n \left[f_i^*(\xb_i)+\frac{1}{2\mu r}\xb_i^2\right] + \frac{1}{\mu}\sum_{i=1}^r p_i\zeta_i^2 +\rm{Linear}\right\}.
\end{align}
Note that minimizing convex function $u(\xb)$ is equivalent to minimizing $y$ over $\left\{(\xb,y): u(\xb)\le y\right\}$. Therefore the minimization problem can be reformulated as
\begin{equation}
    \min_{\begin{array}{c}
\xb\in\Rd~:~\ma^{\trans}\xb=\bbb\\
\forall i\in[d/2]~:~(\xb_{2i-1},\xb_{2i})\in K_i
\end{array}}\cbb^{^\top}\xb\,,\label{eq:intro:2d-barrier}
\end{equation}
where $ d=O(n)$, $\ma\in\R^{d\times r}$ is non-degenerate (or we can add $\cO(1)$ auxiliary constraints),  $\vb\in\Rr$ and $ \vc\in\Rd$. And all $K_i\in\R^2$ are bounded convex sets with self-concordant barrier functions $\phi_i$ whose gradient and Hessian can be computed in $O(1)$ time. For simplicity, we assume $d$ is an even number, or we can add an additional dummy variable. We denote $\Omega$ as the domain of $x$, and $\dInterior$ as its relative interior. We assume that all $\phi_i$ have barrier parameter no more than $2$. Such barrier functions for epigraphs are well known for a variety of univariate convex functions. We have the following theorem.
\begin{thm}[Complexity of IPM Subroutine]
\label{thm:ipm_complexity}
    Given a block-weight-function $g(x)$ (Definition~\ref{def:gen:weight_function}), for Problem (\ref{eq: ermagdprox}), we can compute $\xb$ that is $\epsilon$-optimal 
    in time $\tilde O\left( nr^{1.5}+r^{2.5}\log(1/\epsilon) \right)$.
\end{thm}

\begin{proof}
    The proof is a direct result of the  Theorem~\ref{thm:r-iteration} which shows  that $\tilde\cO(\sqrt r)$-iteration algorithm to find an $\epsilon$-optimal solution   and   Theorem~\ref{thm:iteration_solver} which show that the  per iteration  cost can be reduced to $\tilde \cO(dr+r^2)=\tilde \cO(nr+r^2)$.
    

\end{proof}

Theorem~\ref{thm:r-iteration} and Theorem~\ref{thm:iteration_solver} will be proved in the Section \ref{subsec:weight_function} and \ref{subsec:inverse-maintenance}, respectively. Now we prove Theorem~\ref{thm: ERM-regu-complexity}.
\begin{proof}[Proof of Theorem~\ref{thm: ERM-regu-complexity}]
Algorithm~\ref{alg: eigenextrator} takes $\cO(ndr)$ time to extract the first $r$ eigenvalues. The accelerated gradient method takes $\cO\left(\frac{\tau}{\mu r^{1/\alpha}}\right)$, and each iteration consists of calculating gradient which takes $\cO(nd)$ time, and IPM subroutine which costs $\tilde O\left( dr^{1.5}+r^{2.5}\log(1/\epsilon) \right)$ time.
    
\end{proof}

\subsection{Some Notations}

\textbf{Vector, Matrices, Norm: }
Given a vector $\vb\in\Rd$, we define its infinity norm $\|\vb\|_\infty\defeq\max_i|\vb_i|$ and its square norm $\|v\|_\square \triangleq\max_{i\in [d/2]} \sqrt{v_{2i-1}^2+v_{2i}^2}$.

We say a $2-$block diagonal matrix  $\mm\in\oplus_{i=1}^{d/2}\mathbb R^{2\times 2}$ if $\mm$ can be written as
$$
\mm=\begin{pmatrix}
	\mm_1&&&\\
	&\mm_2&&\\
	&&\ddots&\\
	&&&\mm_{d/2}\\
\end{pmatrix}.
$$ 
Moreover, if each $\mm_i$ is PSD, we write $\mm\in\oplus_{i=1}^{d/2}\R_+^{2\times 2}$ and we define its square norm as $\|\mm\|_\square \triangleq \max_{i\in [d/2]} \|\mm\|_{op}$.

\textbf{Leverage Scores} We denote the leverage scores of a matrix $\ma\in\R^{d\times r}$ by vector $\sigma(\ma)$ and when $\ma$ is clear in the context, we simply use $\sigma$. We say a leverage score of the $i$-th row of $\ma$ is $\sigma_i\defeq\left[ \ma(\ma^\top\ma)^{-1}\ma^\top \right]_{ii}$ for $i\in[d]$. $\sigma^{(k)}$ is an abbreviation of $\sigma({\mm^{(k)}}^{1/2}\ma)$.

\textbf{Analysis} For bivariable functions $\phi_1,\dots, \phi_{d/2}$, we let $\Phi''(\xb)\in \oplus_{i=1}^{d/2}\R_+^{2\times 2}$ be a $2$-block matrix with the $i$-th $2\times 2$ block being the Hessian of $\phi_i$ at $\xb$.  Let $\phi'(\xb)\defeq(\phi_1'(\xb_1,\xb_2)^\top,\dots, \phi_{d/2}'(\xb_{d-1},\xb_{d})^\top)^\top$.

\textbf{Quotient of matrices}
    We define the log-quotient between two $2\times 2 $ positive definite matrices $ \mm_1$ and $ \mm_2$ to be the minimal $\epsilon \ge 0$ such that $e^{-\epsilon}\mm_1 \preceq \mm_2 \preceq e^\epsilon\mm_1$.
For simplicity we denote $\log\left( \mm_1 /\mm_2 \right)$. For two 2-block diagonal matrices $\mm,\mn\in\oplus_{i=1}^{d/2}\R_+^{2\times 2}$, we let $\log\left( \mm /\mn \right)$ be a vector $v$ in $\Rd$ such that
$
v_{2i-1}=v_{2i}=\log\left( \mm_i /\mn_i \right).
$

\subsection{Weight Function and Centering Analysis\label{subsec:weight_function}}
In this part we follow the reweighted path finding technique of Lee and Sidford \cite{leeS14,lsJournal19}. The main difference is that we deal with bivariate barrier functions rather than univariate functions. Therefore we need a new definition of weight function. Given this, the proofs are essentially identical to the counterparts in \cite{leeS14,lsJournal19} up to minor modifications of notation and constants.

Starting from a feasible point $ x^{(0)}$, we alternates increasing $ t$ and minimizing the penalized objective function
\begin{equation}
    \min_{\ma^\top \xb=\bbb}f_t(\xb,w)=t\cdot\cbb^\top\xb+\sum_{i=1}^{d/2} w_i\phi_i(\xb_{2i-1},\xb_{2i}),
\end{equation}
where $\phi(\xb_{2i-1},\xb_{2i})$ is the self-concordance barrier function for $ K_i$, and $w_i>0$ is the weight. 
Denote $\xb_t\defeq \min_{\ma^\top \xb=\bbb} f_t(\xb,w)$. Note that when $t\rightarrow \infty$, we have $\xb_t\rightarrow \xb^*$.

Let $\delta_t(\xb,w)$ denotes the centrality that will define later. It is a distance to measure how close between $\xb$ and $\xb_t$, with $\delta_t(\xb,w)=0$ iff $\xb=\xb_t$. Our goal is to increase $t$ while maintaining the centrality small (specifically, let $\delta$ below some sufficiently small fixed constant) through a Newton step on $\xb$. And the weight $w$ updates each time after $\xb$ updates.

In the rest parts we let $\wb$ denotes a length-$d$ vector, with $\wb_{2i-1}=\wb_{2i}=\vWeight_i$ for all $i\in[d/2]$, and we write $\mw$ as the diagonal matrix form of $\wb$. We denote $f_t(\xb,\wb)$ instead of $f_t(\xb,\vWeight)$.

In order to control the magnitude of a Newton step, we need the following definition.

\begin{defn}[Centrality Measure]
	\label{Def:centrality_measure} For $\{\xb,\wb\}\in\dInterior\times\Rd_{>0}$
	and $t\geq0$, we let $\vh_{t}(\xb,\wb)$ denote the \emph{projected
		newton step} for $\xb$ on the penalized objective $f_{t}$ given
	by
	$$
	\vh_{t}(\xb,\wb)\defeq-\Phi''(\xb)^{-1/2}\mProj_{\xb,\wb}{\Wb^{-1}\Phi''(\xb)^{-1/2}}\left({\grad_{\xb}f_{t}(\xb,\wb)}\right),
	$$
	where $\mProj_{\xb,\wb}\defeq \mi -\mw^{-1}\ma_\xb\left( \ma_\xb^\top \mw^{-1}\ma_\xb \right)^{-1}\ma_\xb^\top$ for $\ma_\xb\defeq \Phi''(\xb)^{-1/2}\ma$.
	We measure the \emph{centrality} of $\{\xb,\wb\}$ by
	\begin{equation}
		\delta_{t}(\xb,\wb)\defeq\min_{\veta\in\Rn}\mixedNormFull{
  {\Wb^{-1}\Phi''(\xb)^{-1/2}}
  \left({\grad_{\xb}f_{t}(\xb,\wb)-\ma\veta}\right)
  }\wb, \label{eq:centrality_definition}
	\end{equation}
	where for all $\vy\in\Rm$ we let $\mixedNorm{\vy}{\vWeight}\defeq\norm{\vy}_{\square}+\cnorm\norm{\vy}_{\mWeight}$
	for $\cnorm>0$ defined in Definition~\ref{def:gen:weight_function}.
\end{defn}

Upon a Newton step on $\xb$, we can reduce the centrality to $ 4\delta_t(\xb,\wb)^2$ if $\delta_t(\xb,\wb)$ is below some efficiently small constant and the weight $\wb$ satisfies some certain properties. Upon increase $t$ to $(1+\alpha)t$, we increase the centrality by $O(\alpha\sqrt {\|\wb\|_1})$.

For our situation we define the block-weight-function to assign the weights $\wb$.
\begin{defn}
	[Block Weight Function]\label{def:gen:weight_function} Differentiable
	$\fvWeight:\dInterior\rightarrow\R_{>0}^{d}$ is a \emph{$(\cWeightSize,\cWeightStab$,$c_{k}$)
		-block-weight function} if the following hold for all $\xb\in\dInterior$
	and $i\in[d]$:
	
	\begin{itemize}
        \item The \emph{block property}, that is $g(\xb)_{2j-1}=g(\xb)_{2j}$ for all $j\in[d/2]$.
		\item The \emph{size, $c_{1}$, }satisfies $\cWeightSize\geq\max\{1,\normOne{\fvWeight(\xb)}\}$.
		\item The \emph{sensitivity,} $c_{s}$, satisfies $c_{s}\geq 4\cordVec i^{^\top}\mg(\xb)^{-1}\ma_{\xb}\left(\ma_{\xb}^{\top}\mg(\xb)^{-1}\ma_{\xb}\right)^{-1}\ma_{\xb}^{\top}\mg(\xb)^{-1}\cordVec i$.
		
		\item The\emph{ consistency, }$c_{k}$,\emph{ }satisfies $\mixedNorm{\fmWeight(\xb)^{-1}\mj_{g}(x)(\mPhi''(\xb))^{-1/2}}{\vg(\xb)}\leq1-c_{k}^{-1}<1$.
		
	\end{itemize}
    For efficiency, we need $ c_1=\cO(r)$, $ c_s, c_k=\tilde \cO(1)$. Besides, we require that the weight function is easy to compute:
    \begin{itemize}
        \item There is an algorithm $\texttt{ComputeApxWeight}(\xb,\wb^{(0)},\epsilon) $ such that given a initial weight $ \wb^{(0)}$ with $\| 
\wb_{(0)}^{-1}(g(\xb)-\wb^{(0)}) \|_\infty\le 2^{-20}$, it can w.h.p. output $\wb$ with $\| 
g(\xb)^{-1}(g(\xb)-\wb^{(0)}) \|_\infty\le \epsilon$ in $\tilde O(\poly(1/\epsilon))$ steps. Each step we can be implemented in $\tilde \cO(dr+\time)$, where $\time$ is the complexity needed to solve $(\ma^\top\mm\ma)^{-1}\zb$.
\item there is an algorithm $\texttt{ComputeInitialWeight}(\ma,\xb)$ such that even without $\wb^{(0)}$, the algorithm have the same guarantee in $\tilde \cO(\sqrt{d})$ steps.
    \end{itemize}
	Through we assume we have such a weight function and define $ \cnorm\defeq24\sqrt{c_{s}}c_{k}, c_\gamma\defeq1+\frac{\sqrt{2c_s}}{C_{norm}}\le 1+\frac{1}{16c_k}$.
\end{defn}

\begin{remark}
    Besides the block property, only the \emph{consistency} is essentially different from the original definition in \cite{lsJournal19}. We left the construction of such block-weight-function in future work.
\end{remark}

\begin{remark}
    The weight function is crucial for reducing the iteration from $\tilde\cO(\sqrt d)$ to $\tilde\cO(\sqrt r)$. In the framework of self-concordance theory by Nesterov and Nemirovski \cite{Nesterov1994}, $w_i$ is set to $1$ so $t$ can only increase by $1+d^{-1/2}$ each time. Hence $\tilde\cO(\sqrt d)$ iterations are needed. There is a trade-off on the choice of $\wb$ since large $\wb$ results in more iterations while small $\wb$ can cause instability in Newton steps.
    
    A breakthrough of Lee and Sidford \cite{leeS14, lsJournal19} shows that one can construct a weight function $g(x)$ to assign $w$ such that $\|w\|_1=\tilde\cO(r)$ and that the Newton step for $x$ can decrease the centrality quadratically. Hence the iterations can be reduced to $\tilde\cO(\sqrt r)$.
    
\end{remark}

In each iteration, we first apply a Newton step for $\xb$, then approximately calculate the weight function $g(\xb)$ and update the weight $\wb$ such that keep $\wb$ and $g(\xb)$ close and that the update of $\wb$ is relatively small. The algorithm \texttt{CenteringStep} formally describe the procedure.

\begin{lemma}[Centrality Change in Each Step]\label{thm:centrality-change}
    We have
    \begin{itemize}
    
        \item (Changing $t$) For $\{\xb,\wb\}\in\dInterior\times\Rd_{>0}$, $t>0$ and $\alpha>0$, we have
        $$\delta_{(1+\alpha)t}(\xb,\wb)\le (1+\alpha)\delta_t(\xb,\wb)+ \alpha\left( \sqrt 2+C_{norm}\sqrt{\|\wb\|_1} \right).$$
        \item (Changing $\xb$) For $\{\xb,\wb\}\in\dInterior\times\Rd_{>0}$ such that $\delta_t(\xb,\wb)\le\frac{1}{10}$ and $\frac{4}{5}g(\xb)\le\xb\le\frac{5}{4}g(\xb)$ and consider a Newton step $ \next\xb=\xb+h_t(\xb,\wb)$, we have
        $$ \delta_t(\xb^{(new)},\wb)\le 4(\delta_t(\xb,\wb))^2.$$
        \item (Changing $\wb$) For $ \wb,\vb$ such that $\epsilon=\mixedNormFull{\log(\wb)-\log(\vb)}{\wb}\le \frac{1}{10}$, we have
        $$ \delta_t(\xb,\vb)\le (1+4\epsilon)(\delta_t(\xb,\wb)+\epsilon).$$
    \end{itemize}
\end{lemma}

Given Lemma~\ref{thm:centrality-change}, there exists an algorithm \texttt{CenteringStep} that preforms a single step on $\xb$ and $\wb$, and decrease the centrality by $ (1-\frac{1}{4c_k})$. The proof essentially follows from Theorem 19 in \cite{lsJournal19}. By alternates updating $(\xb,\wb)$ and increasing $t$, we can prove the following theorem.

\begin{thm}[Path Finding for Linear-Objective Optimization Problem]\label{thm:r-iteration}
    For problem (\ref{eq:intro:2d-barrier}), given a block-weight-function (Definition~\ref{def:gen:weight_function}) and a starting feasible point $\xb^{(0)}$, there exist an algorithm \texttt{PathFollowing} that outputs an $\epsilon$-optimal solution in $\tilde\cO(\sqrt{r})$ iterations,  where each iteration consists of solving $\tilde\cO(1)$ linear systems and linear systems between iterations satisfy \emph{block-$\sigma$-stability assumption}(Definition~\ref{ass:stability_sigma}). 
\end{thm}
We left the detailed algorithms and proofs in \ref{appsubsec:weighted-path}.

\subsection{Inverse Maintenance}\label{subsec:inverse-maintenance}
In this part we follow the inverse maintenance technique of \cite{lee2015efficient} which exploits the leverage score sampling technique to approximately and implicitly maintains a sequence of matrices $(\ma^\top\md^{(k)}\ma)^{-1}$. Our situation requires sampling each two rows in a block simultaneously 

In the previous parts, it is shown that problem (\ref{eq:intro:2d-barrier}) can be solved in $\tilde O\left( \sqrt{r}\right)$ iterations provided a block-weight-function. Furthermore, the sequence of linear systems are slowly changing. In this part, we will adopt and slightly modify the original inverse maintenance technique to adapt the case for $\ma^\top\mm^{(k)}\ma$ when each $\mm^{(k)}$ is a 2-block matrix and do not change too rapidly. The main difference is that we have to sample each two rows in a block simultaneously.

For formality, we need the following definitions.

\begin{defn}[Linear System Solver \cite{lee2015efficient}]
 Given a PD
matrix $\mb\in\R^{d\times d}$, an algorithm $\solver$ w.r.t. $ \mb$ is a $\mathcal{T}$-time solver of  if for all $\vb\in\R^{d}$ and $\epsilon\in(0,1/2]$,
the algorithm outputs a vector $\solver(\vb,\epsilon)\in\R^{d}$ in
time $O(\time\log(\epsilon^{-1}))$ such that with high probability
in $d$, $\norm{\mathcal{\solver}(\vb,\epsilon)-\mb^{-1}\vb}_{\mb}^{2}\leq\epsilon\norm{\mb^{-1}\vb}_{\mb}^{2}$.
We call the algorithm $\solver$ linear if $\solver(\vb,\epsilon)=\mq_{\epsilon}\vb$
for some $\mq_{\varepsilon}\in\R^{d\times d}$ that depends only on
$\mb$ and $\epsilon$.
\end{defn}

\begin{defn}[Block-$\sigma$-Stability Assumption]
\label{ass:stability_sigma} We say that the inverse maintenance
problem satisfies the block-$\sigma$-stability assumption if for each $k\in[l]$
we have $\norm{\log(\mm^{(k)}/\mm^{(k-1)})}_{\vLever^{(k)}}\leq0.1$,
$\norm{\log(\mm^{(k)}/\mm^{(k-1)})}_{\infty}\leq0.1$, and $\beta^{-1}\ma^{\top}\mm^{(0)}\ma\preceq\ma^{\top}\mm^{(k)}\ma\preceq\beta\ma^{\top}\mm^{(0)}\ma$
for $\beta=\poly(n).$
\end{defn}

\begin{remark}\label{remark:decomposition-invariant}
    Note that $\norm{\log(\mm^{(k)}/\mm^{(k-1)})}_{\vLever(\mc^{(k)}\ma)}$ is invariant under different decompositions of  ${\mc_i^{(k)}}^\top\mc_i^{(k)}=\mm_i^{(k)}$ since if $ {\mc_i^{(k)}}^\top\mc_i^{(k)}={\mf_i^{(k)}}^\top\mf_i^{(k)}=\mm_i^{(k)}$ for PD matrix $ \mm_i^{(k)}$, there exists an orthogonal matrix $\mO_i^{(k)} $ such that $ \mc_i^{(k)}\mO_i^{(k)}=\mf_i^{(k)}$.
\end{remark}

In Section~\ref{appsub:inverse-maintain} we will how to adopt the original proof to our situation and give a modified algorithm \texttt{InverseMaintainer}.
\begin{thm}
\label{thm:iteration_solver} Suppose that the inverse maintenance
problem satisfies the block $\sigma$ stability assumption. Then Algorithm~\ref{alg:sparseframework}
maintains a $\tilde{O}(\nnz(\ma)+r^{2})$-time solver with high probability
in total time $\otilde(r^{\omega}+l(\nnz(\ma)+r^{2}))$ where $l$
is the number of rounds. Specifically, when $l=\tilde\cO(\sqrt r)$, the total time for constructing solvers and solving linear systems is $\tilde\cO(dr^{1.5}+r^{2.5})$.
\end{thm}

\appendix

\section{Proofs in Section \ref{sec: quadratic}}\label{sec: pcalowerbound}

\subsection{Proof of Theorem~\ref{thm: pcalowerbound}}

We leverage the following theorem on the Wishart random matrix's largest eigenvalue computing and its spectral properties.

\begin{lemma}[Extentsion of Theorem 10 of \cite{braverman2020gradient}]\label{lm: wishartprop}
There exists a universal constant $p_0$ and function $d : (0, 1) \rightarrow N$ such that the following holds: for all $\beta \in (0, 1)$, and all $d \geq d(\beta)$, we have that $W \sim \mathrm{Wishart}(d)$ satisfies
\begin{itemize}
    \item[(a)] Any algorithm $\mathrm{Alg}$ which makes $T \leq (1-\beta) d$ adaptively chosen oracle calls, and returns an estimate $\hat{\lambda}_{\min}$ of $\lambda_{\min}(\Wb)$ satisfies
    \begin{align*}
        \mathbb{P}\left( \left| \hat{\lambda}_{\min} - \lambda_{\min}(\Wb) \right|\geq \frac{1}{4d^2} \right)\geq c_{\mathrm{wish}}\sqrt{\beta}
    \end{align*}
    \item[(b)] There exists constants $C_1(\beta)$, $C_2(\beta)$ and $C_3(\beta)$ such that
    \begin{align*}
        \mathbb{P}_{\Wb}\big( \left\{ \lambda_d(\Wb) \leq C_1(\beta)d^{-2} \right\} \cap & \left\{C_2(\beta)d^{-2} \leq \lambda_{d-1}(W) - \lambda_d(W)\leq C_3(\beta)d^{-2} \right\}\\
        & \cap  \left\{ \|\Wb\|\leq 5 \right\} \big) \geq 1 - \frac{c_{\mathrm{wish}}\sqrt{\beta}}{2}
    \end{align*}
\end{itemize}
\end{lemma}
\begin{proof}
    The difference of Lemma~\ref{lm: wishartprop} and Theorem 10 of \cite{braverman2020gradient} is that we convert the event $\left\{ \lambda_{d-1}(W) - \lambda_d(W)\geq C_2(\beta)d^{-2}  \right\}$ to $\left\{C_2(\beta)d^{-2} \leq \lambda_{d-1}(W) - \lambda_d(W)\leq C_3(\beta)d^{-2} \right\}$. The correctness follows from the limiting distribution of $(d^2\lambda_d(\Wb_d), d^2\lambda_{d-1}(\Wb_d))$ as in \cite{ramirez2009diffusion}.
\end{proof}

\begin{proof}[Proof of Theorem~\ref{thm: pcalowerbound}]
Denote the event in the claim (b) of Lemma~\ref{lm: wishartprop} by $\mathcal{E}$.

\textbf{Case 1:}
Define
\begin{align}\label{eq: pcalbcase1}
    \Mb = \left(
    \begin{array}{ccc}
        c\left(\Ib_{s} - \frac{1}{5}\Wb_s\right) + (1-c)\Ib_s & \mathbf{0} & \mathbf{0}\\
        \mathbf{0} & \left(1-\left(\frac{\tau_{\alpha}^{\alpha}}{d}\right)^{1/\alpha}\right)\Ib_{d-s-1}  & \mathbf{0}\\
        \mathbf{0} & \mathbf{0} & \mathbf{0}
    \end{array}
    \right),
\end{align}
where $s = \Theta_{\beta}\left(\mu^{-\frac{1}{2}}\right)$. Suppose that $\Wb_s$ in the construction of $\Mb$ follows the Wishart distribution conditioned on $\mathcal{E}$. There exists a constant $c\in(0,1)$, such that $\Mb$ satisfies the following conditions

\begin{itemize}
    \item[(a).] $\lambda_{\max}\left(\Ib_s -\frac{1}{5}\Wb_s\right) \geq 1 - \mu$. And $c_1\mu \leq \mathrm{gap}(\Ib_s - \frac{1}{5}\Wb_s) \leq c_2\mu$.
    Note that here $\Ib_s -\frac{1}{5}\Wb_s$ refers to the first block in \eqref{eq: pcalbcase1} and so is the below.
    \item[(b).] $\mathbf{0}\preceq \Mb \preceq \Ib$.
    \item[(c).] $(\lambda_{1}(\Mb), \lambda_2(\Mb))  = \left(\lambda_{1}\left(c\left(\Ib_s -\frac{1}{5}\Wb_s\right) + (1-c)\Ib_s\right), \lambda_{2}\left(c\left(\Ib_s -\frac{1}{5}\Wb_s\right) + (1-c)\Ib_s\right)\right)$.
    \item[(d).] $\tr\left(\left((1+c_0\mathrm{gap}(\Mb))\Ib - \Mb\right)^{\alpha}\right) \lesssim \tau_{\alpha}^{\alpha}$. 
\end{itemize}

We prove the above claims in sequence.

Claim (a) and (b) are direct consequences of Lemma~\ref{lm: wishartprop} since we choose $s = \Theta(\mu^{-\frac{1}{2}})$. 

For claim (c), it suffices to show that there exists $c$ such that $1 - \left( \frac{\tau_{\alpha}}{d} \right)^{1/\alpha} \leq 1 - c(1+ c_2)\mu$, where the latter is demonstrated as in Claim (a) to be the lower bound of $\Ib_s -\frac{1}{5}\Wb_s$. The aforementioned control is equivalent to $\mu \lesssim \left( \frac{\tau_{\alpha}}{d} \right)^{1/\alpha}$, which is obvious since we impose the constrain $\tau_{\alpha}^{\alpha} \geq d\mu^{\alpha}$.

For claim (d), we have
\begin{align*}
    \tr\big(\left((1 + c\mathrm{gap})\Ib_d -\Mb\right)^{\alpha} \big) \leq & (cs+1)(1+c_0\mathrm{gap})^{\alpha} + (d-s)\left( \mathrm{gap} + \left(\frac{\tau_{\alpha}^{\alpha}}{d}\right)^{1/\alpha} \right)^{\alpha}\\
    \lesssim & cs + 1 + d\mathrm{gap}^{\alpha} +  \tau_{\alpha}^{\alpha}\\
    \overset{\text{a}}{\lesssim} & s + d\mathrm{gap}^{\alpha} + \tau_{\alpha}^{\alpha}\\
    \overset{\text{b}}{\lesssim} & \tau^{\alpha}_{\alpha} + \tau^{\alpha}_{\alpha} + \tau^{\alpha}_{\alpha}\\
    \lesssim & \tau_{\alpha}^{\alpha},
\end{align*}
where $\overset{\text{a}}{\lesssim}$ follows from $\mathrm{gap} = \Theta\left(\mu^{-\frac{1}{2}}\right)$ and $\tau_{\alpha}^{\alpha}\geq d\mu^{\alpha}$; $\overset{\text{b}}{\lesssim}$ follows from that $s = \Theta(\mu^{-\frac{1}{2}}) = \mathcal{O}(\tau_{\alpha}^{\alpha})$ in case 1.

Claim (c) shows that $\mathrm{gap}(\Mb) = c\mathrm{gap}\left(\Ib_s - \frac{1}{5}\Wb_s\right)$. Further, claim (c) indicates that finding the leading eigenvector of $\Mb$ is equivalent to the problem for $\Ib_s - \frac{1}{5}\Wb_s$. Specifically, if a algorithm $\mathrm{Alg}$ finds $\hat{\lambda}$ under distribution induced by our \eqref{eq: pcalbcase1} such that
\begin{align*}
    \mathbb{P}_{\Mb\sim\mathcal{D}, \mathrm{Alg},\xb_0}\left( |\hat{\lambda} - \lambda_1(\Mb)| \leq C\mu \right)\geq \Theta(\sqrt{\beta}).
\end{align*}
Then with the same oracle complexity, there is an algorithm $\mathrm{Alg'}$ finds the largest eigenvalue of $\Ib - \frac{1}{5}\Wb$ under the Wishart distribution conditioned on the event $\mathcal{E}$ such that
\begin{align*}
    \mathbb{P}_{\Wb\sim\mathcal{D'},\mathrm{Alg}',\xb_0}\left( \left|\hat{\lambda} - \lambda_1\left(\Ib - \frac{1}{5}\Wb\right)\right| \leq C\frac{1}{s^2} \right)\geq \Theta(\sqrt{\beta})
\end{align*}
and vice versa. By claim (a) in Lemma~\ref{lm: wishartprop}, the output $\hat{\lambda}$ of any algorithm $\mathrm{Alg}$ using less than $s(1-\beta)$ gradient oracles satisfies that 
\begin{align}\label{eq: temppcalbcase1}
\begin{aligned}
    \mathbb{P}_{\mathrm{Alg}, \Mb\sim\mathcal{D}}\left( \left| \hat{\lambda} - \lambda_1(\Mb) \right| \geq C \mu\right)
    = & \mathbb{P}_{\mathrm{Alg'}, \Wb\sim\mathcal{D}'}\left( \left| \hat{\lambda} - \lambda_1(\Wb) \right| \geq C \frac{1}{s^2}\right)\\
    \geq & \mathbb{P}_{\mathrm{Alg'}, \Wb\sim\mathrm{Wishart}(d)}\left( \left| \hat{\lambda} - \lambda_1(\Wb) \right| \geq C \frac{1}{s^2}\right) - \mathbb{P}(\mathcal{E})\\
    \geq &\frac{c_{\mathrm{wish}}\sqrt{\beta}}{2}\\
    = & \Theta(\sqrt{\beta}).
\end{aligned}
\end{align}
Recall we set $s = \Theta\left( \mu^{-\frac{1}{2}} \right)$, the gradient oracle condition we adopt during the control \eqref{eq: temppcalbcase1} is equal to $\Theta\left((1-\beta)\mu^{-\frac{1}{2}} \right)$. Then we finish the first proof.

\textbf{Case 2:}
Define 
\begin{align}\label{eq: pcalbcase2}
    \Mb = \left(
    \begin{array}{ccc}
        c\left(\Ib_{s} - \frac{1}{5}\Wb_s\right) + (1-c)\Ib_s & \mathbf{0} & \mathbf{0}\\
        \mathbf{0} & \left(1-\left(\frac{\tau_{\alpha}^{\alpha}}{d}\right)^{1/\alpha}\right)\Ib_{d-s-1}  & \mathbf{0}\\
        \mathbf{0} & \mathbf{0} & \mathbf{0}
    \end{array}
    \right).
\end{align}
Set $s = \Theta\left( \tau_{\alpha}^{\frac{\alpha}{1+2\alpha}}\mu^{-\frac{\alpha}{1+2\alpha}} \right)$ and $c = \Theta\left( \mu^{\frac{1}{1+2\alpha}}\tau_{\alpha}^{\frac{2\alpha}{1+2\alpha}} \right)$.
Similar to the previous case, we state a series of claims and prove them in sequence. Suppose that $\Wb_s$ in the construction of $\Mb$ follows the Wishart distribution conditioned on $\mathcal{E}$. We have
\begin{itemize}
    \item[(a).] $\lambda_{\max}(\Ib_s -\frac{1}{5}\Wb_s) \geq 1 - \mu$. And $c_1\mu \leq \mathrm{gap}(\Mb) \leq c_2\mu$.
    \item[(b).] $0\preceq \Mb \preceq \Ib$.
    \item[(c).] $(\lambda_{1}(\Mb), \lambda_2(\Mb))  = \left(\lambda_{1}\left(\Ib_s -\frac{1}{5}\Wb_s\right), \lambda_{2}\left(\Ib_s -\frac{1}{5}\Wb_s\right)\right)$.
    \item[(d).] $\tr\left(\left((1+c_0\mathrm{gap}(\Mb))\Ib - \Mb\right)^{\alpha}\right) \lesssim \tau_{\alpha}^{\alpha}$.
\end{itemize}
Claim (a) follows from
\begin{align*}
    1 - \lambda_{\max}\left(\Ib - \frac{1}{5}\Wb_{s}\right) \lesssim  \mu^{\frac{1}{1+2\alpha}}\tau_{2\alpha}^{\frac{\alpha}{1+2\alpha}} \tau_{\alpha}^{-\frac{2\alpha}{1+2\alpha}}\mu^{\frac{2\alpha}{1+2\alpha}} = \mu
\end{align*}
and
\begin{align*}
    \lambda_1\left(\Ib - \frac{1}{5}\Wb_{s}\right) - \lambda_2\left(\Ib - \frac{1}{5}\Wb_{s}\right) = \Theta\left( \mu^{\frac{1}{1+2\alpha}}\tau_{\alpha}^{2\frac{\alpha}{1+2\alpha}} \tau_{\alpha}^{-\frac{2\alpha}{1+2\alpha}}\mu^{\frac{2\alpha}{1+2\alpha}}\right) = \Theta(\mu).
\end{align*}

For claim (b), if it suffices to prove that $c\in(0,1)$, which can be attained through $\mu^{-\frac{1}{2}}\geq\tau_{\alpha}^{\alpha}$ in case 2.

For claim (c), similar to the previous proof, it is equivalent to $c\mu\leq \left( \frac{\tau^{\alpha}_{\alpha}}{d} \right)^{1/\alpha}$, which is obvious since $c\in(0,1)$ and we impose the condition $d\mu^{\alpha}\leq\tau_{\alpha}^{\alpha}$

For claim (d), the trace of the shifted matrix is 
\begin{align*}
    \tr\big(\left((1 + c_0\mathrm{gap})\Ib_d -\Mb\right)^{\alpha} \big) \leq & 1+ c_0\mathrm{gap} + s (c_0\mathrm{gap} + c)^{\alpha} + (d-s)\left( \mathrm{gap} + \left(\frac{\tau_{\alpha}^{\alpha}}{d}\right)^{1/\alpha} \right)^{\alpha}\\
    \lesssim & 1 + s\mathrm{gap}^{\alpha} + sc^{\alpha} + d \mathrm{gap}^{\alpha} + \tau_{\alpha}^{\alpha}\\
    \overset{\text{a}}{\lesssim} &  \tau_{\alpha}^{\alpha} + sc^{\alpha} + d \mathrm{gap}^{\alpha} + \tau_{\alpha}^{\alpha} \\
    \overset{\text{a}}{\lesssim}    &\tau_{\alpha}^{\alpha} +\tau_{\alpha}^{\alpha} + \tau_{\alpha}^{\alpha} + \tau_{\alpha}^{\alpha}\\
    \lesssim & \tau_{\alpha}^{\alpha},
\end{align*}
where in $\overset{\text{a}}{\lesssim}$ we use that $\mathrm{gap} = \mathcal{O}(c)$; $\overset{\text{b}}{\lesssim}$ follows from $sc^{\alpha} = \Theta\left( \tau_{\alpha}^{\frac{\alpha}{1+2\alpha}}\mu^{-\frac{\alpha}{1+2\alpha}}\mu^{\frac{\alpha}{1+2\alpha}}\tau_{\alpha}^{\frac{2\alpha^2}{1+2\alpha} }\right) = \Theta\left( \tau_{\alpha}^{\alpha} \right)$ and $\mathrm{gap} = \mathcal{O}\left(\mu^{-\frac{1}{2}}\right)$ and $\tau_{\alpha}^{\alpha}\geq d\mu^{\alpha}$.

Equipped with the claims, we can step along the similar analysis as in \eqref{eq: temppcalbcase1}, which demonstrates that any algorithm with oracle calls less than $s(1-\beta)$, satisfies the analysis in \eqref{eq: temppcalbcase1}. Recall that we set $s = \Theta\left( \tau_{\alpha}^{\frac{\alpha}{1+2\alpha}}\mu^{-\frac{\alpha}{1+2\alpha}} \right)$. The oracle upper bound that we adopt is equivalent to $\Theta\left((1-\beta)\tau_{\alpha}^{\frac{\alpha}{1+2\alpha}}\mu^{-\frac{\alpha}{1+2\alpha}}\right)$, which completes our proof for case 2.

\textbf{Case 3:}
We consider the same matrix as in Case 2, i.e. \eqref{eq: pcalbcase2}. Here we set $s = \Theta(d)$ and we can prove similar results as in the previous two cases. Thus we omit the proof in case 3.

\end{proof}

\section{Algorithms and Proofs in Section~\ref{sec:generic}}\label{app:cubicalgs}
In this section, we present the missing algorithms and proofs in Section~\ref{sec:generic}.
\subsection{Convex Case}
\subsubsection{Algorithms}
In this section, we present the algorithms for generic smooth convex and non-convex functions. We define $f_\x(\y)$ to be the second-order Taylor expansion (SOE) of f at $\x$: 
\begin{equation}
    f_\x(\y) = f(\x) + \langle\nabla f(\x), \y-\x\rangle + \frac{1}{2}\langle\nabla^2 f(\x)(\y-\x), \y-\x\rangle.
\end{equation}

The complete algorithm for smooth convex functions is shown in Algorithm~\ref{alg:A-HPEesgd}. In each iteration, Algorithm~\ref{alg:A-HPEesgd} uses Algorithm~\ref{alg:A-HPEsearch} to find an inexact solution that satisfies the following conditions:
\begin{equation}
    \begin{split}
    &a_{k+1}= \frac{\gamma_{k+1}+\sqrt{\gamma_{k+1}^2+\gamma_{k+1}A_k}}{2},\\
    &\tilde\x_k = \frac{A_k}{A_k+a_{k+1}}\y_k + \frac{a_{k+1}}{A_k+a_{k+1}}\x_k,\\
    &\gamma_{k+1}\nabla f_{\tx_{k+1}}(\y_{k+1})+\y_{k+1}-\tx_k\approx\mathbf 0,\\\label{equ:approximate}
    &\frac{2\sigma_l}{H} \le \gamma_{k+1}\|\y_{k+1}-\tilde \x_k\|\le \frac{2\sigma_u}{H}.\\
    \end{split}
\end{equation}
In the third line of Equation~\ref{equ:approximate}, we notice that if the equality holds, then $\y_{k+1}$ is the exact solution of the following quadratic optimization problem:
\begin{equation}
    \min_{\y\in \R^d} f_{\tilde \x_k}(\y) + \frac{1}{2\gamma_{k+1}}\|\y-\tilde \x_k\|^2.\label{equ:APEsubproblem1}
\end{equation}
We use binary search in Algorithm~\ref{alg:A-HPEsearch} to determine $\gamma_{k+1}$, and apply Algorithm~\ref{alg: esgd} to \ref{equ:APEsubproblem1} to find an $\errorB$-approximated solution.

\begin{algorithm}[!h]
    \caption{Inexact Large-step A-NPE with Algorithm \ref{alg: esgd}}\label{alg:A-HPEesgd}
    
    \SetAlgoLined
    \textbf{Input: }{$\sigma_l<\sigma_u<\sigma<1$, $\sigma_l=\frac{\sigma_u}{2}$, $A_0=0$, $\errorB < \frac{(\sigma-\sigma_u)^2}{2\gamma_{k+1}(L\gamma_{k+1} + 1 +(\sigma-\sigma_u)^2)\left(L+\frac{1}{\gamma_{k+1}} \right)} \cdot \left(f(\tx_k) - \min_\y \left\{f_{\tx_k}(\y)+\frac{1}{2\gamma_{k+1}}\|\y-\tx_k\|^2\right\}\right)$, $k=0$, $\gamma_0 = \frac{\sigma_l(1-\sigma^2)^{1/2}}{16DH}$;}\\
    \While{$k<N$}{
        {$(\y_{k+1}, a_{k+1}, \gamma_{k+1})\gets\BSa(\tx_k, H, \sigma_l,\sigma_u, A_k, \gamma_k, \errorB)$};\\
        {$\bbv_{k+1}\gets\nabla f(\y_{k+1})$};\\
        $A_{k+1} \gets A_k+a_{k+1}$;\\
        $\x_{k+1} \gets \x_k - a_{k+1}\bbv_{k+1}$;\\
        $k \gets k+1$;
    }
\end{algorithm}

\begin{algorithm}[!h]
    \caption{$\BSa$: Binary search to find $\gamma_k$}\label{alg:A-HPEsearch}
    {\textbf{Input: }$(\tx_k, H, \sigma_l,\sigma_u, A_k, \gamma_k,\errorB)$};\\
    {$\gamma_{k+1}\gets\gamma_k$};\\
    \While{True}{
        {$ a_{k+1} \gets \frac{\gamma_{k+1}+\sqrt{\gamma_{k+1}^2+4\gamma_{k+1}A_{k}}}{2}$};\\
        {$\tilde\x_k \gets \frac{A_k}{A_k+a_{k+1}}\y_k + \frac{a_{k+1}}{A_k+a_{k+1}}\x_k$};\\
        {Solve \eqref{equ:APEsubproblem1} with Algorithm~\ref{alg: esgd}, and find an $\errorB$-approximated solution $\y_{k+1}$};\\
        \uIf{$\gamma_{k+1}\|\y_{k+1}-\tilde \x_k\|\le \frac{2\sigma_l}{H}$}
        {$\gamma_{k+1}\gets2\gamma_{k+1}$;}
        \uElseIf{$\gamma_{k+1}\|\y_{k+1}-\tilde \x_k\|\ge \frac{2\sigma_u}{H}$}
        {$\gamma_{k+1}\gets\frac{1}{2}\gamma_{k+1}$;}
        \Else
        {\Return{$(\y_{k+1}, a_{k+1}, \gamma_{k+1})$};\hfill{$\triangleright$ Require: $\frac{2\sigma_l}{H} \le \gamma_{k+1}\|\y_{k+1}-\tilde \x_k\|\le \frac{2\sigma_u}{H}$.}}
    }
\end{algorithm}

\subsubsection{Proof of Theorem~\ref{thm:convex}}
We give the proof of Theorem \ref{thm:convex} below.

\begin{proof}[Proof of Theorem~\ref{thm:convex}]
    In each call of Algorithm~\ref{alg:A-HPEsearch}, the problem \ref{equ:APEsubproblem1} is solved $\mathcal{O}\left(\left|\log \frac{\gamma_{k+1}}{\gamma_k}\right|\right)$ times, and in each time $\ltemp\le \max\{\gamma_k,\gamma_{k+1}\}$. Now we consider the gradient complexity of solving problem~\ref{equ:APEsubproblem1}. Denote $g(\y)=f_{\tx}(\y) + \frac{1}{2\ltemp}\|\y-\tx_k\|^2$. We use the eigen extractor in Algorithm~\ref{alg: eigenextrator} to extract some of the large eigenvectors and use accelerated methods to optimize the remainder of the problem. Specifically, $\lambda_l(\nabla^2 g(\y))\le \lambda_l (\nabla^2 f(\y))+\frac{1}{\ltemp}\le \frac{\tau_\alpha}{l^{\frac{1}{\alpha}}} + \frac{1}{\ltemp}$. As in the proof of Theorem~\ref{thm: quadratic}, we choose $k=\tilde\Theta\left(\tau_\alpha^{\frac{\alpha}{1+2\alpha}}\ltemp^{\frac{\alpha}{1+2\alpha}} \right)$. This requires $\tilde{\mathcal{O}}\left(\tau_\alpha^{\frac{\alpha}{1+2\alpha}}\ltemp^{\frac{\alpha}{1+2\alpha}}\right)$ gradient oracle calls. Applying the results of accelerated optimization problems, the optimization of the remainder term needs $\tilde{\mathcal{O}}\left(\sqrt{\left(\frac{\tau_\alpha}{k^{\frac{1}{\alpha}}}+\frac{1}{\ltemp}\right)\cdot \ltemp}\right)=\tilde{\mathcal{O}}\left(\tau_\alpha^{\frac{\alpha}{1+2\alpha}}\ltemp^{\frac{\alpha}{1+2\alpha}}\right)$ gradient oracle calls. Therefore, the overall number of gradient oracle calls is $\tilde{\mathcal{O}}\left(\tau_\alpha^{\frac{\alpha}{1+2\alpha}}\ltemp^{\frac{\alpha}{1+2\alpha}}\right)$.

    In order to find an $\epsilon$-approximated solution, we need to find the first $N$ such that $A_N\ge \frac{D^2}{\epsilon}$. Suppose that $A_N=\Theta\left(\frac{D^2}{\epsilon}\right)$. According to Theorem~\ref{thm:MS4.1}, $N=\tilde {\mathcal{O}}\left(D^{6/7}H^{2/7}\epsilon^{-2/7}\right)$ iterations. Ignoring all the logarithmic factors, the total gradient complexity is:
    \begin{equation}
        \begin{split}
            &\quad\sum_{k=1}^N \left(\tilde{\mathcal O}\left(\tau_\alpha\max\{\gamma_j,\gamma_{j+1}\} \right)^{\frac{\alpha}{1+2\alpha}} + 1 \right)\\
            &= \tilde{\mathcal{O}}\left(\tau_\alpha^{\frac{\alpha}{1+2\alpha}}\left(\sum_{k=1}^N \gamma_k^{\frac{\alpha}{1+2\alpha}} \right)+N\right)\\
            &\le \tilde{\mathcal{O}}\left(\tau_\alpha^{\frac{\alpha}{1+2\alpha}}\left(\sum_{k=1}^N \left(\sqrt{A_{k}}-\sqrt{A_{k-1}}\right)^{\frac{2\alpha}{1+2\alpha}} \right)+N\right)\\
            &\le \tilde{\mathcal{O}}\left(\tau_\alpha^{\frac{\alpha}{1+2\alpha}}\cdot N^{\frac{1}{1+2\alpha}}A_N^{\frac{\alpha}{1+2\alpha}}+N \right)\\
            &= \tilde{\mathcal{O}}\left(\tau_\alpha^{\frac{\alpha}{1+2\alpha}} D^{\frac{14\alpha+12}{14\alpha+7}} H^\frac{2}{14\alpha+7}\epsilon^{-\frac{7\alpha+2}{14\alpha+7}} +  D^{\frac{6}{7}}H^{\frac{2}{7}}\epsilon^{-\frac{2}{7}}\right).
        \end{split}
    \end{equation}

\end{proof}

\subsubsection{Useful Results in \cite{monteiro2013accelerated}}

We first present a theorem on the number of iterations of Algorithm \ref{alg:A-HPEesgd}, whose proof can be found in \cite{monteiro2013accelerated}:
\begin{theorem}[Theorem 4.1 in \cite{monteiro2013accelerated}]
    If all the parameters satisfy the requirements of Algorithm \ref{alg:A-HPEesgd}, then for every integer $1\le k\le n$, the following statements hold:
    \begin{equation}
        A_k\ge \left(\frac{2}{3}\right)^{7/2}\cdot\left(\frac{\sigma_l(1-\sigma^2)^{1/2}}{16DH}\right)\cdot k^{7/2},
    \end{equation}
    and
    \begin{equation}
        f(\y_k)-f^* \le \frac{3^{7/2}}{\sqrt{2}}\frac{HD^3}{\sigma_l\sqrt{1-\sigma^2}}\frac{1}{k^{7/2}}.
    \end{equation}
    \label{thm:MS4.1}
\end{theorem}

We present a new framework for considering errors from inexactly solving solutions. With Lemma \ref{lem:HPEprecision}, we show that if $\errorB$ is small enough, the $(\y_{k+1}, \gamma_{k+1})$ returned by Algorithm~\ref{alg:A-HPEsearch} meets the requirements in the A-NPE method in \cite{monteiro2013accelerated}, thus the results in \cite{monteiro2013accelerated} still hold.

\begin{lemma}
    If $$\errorB < \frac{(\sigma-\sigma_u)^2}{2\gamma_{k+1}(L\gamma_{k+1} + 1 +(\sigma-\sigma_u)^2)\left(L+\frac{1}{\gamma_{k+1}} \right)} \cdot \left(f(\tx_k) - \min_y \left\{f_{\tx_k}(\y)+\frac{1}{2\gamma_{k+1}}\|\y-\tx_k\|^2\right\}\right),$$ $\y_{k+1}$ satisfies 
    \begin{equation}
        \|\gamma_{k+1}\nabla f(\y_{k+1})+ \y_{k+1}-\tx_{k}\|^2\le \sigma^2\|\y_{k+1}-\tx_k^2\|.
    \end{equation}
    \label{lem:HPEprecision}
\end{lemma}
\begin{proof}[Proof of Lemma \ref{lem:HPEprecision}]
    Denote 
    \begin{equation}
        g(\y) = f_{\tilde \x_k}(\y) + \frac{1}{2\gamma_{k+1}}\|\y-\tilde \x_k\|^2.
        \label{equ:gdef}
    \end{equation}
    By the $L+\frac{1}{\gamma_{k+1}}$-Lipschitz contiouity of $\nabla g$, we have
    \begin{equation}
        g(\y)-g^*\ge \frac{1}{2\left(L+\frac{1}{\gamma_{k+1}}\right)} \|\nabla g(\y)\|^2.
        \label{equ:errorBpf}
    \end{equation}
    Let $\y=\y_{k+1}$ in \eqref{equ:errorBpf}. We have
    \begin{equation}
        \begin{split}
        \|\gamma_{k+1}\nabla f_{\tx_k}(\y_{k+1})+ \y_{k+1}-\tx_{k}\|^2 &\stackrel{\eqref{equ:gdef}}{=} \gamma_{k+1}^2 \|\nabla g(\y)\|^2\\
        &\stackrel{\eqref{equ:errorBpf}}{\le} \left(2L\gamma_{k+1}^2 + 2\gamma_{k+1}\right)(g(\y_{k+1})-g^*)\\
        &\le \left(2L\gamma_{k+1}^2 + 2\gamma_{k+1}\right)\errorB.
        \end{split}
        \label{equ:errorBpf1}
    \end{equation}
    The optimal solution to \eqref{equ:APEsubproblem1} is 
    \begin{equation}
    \begin{split}
        \y^* &= \tx_{k} - \left(\nabla^2 f(\tx_k) + \frac{1}{\gamma_{k+1}}\I\right)^{-1} \nabla f(\tx_k)\\
    \end{split}
    \end{equation}
    and
    \begin{equation}
    \begin{split}
        g^* &= f(\tx_k) - \frac{1}{2}\left\langle \left(\nabla^2 f(\tx_k) + \frac{1}{\gamma_{k+1}}\I\right)^{-1} \nabla f(\tx_k), \nabla f(\tx_k)\right\rangle\\
        &\stackrel{}{\ge} f(\tx_k) - \frac{1}{2}\left(L+\frac{1}{\gamma_{k+1}}\right)\|\tx_k-\y^*\|^2\\
        &\stackrel{}{\ge} f(\tx_k) - \left(L+\frac{1}{\gamma_{k+1}}\right)\left(\|\tx_k-\y_{k+1}\|^2+\|\y_{k+1}-\y^*\|^2 \right)\\
        &\stackrel{a}{\ge} f(\tx_k) - \left(L+\frac{1}{\gamma_{k+1}}\right)\left(\|\tx_k-\y_{k+1}\|^2+ 2\gamma_{k+1}\errorB \right),
    \end{split}
    \label{equ:errorBpf2}
    \end{equation}
    where $\stackrel{a}{\ge}$ uses the $\gamma_{k+1}$-strong convexity of $g$.
    Therefore, if $\errorB < \frac{(\sigma-\sigma_u)^2}{2\gamma_{k+1}(L\gamma_{k+1} + 1 +(\sigma-\sigma_u)^2)\left(L+\frac{1}{\gamma_{k+1}} \right)}\cdot(f(\tx_k)-g^*)$, we have
    \begin{align}
            &\|\gamma_{k+1}\nabla f_{\tx_{k}}+\y_{k+1}-\tx_k\|^2 \notag\\
\stackrel{\eqref{equ:errorBpf1}}{\le}&\left(2L\gamma_{k+1}^2 + 2\gamma_{k+1}\right)\errorB\notag\\
            \le& \frac{(\sigma-\sigma_u)^2}{L+\frac{1}{\gamma_{k+1}}}(f(\tx_k)-g^*)\notag\\
            &\quad+ \left(2L\gamma_{k+1}^2+2\gamma_{k+1}-(2L\gamma_{k+1}^2 + 2\gamma_{k+1}+2(\sigma-\sigma_u)^2\gamma_{k+1})\right)\errorB\notag\\
    \stackrel{\eqref{equ:errorBpf2}}{=} &
            (\sigma-\sigma_u)^2 \|\y_{k+1}-\tx_k\|^2, 
        \label{equ:errorBpf3}
    \end{align}
    and we have
    \begin{equation}
   \begin{split}
        &\quad \|\gamma_{k+1}\nabla f(\y_{k+1})+\y_{k+1}-\tx_k\|^2\\
        &=  \|(\gamma_{k+1}\nabla f_{\tx_k}(\y_{k+1})+\y_{k+1}-\tx_k)+(\gamma_{k+1}\nabla f_{\tx_k}(\y_{k+1})-\gamma_{k+1}\nabla f(\y_{k+1}))\|^2\notag\\
        &\le \|\gamma_{k+1}\nabla f_{\tx_k}(\y_{k+1})+\y_{k+1}-\tx_k\|^2 \notag\\
        &\quad + 2\|\gamma_{k+1}\nabla f(\y_{k+1})+\y_{k+1}-\tx_k\|\cdot\|\gamma_{k+1}\nabla f(\tx_k)\nabla^2 f(\tx_{k})(\y_{k+1}-\tx_k)-\gamma_{k+1}\nabla f(\y_{k+1}) \|\notag\\
        &\quad + \|\gamma_{k+1}\nabla f(\tx_k) + \nabla^2 f(\tx_{k})(\y_{k+1}-\tx_k)-\gamma_{k+1}\nabla f(\y_{k+1})\|^2\notag\\
        &\stackrel{\eqref{equ:errorBpf3}}{\le} (\sigma-\sigma_u)^2 \|\y_{k+1}-\tx_k\|^2 + 2(\sigma-\sigma_u) \|\y_{k+1}-\tx_k\|\cdot \frac{H\gamma_{k+1}\|\y_{k+1}-\tx_k\|}{2}\notag\\
        &\quad+ \left(\frac{H\gamma_{k+1}\|\y_{k+1}-\tx_k\|}{2}\right)^2 \notag\\
        &\le \left(\sigma - \sigma_u + \frac{H}{2}\cdot \frac{2\sigma_u}{H}\right)^2 = \sigma^2.\notag
   \end{split}
    \end{equation}
\end{proof}

\subsection{Non-convex Case}
\subsubsection{Algorithms}
An illustration of the  algorithm is shown in Algorithm \ref{alg:cubic}. To solve the subproblem,  we use a binary search to determine $r_k \approx \|\x_{k+1}-\x_k\|$.  With a given $r_k$,  the subproblem can be transferred to a quadratic minimization problem and is solvable by Algorithm~\ref{alg: esgd}. The whole algorithm is shown in Algorithm \ref{alg:cubicesgd}, where the updates use Algorithm \ref{alg:Cubic-search}. In this section, $c_1,c_2$ and $c$ are positive constants. 

\begin{algorithm}[!h]
    \caption{Illustration: Inexact Cubic Regularization Algorithm }\label{alg:cubic}
    \While{\text{stopping criterion is not met}}{
       Approximately solve the following optimization problem using Binary Search and Algorithm~\ref{alg: esgd}:
        {\begin{align}\x_{k+1}&\gets \argmin_\y f(\x_k) + \langle\nabla f(\x_k), \y-\x_k\rangle + \frac12\langle\nabla^2 f(\x_k) (\y-\x_k), \y-\x_k\rangle + \frac{H}{6}\|\y-\x_k\|^3.\notag
        \end{align}}
        $k\gets k+1$;
    }
\end{algorithm}

\begin{algorithm}[!h]
    \caption{Inexact Cubic Regularization Algorithm with Algorithm~\ref{alg: esgd}}\label{alg:cubicesgd}
    \textbf{Input: }{Desired accuracy $\epsilon$;}\\
    \While{$r_k\ge\sqrt{\frac{\epsilon}{H}}$}{
        {$(\x_{k+1}, r_{k+1})\gets \BSb(\x_k, H, r_k)$};\\
        {$k\gets k+1$};
    }
\end{algorithm}

\begin{algorithm}[!h]
    \caption{$\BSb$: Binary search to find $r_k$}\label{alg:Cubic-search}
    {\textbf{Input: }$(\x_k, H, \errorC)$};\\
    {Detect the smallest eigenvalue of $\nabla^2 f(\x_k)$: Compute $\lambda$ such that $\left|\lambda-\lambda_d(\nabla^2 f(\x_k))\right|<\frac{c_1}{2}\sqrt{\epsilon H}$};\\
    {$l_{k+1}\gets\max\{0,-\frac{2\lambda}{H}\}+(5c_1+2c_2)\sqrt{\frac{\epsilon}{H}},u_{k+1}\gets\infty$};\\
    {$\rtemp\gets l_{k+1}$};\\
    \While{True}{
        {Solve \eqref{equ:Cubicsubproblem1} with Algorithm~\ref{alg: esgd}, and find an $\errorC$-approximated solution $\y_{k+1}$:}
        \begin{equation}
            \min_{\y\in \R^d} f_{\x_k}(\y) + \frac{H\rtemp}{4}\|\y-\x_k\|^2.
            \label{equ:Cubicsubproblem1}
        \end{equation}
        \uIf{$\|\y_{k+1}-\x_k\| \le \rtemp-c_2\sqrt{\frac{\epsilon}{H}}$}{
        {$u_{k+1}\gets \rtemp$};\\
        \If{$u_{k+1}==l_{k+1}$}{\Return{$(\y_{k+1},\|\y_{k+1}-\x_k\|)$};}
        \While{True}{
        {$\elltemp \gets \frac{u_{k+1}}{2}$};\\
        {$l_{k+1}\gets \frac{u_{k+1}}{2}$};\\
        {$i\gets 0$};\\
        \While{True}{
            {$\rtemp\gets\frac{u_{k+1}+\elltemp}{2}$};\\
            {Solve \eqref{equ:Cubicsubproblem1} with Algorithm~\ref{alg: esgd}, and find an $\errorC$-approximated solution $\y_{k+1}$};\\
            \uIf{$\|\y_{k+1}-\x_k\| > \rtemp-c_2\sqrt{\frac{\epsilon}{H}}$}{
                {$\elltemp\gets\rtemp$};
            }
            \uElseIf{$\|\y_{k+1}-\x_k\| < l_{k+1} + c_2\sqrt{\frac{\epsilon}{H}}$}{
                {$r_{k+1}\gets \rtemp$};
            }
            \Else{
                \Return{$(\y_{k+1}, \|\y_{k+1}-\x_k\|)$};
            }
            $i\gets i+1$;\\
            \If{$i>K$}{Break;}
        }
        }
        }
        \ElseIf{$\|\y_{k+1}-\x_k\| > \rtemp-c_2\sqrt{\frac{\epsilon}{H}}$}{
        {$\rtemp\gets 2\rtemp$};
        }
    }
    
\end{algorithm}

\subsubsection{Detecting the smallest eigenvalue}

\begin{algorithm}[!h]
    \caption{Smallest Eigenvalue Finder I}\label{alg: smalleigdec}
    {\textbf{Input: }$\Ab\in\mathbb{S}^{d\times d}$, accuracy $\epsilon$.};\\
    {$\epsilon_1 = \mathcal{O}\left(\frac{d}{p^2}\right)$, $\epsilon_2 = \mathcal{O}\left( \frac{\epsilon}{d^2} \right)$, $l = \tilde{\mathcal{O}}(1)$};\\
    {Initial guess $\delta_0$ of leading eigenvalue of $u\Ib - \Ab$};\\
    {$k \leftarrow 0$};\\
    {Initialize $\wb_0$ with uniform distribution on $d$-dimensional sphere};\\
        \For{$t = 1,2,\cdots, l$}{
        {Find $\wb_t$ such that $\left\|\frac{1}{2}\wb^{\top}_t\left((\delta_k - u)\Ib + \Ab \right)\wb_t - \wb^{\top}_t\wb_{t-1}\right\|\leq \epsilon_1$};\\
        }
        {Find $\wb$ such that $\left\|\frac{1}{2}\wb^{\top}\left((\delta_k - u)\Ib + \Ab \right)\wb - \wb^{\top}\wb_{l}\right\|\leq \epsilon_1$};\\
        {$\Delta_k \leftarrow \frac{1}{2}\frac{1}{\wb_l\wb - \epsilon_1}$;}\\
    \While{True}{
        {$ s = 1$};\\
        {$\Ab_1^{(k)} \leftarrow (\delta_k - u)\Ib + \Ab$};\\
        {Apply shift-and-inverse on $\Ab_{s}^{(k)}$ with $\delta = \frac{1}{3}$, $\epsilon_2$ to obtain $\mathbf{v}_s$};\\
        {$a_s^{(k)} \leftarrow \mathbf{v}_s^{\top}\Ab_s^{(k)}\mathbf{v}_s$};\\
        \While{$a_s^{(k)} < \frac{3}{2}\left( \delta_k - u \right)$ and $s \leq \frac{1}{2}\sqrt{\frac{a_s^{(k)}}{\Delta_k}}$}{
        {$\Ab_{s+1}^{(k)} \leftarrow \Ab_s^{(k)} - \frac{a_s^{(k)}}{5}\mathbf{v}_s\mathbf{v}_s^{\top}$};\\
        {Applying shift-and-inverse to $\Ab_{s+1}^{(k)}$ with $\delta = \frac{1}{3}$ and $\epsilon_2$ to obtain $\mathbf{v}_{s+1}$};\\
        {$a_{s+1}^{(k)}\leftarrow \mathbf{v}_s^{\top}\Ab_{s+1}^{(k)}\mathbf{v}_s$};\\
        {$s \leftarrow s+1$};
        }
        {$s^{(k)}\leftarrow s$};\\
        \If{$\Delta_k \leq \frac{\epsilon}{3}$ (criteria 1)}{
        {{\bfseries Output:} $\delta_{k}$};
        }
        \If{$a_s^{(k)} \leq 2 (\delta_k - u)$ and $\Delta_k \leq \frac{1}{3}(\delta_k - u)$ (criteria 2)}{
        {{\bfseries Output:} $\Ab_{s^{(k)}}^{(k)}, a_{s^{(k)}}^{(k)}$};
        }
        {Initialize $\wb_0$ with uniform distribution on $d$-dimensional sphere};\\
        \For{$t = 1,2,\cdots, l$}{
        {Find $\wb_t$ such that $\left\|\frac{1}{2}\wb^{\top}_t\left((\delta_k - u)\Ib + \Ab \right)\wb_t - \wb^{\top}_t\wb_{t-1}\right\|\leq \epsilon_1$};\\
        }
        {Find $\wb$ such that $\left\|\frac{1}{2}\wb^{\top}\left((\delta_k - u)\Ib + \Ab \right)\wb - \wb^{\top}\wb_{l}\right\|\leq \epsilon_1$};\\
        {$\Delta_{k+1} \leftarrow \frac{1}{2}\frac{1}{\wb_l\wb - \epsilon_1}$ and $\delta_{k+1} = \delta_k - \frac{\Delta_k}{2}$};\\
        {$k\leftarrow k+1$};
    }
\end{algorithm}

\begin{algorithm}[!h]
    \caption{Smallest Eigenvalue Finder II}\label{alg: smalleigdec2}
    {\textbf{Input: }$\Mb\in\mathbb{S}^{d\times d}_{+}$, accuracy $\epsilon$.};\\
    {$k \leftarrow 0$, $l = \tilde{\mathcal{O}}(1)$, $\Mb_0 \leftarrow \Mb$, $\Ub \leftarrow ()$};\\
    {$\epsilon_3 = \mathcal{O}\left(\frac{\epsilon^2}{d^2}\right)$};\\
    {Apply shift-and-inverse to $\Mb$ with $\delta = \frac{1}{900}$ and $\epsilon_3$ to obtain $\hat{\delta}$ approximating $\|\Mb\|$};\\
    \While{True}{
        {$k \leftarrow k + 1$};\\
        {Applying shift-and-inverse to $\Mb_k$ with $\delta = \frac{1}{900}$ and $\epsilon_3$ to obtain $\mathbf{v}$};\\
        {$\mathbf{v} \leftarrow \left( (\Ib - \Ub_{k-1}\Ub_{k-1}^{\top})\mathbf{v} \right)/\|(\Ib - \Ub_{k-1}\Ub_{k-1}^{\top})\mathbf{v}\|$};\\
        {$b_k \leftarrow \mathbf{v}^{\top}\Mb\mathbf{v}$};\\
        {$\Ub_k\leftarrow (\Ub_{k-1}, \mathbf{v})$};\\
        {$\Mb_k \leftarrow(\Ib - \mathbf{v}\mathbf{v}^{\top})\Mb_{k-1}(\Ib - \mathbf{v}\mathbf{v}^{\top})$};\\
        \If{$k\geq\sqrt{\frac{\hat{\delta}}{\epsilon}}$}{
        {Apply shift-and-inverse on $\Mb$ with $\delta = \frac{\epsilon}{3\hat{\delta}}$ and $\epsilon_3$ to obtain $\mathbf{v}$};\\
        {{\bfseries Output:} $\mathbf{v}^{\top}\Mb\mathbf{v}$}.
        }
        \If{$b_k\leq \frac{19}{20}\hat{\delta}$
        }{
        {Sample $\Vb_0\in\mathbb{R}^{d\times k}$ with i.i.d. $\mathcal{N}\left(0,\frac{1}{d}\right)$ entries.};\\
        \For{$i = 1,2,\cdots, l$}{
        {$\Vb_i \leftarrow \Mb\Vb_{i-1}$};
        }
        {Perform QR decomposition on $\Vb_l$ and obtain $\Vb_l = \Qb\Rb$};\\
        {{\bfseries Output:} the largest eigenvalue of $\Qb^{\top}\Mb\Qb$}.
        }
        }
\end{algorithm}

\begin{theorem}[Finding Smallest Eigenvalue]\label{thm: smallesteigfind}
    For any matrix $\Ab\in\mathbb{S}^{d\times d}$ with $\|\Ab\|\leq 1$ satisfying the $(\alpha, \tau_{\alpha})$-degeneracy, with $\tilde{\mathcal{O}}\left( \epsilon^{-\frac{\alpha}{1+2\alpha}}\tau_{\alpha}^{\frac{\alpha}{1+2\alpha}} \right)$ gradient oracles, there exists an algorithm that finds $\hat{\lambda}$ satisfying $|\hat{\lambda} - \lambda_d(\Ab)|\leq \epsilon$ with high probability.
\end{theorem}

To prove Theorem~\ref{thm: smallesteigfind}, we propose a two stage algorithm the finds the smallest eigenvalue of the symmetric matrix $\Ab$.

The first stage is Algorithm~\ref{alg: smalleigdec}, which uses an shift-and-inverse method to find the smallest eigenvalue. And for each quadratic optimization problem of the shift-and-inverse problem, we follow the similar procedure as in the Section~\ref{sec: mainquadratic}, applying eigen extractor to find the large eigenvalue space and then performing accelerated proximal gradient descent. However, this method may fail to achieve the claimed gradient complexity when the smallest eigenvalue exceeds certain threshold. Specifically, the amount eigenvalue around $0$ can be up to $\Theta(d)$ even given the degeneracy result, and thus there can be at most $\Theta(d)$ eigenvalues of $\delta\Ib -\Ab$ that are around $\delta$. Therefore, when $\delta$ is large, we cannot find a low rank matrix that represents the large eigenvalue space of $\Ab$ even if the it is highly degenerated.

The second stage is Algorithm~\ref{alg: smalleigdec2}. When the first stage fails, we show that we can find a matrix $\Mb$ that enters the regime of Algorithm~\ref{alg: smalleigdec2} and its largest eigenvalue can be transformed into an approximation of the smallest eigenvalue of $\Ab$. Algorithm~\ref{alg: smalleigdec2} finds a largest eigenvalue of $\Mb$ and consists of two parts. According to the norm of the $\Mb$, Algorithm~\ref{alg: smalleigdec2} either directly applies the shift-and-inverse method, or finds a vector space that almost contains the largest eigenvector of $\Ab$, and then solves the full PCA of $\Ab$ on this vector space.

Before entering the proof of Theorem~\ref{thm: smallesteigfind}, we first state some useful lemmas. Lemma~\ref{lm: simultaneous_iteration} is used in Algorithm~\ref{alg: smalleigdec2}. It shows that if $\lambda_k(\Ab)\leq\rho\lambda_1(\Ab)$ for some constant $\rho< 1$, with $\tilde{\mathcal{O}}(k)$ gradient oracles, we can find a vector space that almost contains the largest eigenvector of $\Ab$.
\begin{lemma}[Simultaneous Iteration]\label{lm: simultaneous_iteration}
Let $\Ab\in\mathbb{S}_+^{d\times d}$ and $\lambda_{k+1}(\Ab) \leq \rho\lambda_1(\Ab)$ for some constant $\rho < 1$. Then for any accuracy $\epsilon$, there exist $l = \tilde{\mathcal{O}}\left(\frac{1}{1-\rho}\right)$ satisfies the following property. Let a random matrix $\Vb_0$ that has i.i.d. $\mathcal{N}\left(0,\frac{1}{d}\right)$ elements. $\Qb\Rb = \Ab^l\Vb_0$ is the QR decomposition of $\Ab^l\Vb_0$. And $\Qb$ satisfies
\begin{align*}
    \|\Qb\Qb^{\top}\ub\| \geq 1 - \epsilon.
\end{align*}
\end{lemma}

\begin{proof}
Write the eigenvalue decomposition of $\Ab$ as $\Ub\Lambdab\Ub^{\top}$. Let $\Lambdab_1\in\mathbb{R}^{k\times k}$ and $\Lambdab_1\in\mathbb{R}^{k\times k}$ be diagonal matrices and $\diag(\Lambdab_1) = (\lambda_1(\Ab),\lambda_2(\Ab),\cdots,\lambda_k(\Ab))$, $\diag(\Lambdab_2) = (\lambda_{k+1}(\Ab),\cdots,\lambdab_d(\Ab))$. Denote $\Ub = (\Ub_1,\Ub_2)$ where $\Ub_1 \in\mathbb{R}^{d\times k}, \Ub_2 = \Ub_1 \in\mathbb{R}^{d\times (d-k)}$. Then writs $\Vb_0 = \Ub\Lb$ and $\Lb = (\Lb_1,\Lb_2)^{\top}$ where $\Lb_1\in\mathbb{R}^{k\times k}, \Lb_1\in\mathbb{R}^{(d-k)\times k}$.

Then 
\begin{align*}
\Ab^k\Vb_0 = \Ub\left( 
\begin{array}{c}
     \Lambda_1^k\Lb_1   \\
     \Lambda_2^k\Lb_2 
\end{array}
\right) = \Ub\left(
\begin{array}{c}
     \Ib   \\
     \Lambda_2^k\Lb_2\Lb_1^{-1}\Lambda_1^{-k} 
\end{array}
\right) \Lambda_1^k\Lb_1.
\end{align*}
Since $\Qb\Rb$ is the QR decomposition of $\Ab^k \Vb_0$, $\Qb\Qb^{\top}\ub$ is the projection of $\ub$ on the column space of $\Ab^k\Vb_0$. The column space of $\Ab^k\Vb_0$ is equal to the one of $\Ub\left(
\begin{array}{c}
     \Ib   \\
     \Lambda_2^k\Lb_2\Lb_1^{-1}\Lambda_1^{-k} 
\end{array}
\right)$. Denote $\ab_1$ to be the first column of $\Ub\left(
\begin{array}{c}
     \Ib   \\
     \Lambda_2^k\Lb_2\Lb_1^{-1}\Lambda_1^{-k} 
\end{array}
\right)$. Then we have
\begin{align*}
    \Qb\Qb^{\top}\ub \geq \Qb\Qb^{\top}\left(\frac{\ab_1}{\|\ab_1\|}\right)\left(\frac{\ab_1}{\|\ab_1\|}\right)^{\top}\ub = \left(\frac{\ab_1}{\|\ab_1\|}\right)^{\top}\ub\left(\frac{\ab_1}{\|\ab_1\|}\right).
\end{align*}
And $\left(\frac{\ab_1}{\|\ab_1\|}\right)^{\top}\ub$ can be controlled by
\begin{align*}
    \left(\frac{\ab_1}{\|\ab_1\|}\right)^{\top}\ub = & \left(\ub_1 + \sum_{i = k+1}^d\frac{\lambda_i^k}{\lambda_1^k}\left(\Lb_2\Lb_1^{-1}\right)_{i,1}\ub_i\right)^{\top}\ub_1 / \left\|\ub_1 + \sum_{i = k+1}^d\frac{\lambda_i^k}{\lambda_1^k}\left(\Lb_2\Lb_1^{-1}\right)_{i,1}\ub_i\right\|\\
    \geq & 1/\left(1+\rho^k\sum_{i = k+1}^d\left(\Lb_2\Lb_1^{-1}\right)_{i,1}\right).
\end{align*}
Since $\Lb_1$ is a Gaussian ensemble, again by Lemma~\ref{lm: wishartprop}, we have $(\lambda_{\min}(\Lb_1))^2 = \frac{1}{\mathrm{poly}\left(d,\frac{1}{p}\right)}$, where $p$ is the failure probability. Besides, we have $\|\Lb_2\|_F = \mathrm{poly}\left(d,\frac{1}{p}\right)$ directly. A simple bound on $\sum_{i = k+1}^d\left(\Lb_2\Lb_1^{-1}\right)_{i,1}$ indicates that $\sum_{i = k+1}^d\left(\Lb_2\Lb_1^{-1}\right)_{i,1} \leq \|\Lb_2\|_F/\lambda_{\min}(\Lb_1) = \mathrm{poly}\left(d,\frac{1}{p}\right)$. Therefore, with $\log_{\rho}\left(\frac{\epsilon}{\mathrm{poly}(d,1/p)}\right) = \tilde{\mathcal{O}}\left( \frac{1}{1-\rho} \right)$ iterations, we have 
\begin{align*}
    \|\Qb\Qb^{\top}\ub\| \geq \left\| \left(\frac{\ab_1}{\|\ab_1\|}\right)^{\top}\ub\left(\frac{\ab_1}{\|\ab_1\|}\right) \right\|\geq \frac{1}{1+\epsilon}\geq 1-\epsilon,
\end{align*}
which completes the proof.
\end{proof}

Note that after neglecting the method solving the quadratic problem, the update of $\delta$, $\Delta$ in Algorithm~\ref{alg: smalleigdec} follows the course of shift-and-inverse method of $u\Ib - \Ab$. Therefore we have the following Lemma adapted from \cite{garber2015fast}. 
\begin{lemma}[Lemma 4.2 of \cite{garber2015fast}]\label{lm: inexectpower}
    With high probability, for any iteration $k$, we have $u$, $\delta_k$, $\Delta_k$ in Algorithm~\ref{alg: smalleigdec} satisfy:
    \begin{itemize}
        \item[(a)] $0\leq \frac{1}{2}\left( \delta_s - u + \lambda_d(\Ab) \right)\leq \Delta_s \leq \delta_s - u + \lambda_d(\Ab)$;
        \item[(b)] $\Delta_{k+1} = \Theta(\Delta_k)$;
        \item[(c)] $\delta_k - u \leq \delta_{k-1} - u$.
    \end{itemize}
\end{lemma}
Besides, ignoring the stopping criteria and the followup procedure, Algorithm~\ref{alg: smalleigdec2} follows the course of Lazysvd of \cite{allen2016lazysvd} to perform the adaptive search on the eigenvalue. And thus $b_k$ in Algorithm~\ref{alg: smalleigdec2} satisfies the following lemma.
\begin{lemma}[Theorem~4.1 of \cite{allen2016lazysvd}]\label{lm: lazysvd}
Let $\delta$ be the parameter of the shift-and-inverse subroutine in Algorithm~\ref{alg: smalleigdec2}. Then with high probability, for each iteration $k$, $b_k$ in Algorithm~\ref{alg: smalleigdec2} satisfies that
\begin{align*}
    (1-2\delta)\lambda_k(\Mb)\leq b_k\leq \frac{\lambda_k(\Mb)}{1-2\delta}.
\end{align*}
\end{lemma}


\noindent\textbf{Proof of Theorem~\ref{thm: smallesteigfind}.} Equipped with the above lemmas, we begin our proof of Theorem~\ref{thm: smallesteigfind}. The proof is a combination of the following Lemma~\ref{lm: smalleigfind1} and \ref{lm: smalleigfind2}. These two lemma corresponds the two stages of our algorithm.
\begin{lemma}[Output of Algorithm~\ref{alg: smalleigdec}]\label{lm: smalleigfind1}
    Let matrix $\Ab\in\mathbb{S}^{d\times d}$ and $\epsilon>0$ be the desired accuracy. Assume $\Ab$ is $(\alpha, \tau_{\alpha})$-degenerated, and we can assume $\epsilon\leq\tau_{\alpha}$ without loss of generality. Then with high probability, the output Algorithm~\ref{alg: smalleigdec} satisfies that: if algorithm stops with criteria 1, $|u - \delta_{k} - \lambda_d(\Ab)|\leq \epsilon$; if algorithm stops with criteria 2, we have
    \begin{itemize}
    \item[0.] $\lambda_d(\Ab)\leq 0$.
    \item[1.] $2a_{s^{(k)}}^{(k)}\Ib - \Ab_{s^{(k)}}^{(k)} \succeq \mathbf{0}$.
    \item[2.] $a_{s^{(k)}}^{(k)} \leq -6\lambda_d(\Ab)$ 
    \item[3.] $\left|\lambda_l\left(2a_{s^{(k)}}^{(k)}\Ib - \Ab_{s^{(k)}}^{(k)}\right) - (2a_{s^{(k)}}^{(k)} - \delta_k + u - \lambda_{d-l+1}(\Ab))\right|\leq\frac{\epsilon}{2}$ for any $l$ satisfying $\lambda_{d+1-l}(\Ab)\leq 0$.
    \end{itemize}
    And the total gradient oracle calls of Algorithm~\ref{alg: smalleigdec} is $\tilde{\mathcal{O}}\left( \tau_{\alpha}^{\frac{\alpha}{1+2\alpha}}\epsilon^{-\frac{\alpha}{1+2\alpha}} \right)$.
\end{lemma}
\begin{proof}



We first show that the procedure takes $\tilde{\mathcal{O}}\left( \tau_{\alpha}^{\frac{\alpha}{1+2\alpha}}\epsilon^{-\frac{\alpha}{1+2\alpha}} \right)$ gradient oracle calls. First, we claim the inexact power method outside the iteration can be solved within $\tilde{\mathcal{O}}(1)$ gradient oracle calls since we can find a large enough initial $\delta_0$ to ensure a constant level condition number of this initial quadratic problem. 

Note that each iteration of the algorithm can be split into following two parts:
\begin{itemize}
    \item[(a)] Eigen extractor of $(\delta_k - u)\Ib + \Ab$;
    \item[(b)] A series of quadratic problems that consists the inexact power method.
\end{itemize}

By Lemma~\ref{lm: shiftninverse}, each iteration of the eigen extractor part takes $\tilde{\mathcal{O}}(1)$ gradient oracle calls. We now begin with showing that this procedure will stop in $\tilde{\mathcal{O}}\left(\tau_{\alpha}^{\frac{\alpha}{1+2\alpha}}\epsilon^{-\frac{\alpha}{1+2\alpha}}\right)$ iterations. For any $k$ in the iterations, consider $s$ satisfying $a_s^{(k)} > \frac{3}{2}(\delta_k - u)$ and $s < \frac{1}{2}\sqrt{\frac{a_s^{(k)}}{\Delta_k}}$. By Theorem~\ref{thm: eigenextractor}, there exists $l = \tilde{\Theta}(s)$, such that $a_s^{(k)} = \mathcal{O}\left( \lambda_l\big((\delta_k - u)\Ib + \Ab\big) \right)$. We have the following control on $a_s^{(k)}$
\begin{align}\label{eq: escsadl_aest}
    a_s^{(k)} = \mathcal{O}\left( \lambda_l\big((\delta_k - u)\Ib + \Ab\big) \right) \overset{\text{a}}{=} \mathcal{O}\left( \lambda_l(\Ab) \right) \overset{\text{b}}{=} \mathcal{O}\left( \frac{\tau_{\alpha}}{l^{1/\alpha}} \right) \overset{\text{c}}{=} \tilde{\mathcal{O}}\left( \frac{\tau_{\alpha}}{s^{1/\alpha}} \right).
\end{align}
where the $\overset{\text{a}}{=}$ follows from $a_s^{(k)} >\frac{3}{2}(\delta_k - u)$ and $a_s\geq 0$; $\overset{\text{b}}{=}$ follows from the degeneracy condition; $\overset{\text{c}}{=}$ follows from $l = \tilde{\Theta}(s)$. Plugging \eqref{eq: escsadl_aest} into $s = \mathcal{O}\left( \sqrt{\frac{a_s^{(k)}}{\Delta_k}} \right)$ yields
\begin{align}\label{eq: smlfind1sbound}
    s = \tilde{\mathcal{O}}\left( \frac{\tau_{\alpha}}{s^{1/\alpha}\Delta_k} \right),
\end{align}
which indicates that 
\begin{align*}
s \overset{\text{a}}{=} \tilde{\mathcal{O}}\left( \tau_{\alpha}^{\frac{\alpha}{1+2\alpha}}\Delta_k^{-\frac{\alpha}{1+2\alpha}}\right) \overset{\text{b}}{=} \tilde{\mathcal{O}}\left( \tau_{\alpha}^{\frac{\alpha}{1+2\alpha}}\Delta_{k-1}^{-\frac{\alpha}{1+2\alpha}}\right) \overset{\text{c}}{=} \tilde{\mathcal{O}}\left( \tau_{\alpha}^{\frac{\alpha}{1+2\alpha}}\epsilon^{-\frac{\alpha}{1+2\alpha}}\right),
\end{align*}
where $\overset{\text{a}}{=}$ is an immediate consequence of \eqref{eq: smlfind1sbound}; $\overset{\text{b}}{=}$ follows from Lemma~\ref{lm: inexectpower} that $\Delta_k = \Theta(\Delta_{k+1})$; $\overset{\text{c}}{=}$ follows from the stopping criteria 1 and Algorithm~\ref{alg: smalleigdec} does not stop at the $(k-1)$-th iteration. This shows that the eigen extractor part takes $\tilde{\mathcal{O}}\left( \tau_{\alpha}^{\frac{\alpha}{1+2\alpha}}\epsilon^{-\frac{\alpha}{1+2\alpha}}\right)$ gradient oracle calls.

After the eigen extractor part, we solve $\tilde{\mathcal{O}}(1)$ quadratic optimization method with quadratic problem having the same quadratic term at each iteration. We can solve these quadratic optimization problems by accelerated proximal gradient method that uses the gradient of
$\frac{1}{2}\mathbf{w}^{\top} \Ab_{s^{(k)}}^{(k)} \mathbf{w} - \mathbf{w}^{\top}\bbb$ and solves the proximal operator on $\frac{1}{2}\mathbf{w}^{\top}\left( (\delta_k - u)\Ib + \Ab - \Ab_{s^{(k)}}^{(k)} \right)\mathbf{w}$ for any $\bbb\in\mathbb{R}^d$.


Given the stopping criteria $s \geq \frac{1}{2}\sqrt{\frac{a_s^{(k)}}{\Delta_k}}$, for any $s < s^{(k)}$, we have $\lambda_1(\Ab_s^{(k)})\geq a_s^{(k)} \geq 4s^2\Delta_k \geq 4\Delta_k$. Further, Lemma~\ref{lm: inexectpower} indicates that $\Delta_k\geq \frac{1}{2}\lambda_d((\delta_k - u)\Ib+\Ab)\geq \frac{1}{2}\lambda_d(\Ab_s^{(k)})$. Thus $\lambda_1(\Ab_s^{(k)})\geq 2\lambda_d(\Ab_s^{(k)})$, and a similar argument as the smallest eigenvalue control in Theorem~\ref{thm: eigenextractor} shows that the $\lambda_d(\Ab_s^{(k)}) = \Theta(\lambda_d((\delta_k - u)\Ib - \Ab))$ for any $s\leq s^{(k)}$. (Here $s = s^{(k)}$ holds since the we do not extract $a^{(k)}_{s^{(k)}}/5\mathbf{v}_{s^{(k)}}\mathbf{v}_{s^{(k)}}^{\top}$ from $\Ab_{s^{(k)}}^{(k)}$.) Then by Lemma~\ref{lm: shiftninverse} and \ref{lm: inexectpower}, the condition number of the matrix $\Ab_s$ is bounded by $\mathcal{O} \left(a_{s^{(k)}}^{(k)}/\Delta_k\right)$.
    
For any $k$ in the iteration, if $a_{s^{(k)}}^{(k)} > 2(\delta_k - u)$, the eigen extractor step stops with stopping criteria $s \geq \frac{1}{2}\sqrt{\frac{a_s^{(k)}}{\Delta_k}}$. By the analysis in the previous part we have $s^{(k)} = \tilde{\mathcal{O}}\left(\tau_{\alpha}^{\frac{\alpha}{1+2\alpha}}\epsilon^{-\frac{\alpha}{1+2\alpha}}\right)$. Then the gradient complexity of the optimization is bounded by
\begin{align*}
    \tilde{\mathcal{O}}\left(\sqrt{a_{s^{(k)}}^{(k)}/\Delta_k}\right) = \tilde{\mathcal{O}}\left(s^{(k)}\right) = \tilde{\mathcal{O}}\left( \tau_{\alpha}^{\frac{\alpha}{1+2\alpha}} \epsilon^{-\frac{\alpha}{1+2\alpha}}  \right).
\end{align*}
Therefore the quadratic problem can be solved within $\tilde{\mathcal{O}}\left( \tau_{\alpha}^{\frac{\alpha}{1+2\alpha}} \epsilon^{-\frac{\alpha}{1+2\alpha}} \right)$ gradient oracle calls. 



Then we consider the case where $a_{s^{(k)}}^{(k)} \leq 2(\delta_k - u)$. Since the the last iteration (here we refer to the last iteration of $k$) of the algorithm may conflict with the property of the other iterations, we consider bounding the gradient complexity of the quadratic problem in the $k$-th iteration by: (1) considering the previous iteration; (2) bound the difference between two iterations.

If $a_{s^{(k-1)}}^{(k-1)} \leq 2(\delta_{(k-1)} - u)$, since the algorithm does not stop at the previous iteration, we have $\Delta_{k-1} \geq \frac{1}{3}(\delta_{k-1} - u)$. Then, we have
\begin{align*}
    \frac{a_{s^{(k)}}^{(k)}}{\Delta_k} \overset{\text{a}}{=} \mathcal{O}\left( \frac{\delta_k - u}{\Delta_{k-1}} \right) \overset{\text{b}}{=} \mathcal{O}\left( \frac{\delta_{k-1} - u}{\Delta_{k-1}} \right) = \mathcal{O}(1),
\end{align*}
where $\overset{\text{a}}{=}$ follows from $a_{s^{(k)}}^{(k)} \leq 2\left(\delta_k - u\right)$ in the interested case and $\Delta_k = \Theta(\Delta_{k-1})$; $\overset{\text{a}}{=}$ follows from $\delta_k \leq \delta_{k-1}$. If $a_{s^{(k-1)}}^{(k-1)}> 2(\delta_{k-1} - u)$,
\begin{align*}
    \frac{a_{s^{(k)}}^{(k)}}{\Delta_k} = \mathcal{O}\left( \frac{\delta_k - u}{\Delta_{k-1}} \right) = \mathcal{O}\left( \frac{\delta_{k-1} - u}{\Delta_{k-1}} \right) = \mathcal{O}\left(\frac{a_{s^{(k-1)}}^{(k-1)}}{\Delta_{k-1}}\right).
\end{align*}
From the previous analysis of the situation $a_s\geq2(\delta_k - u)$, we have $a_{s^{(k-1)}}^{(k-1)}/\Delta_{k-1} = \tilde{\mathcal{O}}\left( \tau_{\alpha}^{\frac{2\alpha}{1+2\alpha}}\epsilon^{-\frac{2\alpha}{1+2\alpha}} \right)$. These analysis conclude that the gradient oracle calls of the optimization problem is bounded by $\tilde{\mathcal{O}}\left( \tau_{\alpha}^{\frac{\alpha}{1+2\alpha}}\epsilon^{-\frac{\alpha}{1+2\alpha}} \right)$.


Combining the analysis from the eigen extractor part and quadratic problem part, the gradient complexity of each iteartion is $\tilde{\mathcal{O}}\left( \tau_{\alpha}^{\frac{\alpha}{1+2\alpha}} \epsilon^{-\frac{\alpha}{1+2\alpha}} \right)$. To obtain the claimed gradient complexity, it remains proving a $\tilde{\mathcal{O}}(1)$ iteration complexity bound of Algorithm~\ref{alg: smalleigdec}. Noting that by Lemma~\ref{lm: inexectpower} $\Delta_k \geq \frac{1}{2}(\delta_s - u + \lambda_d(\Ab))$, 
\begin{align*}
    \delta_{k+1} - (u - \lambda_d) \leq \delta_k - (u-\lambda_d(\Ab)) - \frac{1}{4}(\delta_k - u + \lambda_d(\Ab)) = \frac{3}{4}(\delta_k - (u - \lambda_d(\Ab))).
\end{align*}
Again by Lemma~\ref{lm: inexectpower}, $\Delta_k \leq \delta_{k} - (u - \lambda_d) \leq \left(\frac{3}{4}\right)^k \left( \delta_0 - (u - \lambda_d)\right)$, which implies that the algorithm will stop within $\tilde{\mathcal{O}}(1)$ iterations.

Then we prove the claim properties of Algorithm~\ref{alg: smalleigdec} output. If Algorithm~\ref{alg: smalleigdec} stops with criteria 1, then Lemma~\ref{lm: inexectpower} shows that $\delta_k - u + \lambda_d(\Ab) \leq \Delta_k \leq \frac{\epsilon}{3}$ and $\delta_k - u + \lambda_d(\Ab) \geq 0$. This indicates that $\lambda_d(\Ab) - \epsilon \leq u - \delta_k \leq \lambda_d(\Ab)$. Thus we obtain an $\epsilon$ approximate estimate of $\lambda_d(\Ab)$. 





Then we consider the case when Algorithm~\ref{alg: smalleigdec} stops with criteria 2. For claim 0, if $\lambda_d(\Ab)\geq 0$, by Lemma~\ref{lm: inexectpower}, $\Delta_k \geq \frac{1}{2}\left( \delta_k - u + \lambda_d(\Ab) \right) \geq \frac{1}{2}\left( \delta_k - u\right) > \frac{1}{3}\left( \delta_k - u\right)$. This violates the condition that the algorithm stops with criteria 2 and therefore $\lambda_d(\Ab) < 0$. For claim 1, by Lemma~\ref{lm: shiftninverse}, $\lambda_1\left(\Ab_{s^{(k)}}^{(k)}\right) \leq \frac{a_{s^{(k)}}^{(k)}}{\left(1-\frac{1}{3}\right)(1-\epsilon_2)}\leq 2a_{s^{(k)}}^{(k)}$. 

For claim 2, combining the stopping criteria $\Delta_k \leq \frac{1}{3}(\delta_k - u)$ and Lemma~\ref{lm: inexectpower} yields $\frac{1}{2}(\delta_k - u + \lambda_d(\Ab)) \leq \Delta_k \leq \frac{1}{3}(\delta_k - u)$, which indicates $\delta_k - u \leq - 3\lambda_d(\Ab)$. Thus we have $a_s \leq 2(\delta_k - u)\leq -6\lambda_d(\Ab)$, which proves claim 2.

For claim 3, we have
\begin{align*}
    \lambda_l\left(2a_{s^{(k)}}^{(k)}\Ib - \Ab_{s^{(k)}}^{(k)}\right) = & \lambda_l\left(2a_{s^{(k)}}^{(k)}\Ib - \Ab_1\right) + \lambda_l\left(2a_{s^{(k)}}^{(k)}\Ib - \Ab_{s^{(k)}}^{(k)}\right) - \lambda_l\left(2a_{s^{(k)}}^{(k)}\Ib - \Ab_1\right)\\
    = & 2a_{s^{(k)}}^{(k)} - \delta_k + u -\lambda_{d-l+1}(\Ab) + \left(\lambda_{d-l+1}(\Ab_1) - \lambda_{d-l+1}\left(\Ab_{s^{(k)}}^{(k)}\right)\right).
\end{align*}
For any $s<s^{(k)}$ $a_s^{(k)} \geq \frac{3}{2}(\delta_k - u)$ and by Lemma~\ref{lm: shiftninverse}, we have $\frac{1}{1-\frac{1}{3}}\cdot\frac{3}{2}(\delta_k - u)\geq \delta_k - u$. Further, for any $l$ satisfying $\lambda_{d-l+1}(\Ab)\leq 0$, we have
\begin{align*}
    \lambda_{d-l+1}\left(\Ab_s^{(k)}\right) \leq \lambda_{d-l+1}(\Ab_1) = (\delta_k - u) + \lambda_{d-l+1}(\Ab_1) \leq \delta_k - u. 
\end{align*}
Similar to the proof of Theorem~\ref{thm: eigenextractorgen}, iteratively applying Lemma~\ref{lm: shiftninverse} yields $\left\|\lambda_{d-l+1}\left(\Ab_{^{(k)}}^{(k)}\right) - \lambda_{d-l+1}(\Ab_1)\right\|\leq s^{(k)}\epsilon_2\leq \frac{\epsilon}{2}$ (since $s^{(k)} \leq d$), which proves claim 3.


\end{proof}

In the following lemma, we omit the $(k)$ superscript on $s$ and $\Ab_s$ indicating the iteration of Algorithm~\ref{alg: smalleigdec}, since we only consider the stopping iteration. And we apply Algorithm~\ref{alg: smalleigdec2} on $\Mb = 2a_s\Ib - \Ab_s$.

\begin{lemma}[Output of Algorithm~\ref{alg: smalleigdec2}]\label{lm: smalleigfind2}
Assume $\Ab\in\mathbb{S}^{d\times d}$ and satisfies the $(\alpha,\tau_{\alpha})$-degeneracy condition. Let $\epsilon$ be the desired accuracy, we assume $\tau_{\alpha} \geq \epsilon$ with loss of generality. If $a_s$ and $\Ab_s$ are the output of Algorithm~\ref{alg: smalleigdec} when stopping with criteria 2. Then with high probability, applying Algorithm~\ref{alg: smalleigdec} on $\Mb = 2a_s\Ib - \Ab_s$ generates $\hat{\delta}$ that satisfies $|(-\hat{\delta} + 2a_s - (\delta_k - u)) - \lambda_d(\Ab)|\leq \epsilon$ with $\tilde{\mathcal{O}}\left(\tau_{\alpha}^{\frac{\alpha}{1+2\alpha}}\epsilon^{-\frac{\alpha}{1+2\alpha}}\right)$ gradient oracle calls.
\end{lemma}

\begin{proof}


We first show that the algorithm will stop within $\tilde{\mathcal{O}}\left( \tau_{\alpha}^{\frac{\alpha}{1+2\alpha}}\epsilon^{-\frac{\alpha}{1+2\alpha}} \right)$ iterations. We consider two scenarios. The first one is that $\lambda_d(\Ab) \leq - 2\tau_{\alpha}^{\frac{2\alpha}{1+2\alpha}}\epsilon^{\frac{1}{1+2\alpha}}$. We consider $l= 2\tau_{\alpha}^{\frac{\alpha}{1+2\alpha}}\epsilon^{-\frac{\alpha}{1+2\alpha}}$. We assume that $\tau_{\alpha}^{\frac{2\alpha}{1+2\alpha}}\epsilon^{\frac{1}{1+2\alpha}} = \Omega(\epsilon)$, otherwise we have $\sqrt{\frac{\hat{\delta}}{\epsilon}} = \mathcal{O}(1)$ and the algorithm will stop within $\mathcal{O}(1)$ iterations because of the stopping criteria $k\geq \sqrt{\frac{\hat{\delta}}{\epsilon}}$. When $\lambda_{d-l+1}(\Ab)\leq 0$, we have
\begin{align*}
    l|-\lambda_{d-l+1}(\Ab)|^{\alpha} \leq \sum_{i=1}^l|-\lambda_{d-i+1}(\Ab)|^{\alpha}\leq \tau_{\alpha}^{\alpha},
\end{align*}
which indicates that $-\lambda_{d-l+1}(\Ab)\leq \frac{1}{2} \tau_{\alpha}^{\frac{2\alpha}{1+2\alpha}}\epsilon^{\frac{1}{1+2\alpha}} \leq -\frac{1}{4}\lambda_d(\Ab)$. Provided this analysis, we have
\begin{align}\label{eq: boundlargeeignum}
\begin{aligned}
     \frac{49}{52}(-\lambda_d(\Ab)) - (-\lambda_{d-l+1}(\Ab)) \overset{\text{a}}{\geq} & \frac{9}{13}(-\lambda_d(\Ab))\\
    \overset{\text{b}}{\geq} & \frac{1}{6}\cdot\frac{9}{13}a_s\\
    \overset{\text{c}}{\geq} & \frac{3}{52}(2a_s - (\delta_k - u)),
\end{aligned}
\end{align}
where $\overset{\text{a}}{\geq}$ is by the control of $\lambda_{d-l+1}(\Ab)$; $\overset{\text{b}}{\geq}$ follows from the property of the Algorithm~\ref{alg: smalleigdec} output; $\overset{\text{c}}{\geq}$ follows from Lemma~\ref{lm: inexectpower} and property of the Algorithm~\ref{alg: smalleigdec} output, which shows that $\delta_k - u\geq -\lambda_d \geq 0$. Rewriting \eqref{eq: boundlargeeignum} yields $ 2a_s - (\delta_k - u) - \lambda_{d-l+1}(\Ab) \leq \frac{49}{52}\left( 2a_s - (\delta_k - u) - \lambda_d(\Ab) \right)$. By the property of Algorithm~\ref{alg: smalleigdec} output,
\begin{align}\label{eq: smalleigdec2rho1}
\begin{aligned}
    \lambda_l(\Mb) = \lambda_l(2a_s\Ib - \Ab_s) \overset{\text{a}}{\leq} & 2a_s - (\delta_k - u) - \lambda_{d-l+1}(\Ab) + \frac{\epsilon}{2}\\
    \leq & \frac{49}{52}\left( 2a_s - (\delta_k - u) - \lambda_d(\Ab) \right) + \frac{\epsilon}{2}\\
    \overset{\text{b}}{\leq} & \frac{49}{52}\lambda_1(2a_s\Ib - \Ab_s) + \frac{101}{104}\epsilon\\
    \overset{\text{c}}{\leq} & \frac{17}{18}\lambda_1(2a_s\Ib-\Ab_s) = \frac{17}{18}\lambda_1(\Mb),
\end{aligned}
\end{align}
where $\overset{\text{a}}{\leq}$ and $\overset{\text{b}}{\leq}$ come from the property of the output of Algorithm~\ref{alg: smalleigdec}; $\overset{\text{c}}{\leq}$ comes from $\tau_{\alpha}^{\frac{2\alpha}{1+2\alpha}}\epsilon^{\frac{1}{1+2\alpha}} = \Omega(\epsilon)$. And when $\lambda_{d-l+1}(\Ab) > 0$, we have $\lambda_{d-l+1}(\Ab_s)\geq (\delta_k-u) - \frac{\epsilon}{2}$. And thus a similar argument to \eqref{eq: smalleigdec2rho1} states that
\begin{align*}
    \lambda_l(\Mb) = \lambda_l(2a_s\Ib - \Ab_s) \leq & 2a_s - (\delta_k - u) + \frac{\epsilon}{2}\\
    \leq & \frac{12}{13}\left(2a_s - (\delta_k - u) - \lambda_d(\Ab)\right) + \frac{\epsilon}{2}\\
    \leq & \frac{12}{13}\lambda_1(2a_s\Ib - \Ab_s) + \frac{25}{26}\epsilon\\
    \leq & \frac{17}{18}\lambda_1(2a_s\Ib - \Ab_s) = \frac{17}{18}\lambda_1(\Mb).
\end{align*}
Given that $\lambda_l(\Mb) \leq \frac{17}{18}\lambda_1(\Mb)$, by Lemma~\ref{lm: lazysvd}, we have $b_l \leq \frac{1}{1-\frac{1}{450}}\lambda_l(\Mb)\leq \frac{1}{1-\frac{1}{450}}\cdot\frac{17}{18}\lambda_1(\Mb) \leq \frac{18}{19} \frac{1}{1-1/450}\hat{\delta}\leq \frac{19}{20}\hat{\delta}$, which meets the stopping criteria $b_k\leq \frac{19}{20}\hat{\delta}$. Since we set $l = 2 \tau_{\alpha}^{\frac{\alpha}{1+2\alpha}}\epsilon^{-\frac{\alpha}{1+2\alpha}}$, this indicates that the algorithm will stop in $\mathcal{O}\left(\tau_{\alpha}^{\frac{\alpha}{1+2\alpha}}\epsilon^{-\frac{\alpha}{1+2\alpha}}\right)$ iterations.









The second scenario is that $\lambda_d(\Ab) \geq - 2\tau_{\alpha}^{\frac{2\alpha}{1+2\alpha}}\epsilon^{\frac{1}{1+2\alpha}}$. In this case, the iteration upper bound in (criteria 1) satisfies that  $\sqrt{\frac{\hat{\delta}}{\epsilon}} \overset{\text{a}}{=} \mathcal{O}\left(\sqrt{\frac{-\lambda_d(\Ab)}{\epsilon}}\right) = \mathcal{O}\left( \sqrt{\frac{\tau_{\alpha}^{\frac{2\alpha}{1+2\alpha}}\epsilon^{\frac{1}{1+2\alpha}}}{\epsilon}} \right) = \mathcal{O}\left( \tau_{\alpha}^{\frac{\alpha}{1+2\alpha}}\epsilon^{-\frac{\alpha}{1+2\alpha}} \right)$, where $\overset{\text{a}}{=}$ comes from $\hat{\delta} = \mathcal{O}(\|\Mb\|) = \mathcal{O}(a_s) = \mathcal{O}(-\lambda_d(\Mb))$. This leads to a $\tilde{\mathcal{O}}\left(\tau_{\alpha}^{\frac{\alpha}{1+2\alpha}}\epsilon^{-\frac{\alpha}{1+2\alpha}}\right)$ iteration bound.

If the algorithm stops with criteria $k\geq\sqrt{\frac{\hat{\delta}}{\epsilon}}$, by Lemma~\ref{lm: shiftninverse}, the shift-and-inverse algorithm takes $\tilde{\mathcal{O}}\left(\sqrt{\frac{\|\Mb\|}{\epsilon}}\right) = \tilde{\mathcal{O}}\left(\sqrt{\frac{\hat{\delta}}{\epsilon}}\right) = \tilde{\mathcal{O}}\left( \tau_{\alpha}^{\frac{\alpha}{1+2\alpha}}\epsilon^{-\frac{\alpha}{1+2\alpha}} \right)$ gradient oracle calls. And also by Lemma~\ref{lm: shiftninverse}, the output satisfies $\lambda_1(\Mb) - \frac{1}{2}\epsilon\leq(1-\frac{\epsilon}{3\hat{\delta}})(1-\epsilon_3)\lambda_1(\Mb)\leq\mathbf{v}^{\top}\Mb\mathbf{v}\leq \lambda_1(\Mb)$. Thus we obtain an $\frac{\epsilon}{2}$-approximation of $\lambda_1(\Mb)$. Further, by the property of Algorithm~\ref{alg: smalleigdec} output, $-\lambda_1(\Mb) + 2a_s - \delta_k + u$ is an $\frac{\epsilon}{2}$-approximation of $\lambda_d(\Ab)$. Combining two error bound shows that we can obtain an $\epsilon$-approximation of $\lambda_d(\Ab)$.

If the algorithm stops with criteria $b_k \leq \frac{19}{20}\hat{\delta}$, we enter the part of simultaneous iteration. Algorithm takes $\tilde{\mathcal{O}}(1)$ iterations of $\Vb_{i} = \Mb\Vb_{i-1}$ and each one takes $k$ gradient oracle calls. Further finding the largest eigenvalue of $\Qb^{\top}\Mb\Qb$ does not take additional gradient calls. Therefore the total gradient complexity of this procedure is $\tilde{\mathcal{O}}\left(k\right) = \tilde{\mathcal{O}}\left(\tau_{\alpha}^{\frac{\alpha}{1+2\alpha}}\epsilon^{-\frac{\alpha}{1+2\alpha}}\right)$. Meanwhile, Lemma~\ref{lm: lazysvd} indicates that
\begin{align*}
    \lambda_k(\Mb)\leq \frac{1}{1-\frac{1}{450}}b_k \leq \frac{1}{1-\frac{1}{450}}\frac{19}{20}\hat{\delta} \leq \frac{20}{21}\lambda_1(\Mb).
\end{align*}
Let $\ub$ the eigenvector corresponding the largest eigenvalue of $\Mb$. Lemma~\ref{lm: simultaneous_iteration} shows that
\begin{align}\label{eq: QMQapproxiamteM}
\begin{aligned}
   \left(\Qb\ub\right)^{\top}\Qb^{\top}\Mb\Qb\left(\Qb^{\top}\ub\right) = & \ub^{\top}\Qb\Qb^{\top}\Mb\Qb\Qb^{\top}\ub - \ub^{\top}\Mb\ub + \lambda_1(\Mb)\\
    = & \ub^{\top}\Qb\Qb^{\top}\Mb(\Ib - \Qb\Qb^{\top})\ub + \ub^{\top}(\Ib - \Qb\Qb^{\top})\Mb\ub + \lambda_1(\Mb)\\
    \geq & \lambda_1(\Mb)(1 - 2\epsilon_4),
\end{aligned}
\end{align}
where $\epsilon_4$ is an arbitrary error that is of the order $\epsilon_4 = \mathrm{poly}(\frac{1}{\epsilon}, d)$. We set $\epsilon_4 \leq \frac{\epsilon}{8\hat{\delta}}$. Plugging this into \eqref{eq: QMQapproxiamteM} yields $\left(\Qb\ub\right)^{\top}\Qb^{\top}\Mb\Qb\left(\Qb^{\top}\ub\right) \geq \lambda_1(\Mb) - \frac{\epsilon}{2}$. Since $\|\Qb\ub\|\leq 1$, $\lambda_1\left(\Qb^{\top}\Mb\Qb\right)$ is an $\frac{\epsilon}{2}$-approximation of $\lambda_1(\Mb)$. Similar to the previous discussion, we can obtain an $\epsilon$ approximation of $\lambda_d(\Ab)$.

Besides, each iteration of the algorithm invokes a shift-and-inverse algorithm with constant multiplicative gap, which consumes $\tilde{\mathcal{O}}(1)$ gradient oracle calls. To conclude, Algorithm~\ref{alg: smalleigdec2} obtains an $\epsilon$-approximation of $\lambda_d(\Ab)$, using $\tilde{\mathcal{O}}\left(\tau_{\alpha}^{\frac{\alpha}{1+2\alpha}}\epsilon^{-\frac{\alpha}{1+2\alpha}} + \tau_{\alpha}^{\frac{\alpha}{1+2\alpha}}\epsilon^{-\frac{\alpha}{1+2\alpha}}\right) = \tilde{\mathcal{O}}\left(\tau_{\alpha}^{\frac{\alpha}{1+2\alpha}}\epsilon^{-\frac{\alpha}{1+2\alpha}}\right)$ gradient oracle calls.



\end{proof}

\subsubsection{Proof of Theorem~\ref{thm:non-convex-m}}
We give the proof of Theorem \ref{thm:non-convex-m} below.
\begin{proof}[Proof of Theorem \ref{thm:non-convex-m}]
    In each call of Algorithm \ref{alg:Cubic-search}, the problem \ref{equ:Cubicsubproblem1} is solved polylog times. We consider the inner gradient complexity of solving problem~\ref{equ:Cubicsubproblem1}. We denote $g(\y)=f_{\x_k}(\y) + \frac{H\rtemp}{4}\|\y-\x_k\|^2$ and use the eigen extractor in Algorithm~\ref{alg: eigenextrator} to extract some of the large eigenvectors, and use accelerated methods to optimize the remainder of the problem. Specifically, $\lambda_l(\nabla^2 g(\y))\le \lambda_l f(\y)+H\rtemp \le \frac{\tau_\alpha}{l^{\frac{1}{\alpha}}} + H\rtemp$. As in the proof of Theorem~\ref{thm: quadratic}, we choose $k=\tilde\Theta\left(\tau_\alpha^{\frac{\alpha}{1+2\alpha}}(H\rtemp)^{-\frac{\alpha}{1+2\alpha}} \right)$. This requires $\tilde{\mathcal{O}}\left(\tau_\alpha^{\frac{\alpha}{1+2\alpha}}(H\rtemp)^{-\frac{\alpha}{1+2\alpha}}\right)$ gradient oracle calls. Then according to the accelerated gradient-based algorithms for optimization problems, the optimization of the remainder term needs $\tilde{\mathcal{O}}\left(\sqrt{\left(\frac{\tau_\alpha}{k^{\frac{1}{\alpha}}}+H\rtemp\right)\cdot (H\rtemp)^{-1}}\right)=\tilde{\mathcal{O}}\left(\tau_\alpha^{\frac{\alpha}{1+2\alpha}}(H\rtemp)^{-\frac{\alpha}{1+2\alpha}}\right)$ gradient oracle calls. Therefore, the overall number of gradient oracle calls is $\tilde{\mathcal{O}}\left(\tau_\alpha^{\frac{\alpha}{1+2\alpha}}(H\rtemp)^{-\frac{\alpha}{1+2\alpha}}\right)$.

    Now we consider the outer iteration to find an $\left(\epsilon,\sqrt{H\epsilon}\right)$-approximate second-order stationary point for the non-convex objective. To obtain it, we need to find a sequence of $r_k$ such that $r_k \ge \sqrt{\frac{\epsilon}{H}}$ and $r_N \leq \cO\left(\sqrt{\frac{\epsilon}{H}}\right)$ where $N = \cO(H^{1/2}\Delta\epsilon^{-3/2})$. We note that $\rtemp \ge \Omega\left(\sqrt{\frac{\epsilon}{H}}\right)$. Ignoring all the logarithmic factors, the total gradient complexity is:
    \begin{equation}
        \begin{split}
            &\quad\sum_{k=1}^N \left(\tilde{\mathcal O}\left(\tau_{\alpha}^{\frac{\alpha}{1+2\alpha}}\left(H\rtemp\}\right)^{-\frac{\alpha}{1+2\alpha}} \right) + 1 \right)\\
            &\le \tilde{\mathcal{O}}\left(H^{-\frac{\alpha}{2+4\alpha}}\cdot\tau_\alpha^{\frac{\alpha}{1+2\alpha}}\epsilon^{-\frac{\alpha}{2+4\alpha}}\cdot N+N \right)\\
            &\le \tilde{\mathcal{O}}\left(H^{\frac{1+\alpha}{2+4\alpha}} \cdot \Delta \cdot\tau_\alpha^{\frac{\alpha}{1+2\alpha}} \epsilon^{-\frac{3+7\alpha}{2+4\alpha}}\right).
        \end{split}
    \end{equation}
\end{proof}

\subsubsection{Properties of Approximate Solutions}
To prove Theorem~\ref{thm:non-convex}, we first analyze the properties of the output of Algorithm~\ref{alg:Cubic-search}. Suppose $M\ge H$ is the regularization parameter. For the ease of notations, we define $h_{\x, r}(\y)$, $\bar g_{\x, M}(\y)$, $T_M(\x)$ and $\bar r_{M}$ as follows:
\begin{equation}
    \begin{split}
        h_{\x,a}(\y) &\overset{\triangle}{=} f_{\x}(\y) + \frac{a}{4}\|\y-\x\|^2,\\
        g_{\x, M}(\y) &\overset{\triangle}{=} f_\x(\y) + \frac{M}{6}\|\y-\x\|^3,\\
        T_M(\x) &\overset{\triangle}{=} \argmin_{\y\in\R^d} g_{\x, M}(\y),\\
        \bar r_M&\overset{\triangle}{=} \|T_M(\x)-\x\|.
    \end{split}
\end{equation}
Note that if we perform an exact Cubic regularization optimization step at $\x_k$, we arrive at $T_M(\x_k)$. However, the exact solution $T_M(\x_k)$ cannot be directly computed with gradient oracles. Instead, we optimize $g_{\x,M}(\y)$ to give an inexact solution. The following lemma establishes the link between optimizing $g_{\x,M}(\y)$ and optimizing $h_{\x,M\bar r_M}(\y)$:
\begin{lemma}
    $h_{\x,\bar r_M}(\y)$ is convex, and
    \begin{equation}
        T_M(\x) \in \argmin_{\y\in \R^d} h_{\x, M\bar r_M}(\y).
    \end{equation}
    \label{lem:samesolution}
\end{lemma}
\begin{proof}
    The Hessian matrix of $h_{\x,\bar r_M}$ is:
    \begin{equation}
        \nabla^2 h_{\x,M\bar r_M}(\y) = \nabla^2 f(\x) + \frac{M\bar r_M}{2}\cdot \I.
    \end{equation}
    According to Lemma~\ref{lem:Cubic1}, $\nabla^2 h_{\x,M\bar r_M}(\y)\succeq \mathbf 0$. Therefore, $h_{\x, M\bar r_M}(\y)$ is convex. Using the first-order condition of $\tx_{k+1}$, we have:
    \begin{equation}
        \nabla g_{\x,M}(T_M(\x))=\nabla f(\x) + \langle\nabla^2 f(\x), T_M(\x)-\x\rangle+\frac{M}{2}\bar r_M\cdot(T_M(\x)-\x)=\mathbf 0.
    \end{equation}
    We can verify that 
    \begin{equation}
        \nabla h_{\x, M\bar r_M}(T_M(\x)) = \nabla g_{\x, M}(T_M(\x))=\mathbf 0.
    \end{equation}
    Therefore, $T_M(\x)$ is also a minimizer of $h_{\x, M\bar r_M}$, namely
    \begin{equation}
        T_M(\x) \in \argmin_{\y\in \R^d} h_{\x, M\bar r_M}(\y).
    \end{equation}
\end{proof}

With Lemma~\ref{lem:samesolution}, we optimize $h_{\x,M\bar r_M}(\y)$ instead of $g_{\x,M}(\y)$ in each Cubic regularization step. The optimization of $h_{\x,M\bar r_M}(\y)$ can be handled by Algorithm~\ref{alg: esgd}, as $h_{\x,M\bar r_M}(\y)$ is a convex quadratic function. However, the function $h_{\x,M\bar r_M}(\y)$ cannot be directly computed and optimized, as $\bar r_M$ is unknown in prior. We propose Algorithm~\ref{alg:Cubic-search} to search for the parameter $\bar r_H$. We add an $c_2\sqrt{\frac{\epsilon}{H}}$ term to $r_{k+1}$ as a regularization, to make $h_{\x_k, Hr_{k+1}}(\y)$ $c_2\sqrt{\frac{\epsilon}{H}}$-strongly convex.

\begin{lemma}
    If $M_1\ge M_2$ and $\|\nabla f(\x)\|>0$, then $\bar r_{M_1}\le\bar r_{M_2}$.
    \label{lem:regularization}
\end{lemma}
\begin{proof}
    We prove the lemma by contradiction. $\|\nabla f(\x)\|>0$ implies that $\bar r_{M}(\x)>0$. If $\bar r_{M_1}>\bar r_{M_2}$, we have:
    \begin{equation}
        \begin{split}
            h_{\x, M_1\bar r_{M_1}}(T_{M_1}(x))&\overset{a}{\le} f_{\x}(T_{M_2}(\x)) + \frac{M_1\bar r_{M_1}}{4} \|T_{M_2}(\x)-\x\|^2\\
            &= f_{\x}(T_{M_2}(\x)) + \frac{M_1\bar r_{M_1}}{4}\|T_{M_2}(\x)-\x\|^2\\
            &= f_{\x}(T_{M_2}(\x)) + \frac{M_1\bar r_{M_2}}{4} \|T_{M_2}(\x)-\x\|^2 + \frac{M_1(\bar r_{M_1}-\bar r_{M_2})}{4} \|T_{M_2}(\x)-\x\|^2\\
            &\overset{b}{\le} f_{\x}(T_{M_1}(\x)) + \frac{M_1\bar r_{M_2}}{4} \|T_{M_1}(\x)-\x\|^2 + \frac{M_1(\bar r_{M_1}-\bar r_{M_2})}{4} \|T_{M_2}(\x)-\x\|^2\\
            &\overset{c}{<} f_{\x}(T_{M_1}(\x)) + \frac{M_1\bar r_{M_2}}{4} \|T_{M_1}(\x)-\x\|^2 + \frac{M_1(\bar r_{M_1}-\bar r_{M_2})}{4} \|T_{M_1}(\x)-\x\|^2\\
            &= h_{\x, M_1\bar r_{M_1}}(T_{M_1}(x)),
        \end{split}
        \label{equ:contradiction}
    \end{equation}
    which is a contradiction. In \eqref{equ:contradiction}, $\overset{a}{\le}$ uses $T_{M_1}(\x) \in \argmin_{\y\in \R^d} h_{\x, M_1\bar r_{M_1}}(\y)$, $\overset{b}{\le}$ uses $T_{M_2}(\x) \in \argmin_{\y\in \R^d} h_{\x, M_2\bar r_{M_1}}(\y)$, and $\overset{c}{<}$ uses $\bar r_{M}(\x) > 0$.
\end{proof}
 Define $\tx_{k+1} = \argmin_{\y\in \R^d} h_{\x_k, Hu_{k+1}} $, and $\tilde r_{k+1} = \left\|\tx_{k+1}-\x_k\right\|$. We note that $\tx_{k+1} = T_{\frac{Hu_{k+1}}{\tilde r_{k+1}}}(\x_k)$. We have the following Lemmas:
\begin{lemma}
    In Algorithm~\ref{alg:cubicesgd}, if $r_{k+1} \ge (4c_1+2c_2)\cdot\sqrt{\frac{\epsilon}{H}}$ and $\errorC < \frac{c_2^3}{2}\cdot \sqrt{\frac{\epsilon^3}{H}}$, then $f(\x_{k+1})\le f(\x_k) - \left(\frac{8c_1^3+24c_1^2c_2+12c_1^2c_2+c_2^3}{12(4c_1+2c_2)^3}\right)\cdot Hr_{k+1}^3$. 
\end{lemma}
\begin{proof}


    The objective function $h_{\x_k,Hu_{k+1}}$ in Algorithm~\ref{alg:Cubic-search} is $c_2\sqrt{\epsilon H}$-strongly convex. Therefore, to ensure that $\x_{k+1}$ is an $\errorC$-approximated solution, we have $|r_{k+1}-\tilde r_{k+1}|\le \|\tx_{k+1}-\x_{k+1}\|\le c_2\sqrt{\frac{\epsilon}{H}}$. Combined with the assumption that $r_{k+1} \ge (4c_1+2c_2)\sqrt{\frac{\epsilon}{H}}$, we have $\frac{4c_1+c_2}{4c_1+2c_2} r_{k+1} \le \tilde r_{k+1}\le \frac{4c_1+3c_2}{4c_1+2c_2} r_{k+1}$.

    Finally, using Lemma~\ref{lem:Cubic4}, we have:
    \begin{equation}
        \begin{split}
            &\quad f(\x_{k+1})-f(\x_k)\\
            &= f(\x_{k+1})-g_{\x,\frac{Hu_{k+1}}{\tilde r_{k+1}}}(\x_{k+1}) + g_{\x,\frac{Hu_{k+1}}{\tilde r_{k+1}}}(\x_{k+1}) - g_{\x,\frac{Hu_{k+1}}{\tilde r_{k+1}}}(\tx_{k+1}) + g_{\x,\frac{Hu_{k+1}}{\tilde r_{k+1}}}(\tx_{k+1}) - f(\x_k)\\ 
            &\leq \errorC + g_{\x,\frac{Hu_{k+1}}{\tilde r_{k+1}}}(\tx_{k+1}) - f(\x_k)\\
            &\leq \frac{c_2^3}{4}\cdot\sqrt{\frac{\epsilon^3}{H}} - \frac{1}{12}\cdot \frac{Hu_{k+1}}{\tilde r_{k+1}}\cdot \tilde r_{k+1}^3\\
            &\leq \frac{c_2^3}{4\cdot (4c_1+2c_2)^3}\cdot Hr_{k+1}^3  - \frac{1}{6}\cdot \frac{(4c_1+c_2)^2}{(4c_1+2c_2)^2}\cdot Hr_{k+1}^3\\
            &=-\left(\frac{8c_1^3+24c_1^2c_2+12c_1^2c_2+c_2^3}{12(4c_1+2c_2)^3}\right)\cdot Hr_{k+1}^3.
        \end{split}
    \end{equation}
\end{proof}

\begin{lemma}
    In Algorithm~\ref{alg:cubicesgd}, if $r_{k+1} \ge (4c_1+2c_2)\cdot\sqrt{\frac{\epsilon}{H}}$ and $\errorC < 2^{-c}\min\left\{\frac{c_2^3}{2}\cdot \sqrt{\frac{\epsilon^3}{H}}, \frac{(4c_1+2c_2)^4\epsilon^{2.5}H^{0.5}}{2\tau_\alpha^2}\right\}$, we have:
    \begin{equation}
        \|\nabla f(\x_{k+1})\| < \left(1+\frac{(4c_1+3c_2)(12c_1+7c_2)}{2(4c_1+2c_2)^2}\right)\cdot Hr_{k+1}^2,
    \end{equation}
    \begin{equation}
        \nabla^2 f(\x_{k+1})\succeq -\left(1+\frac{c_2}{4c_1+2c_2}\right)\cdot Hr_{k+1}\I.
    \end{equation}
\end{lemma}

\begin{proof}
    We have $\|\tx_{k+1}-\tx_k\|\le \frac{(4c_1+2c_2)^2\epsilon}{\tau_\alpha}$, $\|\tx_{k+1}-\tx_k\|\le c_2 \sqrt{\frac{\epsilon}{H}}$ and $\frac{4c_1+c_2}{4c_1+2c_2} r_{k+1} \le \tilde r_{k+1}\le \frac{4c_1+3c_2}{4c_1+2c_2} r_{k+1}$. With Lemmas~\ref{lem:Cubic1} and \ref{lem:Cubic3}, we have:
    \begin{equation}
        \nabla^2 f(\tx_{k+1}) \succeq - \frac{Hu_{k+1}}{2\tilde r_{k+1}}\cdot \tilde r_{k+1}\I\succeq -Hr_{k+1}\I,
    \end{equation}
    \begin{equation}
        \|\nabla f(\tx_{k+1})\|\le \frac{H}{2}\left(1+\frac{u_{k+1}}{\tilde r_{k+1}} \right)\tilde r_{k+1}^2\le \frac{H}{2} \frac{(4c_1+3c_2)(12c_1+7c_2)}{(4c_1+2c_2)^2}\cdot r_{k+1}^2.
    \end{equation}
    Using the Hessian-Lipschitz property of $f$, we have:
    \begin{equation}
        \nabla^2 f(\x_{k+1})\succeq \nabla^2 f(\tilde \x_{k+1})- H\|\tx_{k+1}-\x_{k+1}\|\I\succeq -\left(1+\frac{c_2}{4c_1+2c_2}\right)\cdot Hr_{k+1}\I.
    \end{equation}
    Using the Gradient-Lipschitz property of $f$, we have:
    \begin{equation}
        \|\nabla f(\x_{k+1})\|\le \left(1+\frac{(4c_1+3c_2)(12c_1+7c_2)}{2(4c_1+2c_2)^2}\right)\cdot Hr_{k+1}^2.
    \end{equation}
\end{proof}

\begin{lemma}
    In Algorithm~\ref{alg:cubicesgd}, if $r_{k+1} < (4c_1+2c_2)\cdot\sqrt{\frac{\epsilon}{H}}$ and $\errorC < 2^{-c}\min\left\{\frac{c_2^3}{2}\cdot \sqrt{\frac{\epsilon^3}{H}}, \frac{(4c_1+2c_2)^4\epsilon^{2.5}H^{0.5}}{2\tau_\alpha^2}\right\}$, we have:
    \begin{equation}
        \|\nabla f(\x_{k+1})\| \le \frac{68c_1^2+79c_1c_2+23c_2^2}{2}\epsilon
    \end{equation}
    \begin{equation}
        \nabla^2 f(\x_{k+1})\succeq -\frac{4c_1+3c_2}{2}\sqrt{H\epsilon}
    \end{equation}
\end{lemma}

\begin{proof}
    We have $\|\tx_{k+1}-\tx_k\|\le \frac{(4c_1+2c_2)^2\epsilon}{\tau_\alpha}$ and $\|\tx_{k+1}-\tx_k\|\le c_2 \sqrt{\frac{\epsilon}{H}}$. If $r_{k+1} < (4c_1+2c_2)\cdot\sqrt{\frac{\epsilon}{H}}$, then $\tilde r_{k+1}<(4c_1+3c_2)\cdot\sqrt{\frac{\epsilon}{H}}$. By Lemma ~\ref{lem:regularization}, we have $\bar r_H = \|T_H(\x_k)-\x_k\|\le \tilde r_{k+1}\le(4c_1+3c_2)\cdot\sqrt{\frac{\epsilon}{H}}$. Using Lemma~\ref{lem:Cubic1}, we have:
    \begin{equation}
        \nabla^2 f(T_H(\x_k))\succeq -\frac{H\bar r_H}{2}\I \succeq -\frac{4c_1+3c_2}{2}\sqrt{H\epsilon}.
    \end{equation}
    Using the Hessian-Lipschitz property of $f$, we have:
    \begin{equation}
        \nabla^2 f(\x_{k+1})\succeq \nabla^2 f(T_H(\x_k))-(\bar r_H+r_{k+1})H \succeq -\frac{16c_1+11c_2}{2}\sqrt{H\epsilon}.
    \end{equation}
    Using the first-order condition at $\tilde \x_{k+1}$, we have:
    \begin{equation}
        \nabla f(\x_k) + \nabla^2 f(\tx_k)\cdot(\tx_{k+1}-\x_k) + \frac{Hu_{k+1}}{2}(\tx_{k+1}-\x_k)=\mathbf 0.
        \label{equ:SSPpf1}
    \end{equation};
    
    Using the Hessian-Lipschitz property of $f$, we have:
    \begin{equation}
        \|\nabla f(\tx_{k+1})-\nabla f(\x_k)-\nabla^2 f(\x_k)\cdot(\tx_{k+1}-\x_k)\|\le \frac{1}{2}H\|\tx_{k+1}-\x_k\|^2.
        \label{equ:SSPpf2}
    \end{equation}
    Combining \eqref{equ:SSPpf1} and \eqref{equ:SSPpf2}, we have:
    \begin{equation}
        \|\nabla f(\tx_{k+1})\| \le \frac{H\tilde r_{k+1}(\tilde r_{k+1}+u_{k+1})}{2}\le \frac{(4c_1+3c_2)(9c_1+5c_2)}{2}\epsilon.
    \end{equation}
    Using $\|\tx_{k+1}-\tx_k\|\le \frac{(4c_1+2c_2)^2\epsilon}{L}$ and the gradient-Lipschitz property of $f$, we have:
    \begin{equation}
        \|\nabla f(\x_{k+1})\| \le \frac{(4c_1+3c_2)(9c_1+5c_2)+2(4c_1+2c_2)^2}{2}\epsilon = \frac{68c_1^2+79c_1c_2+23c_2^2}{2}\epsilon.
    \end{equation}
\end{proof}

\subsubsection{Useful Results in \cite{nesterov_cubic_2006}}
In this subsection, we present some  results in \cite{nesterov_cubic_2006}, which we use in our analysis.

\begin{lemma}[Proposition 1 of \cite{nesterov_cubic_2006}]
    \begin{equation}
        \nabla^2 f(\x) + \frac{M\bar r_M}{2}\I\succeq \mathbf 0.
    \end{equation} 
    \label{lem:Cubic1}
\end{lemma}

\begin{lemma}[Lemma 2 of \cite{nesterov_cubic_2006}]
    For any $k\ge 0$, we have
    \begin{equation}
        \langle\nabla f(\x), \x_k-T_M(\x)\rangle \ge 0. 
    \end{equation}
    \label{lem:Cubic2}
\end{lemma}

\begin{lemma}[Lemma 3 of \cite{nesterov_cubic_2006}]
     For any $k\ge 0$, we have
    \begin{equation}
        \|\nabla f(T_M(\x))\|\le \frac{H+M}{2}\bar r_{M}^2.
    \end{equation}
    \label{lem:Cubic3}
\end{lemma}

\begin{lemma}[Lemma 4 of \cite{nesterov_cubic_2006}]
    \begin{equation}
        f(\x)- g_{\x, M}(T_M(\x))\ge f(\x)-f(T_M(\x))\ge \frac{M}{12}\bar {r}_M^3.
    \end{equation}
    \label{lem:Cubic4}
\end{lemma}

\section{Proofs in Section~\ref{sec:ipm}}

\subsection{Additional Supplements in Weighted Path Finding}\label{appsubsec:weighted-path}

\begin{proof}[Proof of Lemma~\ref{thm:centrality-change}]
    This result is very similar to Lemma 14, 15, 17 of \cite{lsJournal19}. Here we emphasis the "global" differences:
    \begin{itemize}
        \item The original mixed norm $\|\cdot\|_{w+\infty}$ should be replaced by $\|\cdot\|_{w+\square}$.
        \item The matrices in \cite{lsJournal19} are mostly diagonal matrices, and in the proof their vector form is used for convenience. However, the matrices we deal with are mostly $2$-block matrices so we need to keep the order of matrix operations. 
        
        \item The norm for $\left\|\sqrt{\phi''}\Delta_x\right\|_\infty$ and $\left\|\sqrt{\frac{\phi''_t}{\phi''_0}}\right\|_\infty$ should be replaced by $\left\|\sqrt{\Phi''}\Delta_x\right\|_\infty$ and $\left\|\Phi_t^{1/2}\Phi_0^{-1/2}\right\|_\infty$ respectively. We still keep the infinity norm for measuring distance between weight functions such as $\left\|\frac{g_s-g_t}{g_t}\right\|_\infty$.
    \end{itemize}
\end{proof}

The modified algorithm \texttt{CenteringStep}. Given Lemma~\ref{thm:centrality-change}, we can directly follow the original proof(Theorem 19 of \cite{lsJournal19}). The constant in $\log(\cdot)$ changes from 36 to 72 since in the \emph{chasing game} procedure, there exists a less than $\sqrt 2$ multiplicative factor between $\|\cdot\|_{\xb+\infty}$ and $\|\cdot\|_{\xb+\square}$.
\begin{algorithm}[H]\label{Alg: centering-step}

\caption{$(\next{\xb},\next{\vWeight})=\text{CenteringStep}(\xb,\wb,K)$}

\SetAlgoLined

\textbf{Input:} $K\leq \frac{1}{16c_k}$

Let $R=\frac{K}{48c_{k}\log(72c_{1}c_{s}c_{k}d)}$, $\delta=\delta_{t}(\xb,\vWeight)$
, $\epsilon=\frac{1}{2c_{k}}$, $\mu=\frac{\epsilon}{12R}$, function  $\Phi(\xb)=\sum_{i=1}^{d}(e^{\mu \xb_i}+e^{-\mu \xb_i})$.

\textbf{Input:} $\xb\in\dInterior, \wb\in\Rd_{>0}$ such that $
\delta\leq R\enspace\text{ and }\enspace\Phi_{\mu}(\log(\vg(\xb))-\log(\wb))\le72c_{1}c_{s}c_{k}d$.

$\next{\xb}=\xb-\Phi''(\xb)^{-1/2}\mProj_{\xb,\wb}\mw^{-1}\Phi''(\xb)^{-1/2}\left({t\vc-\mw\phi'(\xb)}\right).$

Let $U=\{\xb\in\Rm~|~\mixedNorm{\xb}{\vWeight}\leq(1-\frac{6}{7c_{k}})\delta\}$.

Find $\zb$ such that $\normInf{\zb-\log(\vg(\next{\xb}))}\leq R$.

$\next{\wb}=\exp\left(\log(\wb)+\argmin_{\ub\in(1+\epsilon)U}\left\langle \nabla\Phi_{\mu}(\zb-\log(\wb)),\ub\right\rangle \right)$. 

\textbf{Output:} $\next\xb\in\dInterior, \next\wb\in\Rd_{>0}$ such that $
\delta_t(\next\xb, \next\wb)\leq (1-\frac{1}{4c_k})\delta_{t}(\xb,\vWeight)$ and  $\Phi_{\mu}(\log(\vg(\xb))-\log(\wb))\le72c_{1}c_{s}c_{k}d$.

\end{algorithm}

Proof of Theorem~\ref{thm:r-iteration} follows from Theorem 40-42 in \cite{lsJournal19}.
\subsection{Additional Supplements in Inverse Maintenance}\label{appsub:inverse-maintain}

\begin{proof}[Proof of Theorem~\ref{thm:iteration_solver}]
    The algorithm is a modification of Algorithm 3 in \cite{lee2015efficient}. The main difference is that each time we sample a block instead of sampling a row. For completeness, we argue that through some modification, we can safely leverage those technical results in \cite{lee2015efficient}.
    \begin{itemize}
        \item The condition for Lemma 14 in \cite{lee2015efficient} is replaced by $\norm{\log(\mm^{(k)}/\mm^{(k-1)})}_{\infty}\leq\epsilon $, and the result is modified to \emph{At each time $k$, we can inductively construct non-degenerate $\mc_i^{(k)}\in\R^{2\times2}$ such that ${\mc_i^{(k)}}^\top\mc_i^{(k)}=\mm_i^{(k)}$  for $i\in[d/2]$ and $\|\log\sigma(\mc^{(k)}\ma)-\log\sigma(\mc^{(k)}A)\|_{\sigma(\mc^{(k)}\ma)}\le e^{\epsilon}\norm{\log(\mm^{(k)}/\mm^{(k-1)})}_{\vLever^{(k)}}$.} Using Remark~\ref{remark:decomposition-invariant}  and diagonalizing $ {\mc_i^{(k)}}^{-\top}\mm_i^{(k+1)}{\mc_i^{(k)}}^{-1}$, we can reduce the case to Lemma 14 in \cite{lee2015efficient}.
        
        \item Leverage Score Sampling (Lemma 5 in\cite{lee2015efficient}) holds if $u_i\geq 2\max{\sigma_{2i-1},\sigma_{2i}}$ for $i\in[d/2]$ for following matrix concentration results in the original proof(Lemma 4 in \cite{cohen2015uniform}).
        \item The complexity in Theorem 9 in \cite{lee2015efficient} will not increase by more than a constant factor if $\md$ is replaced by a 2-block matrix $\mm$.
    \end{itemize}
    Given the theoretical results, our proof consists of three parts.

    \textbf{Correctness:} It suffices to show that $\mq^{(k)}\approx_{O(1)} \ma^{^\top}\mm^{(k)}\ma$. Note that in each iteration $ k$ we maintain
    $$
    \mm^{(old)}\approx_{0.2} \mm^{(k)}, \mSigma^{old}\approx_{0.2}\mSigma^{(k)},\text{where } \mSigma\defeq\diag(\sigma).
    $$
    Thus, in each iteration $k$, we see that the sample probability of each $\mm_i^{(k)}$ was chosen to satisfy the assumptions of Lemma 5 in \cite{lee2015efficient}. Hence we have $\ma^\top\mh^{(k)}\ma\approx_{0.1}\ma^\top\mm^{(old)}\ma\approx_{0.2}\ma^\top\mm^{(k)}\ma$.

    \textbf{Update Times:} Following Lemma 15 of \cite{lee2015efficient}, we prove that the Algorithm~\ref{alg:sparseframework} only changes $O(l^2)$ blocks (in expectation) in total.

    Suppose we resample the $i$-th block at time $k_2$, and the last resampling time for block $i$ is $k_1$, we can see that the probability of an actual change of the matrix $\ma^\top\mh\ma$ is $O(\tau_i^{(k)}\log(d))$.

    Observe that whenever we re-sampled the $i$-th block, either $\tau_i^{(k)}$ or $\mm_i$ has changed by more than a multiplicative constant. If $ \mm_i$ changes by more than a multiplicative constant, then the proof is similar to the original case (Lemma 15 of \cite{lee2015efficient}). If $\tau_i^{(k)}$ changes by more than a multiplicative constant. WLOG, we assume $\sigma({\mc^{(k_1)}}^{1/2}\ma)_{2i-1}\ge \sigma({\mc^{(k_1)}}^{1/2}\ma)_{2i}$. If $\sigma({\mc^{(k)}}^{1/2}\ma)_{2i-1}$ changes by more than a multiplicative constant, then following the original proof, we have
    $$
    \begin{aligned}
        &\sum_{k=k_1}^{k_2-1}\sigma({\mc^{(k)}}^{1/2}\ma)_{2i-1}\left( \log \sigma({\mc^{(k+1)}}^{1/2}\ma)_{2i-1}- \log \sigma({\mc^{(k)}}^{1/2}\ma)_{2i-1}\right)^2\\
        =&\Omega\left(\frac{\sigma^{(k_1)}_{2i-1}}{l}\right)
        =\Omega\left(\frac{\tau^{(k_1)}_i}{l}\right)
        =\Omega\left(\frac{\tau^{(k_2)}_i}{l}\right).\\
    \end{aligned}
    $$
    If it is not the case, we must have $\sigma({\mc^{(k_2)}}^{1/2}\ma)_{2i-1}\le \sigma({\mc^{(k_2)}}^{1/2}\ma)_{2i}$ and $\sigma({\mc^{(k)}}^{1/2}\ma)_{2i}$ changes by more than a multiplicative constant. Then similarly, we have
    $$
    \begin{aligned}
        &\sum_{k=k_1}^{k_2-1}\sigma({\mc^{(k)}}^{1/2}\ma)_{2i}\left( \log \sigma({\mc^{(k+1)}}^{1/2}\ma)_{2i}- \log \sigma({\mc^{(k)}}^{1/2}\ma)_{2i}\right)^2\\
        =&\Omega\left(\frac{\sigma^{(k_2)}_{2i}}{l}\right)
        =\Omega\left(\frac{\tau^{(k_2)}_i}{l}\right).\\
    \end{aligned}
    $$
    
    Both the cases lead to

    $$
    \begin{aligned}
        &\sum_{k=k_1}^{k_2-1}\sigma({\mc^{(k)}}^{1/2}\ma)_{2i-1}\left( \log \sigma({\mc^{(k+1)}}^{1/2}\ma)_{2i-1}- \log \sigma({\mc^{(k)}}^{1/2}\ma)_{2i-1}\right)^2\\
        +&\sum_{k=k_1}^{k_2-1}\sigma({\mc^{(k)}}^{1/2}\ma)_{2i}\left( \log \sigma({\mc^{(k+1)}}^{1/2}\ma)_{2i}- \log \sigma({\mc^{(k)}}^{1/2})_{2i}\right)^2=\Omega\left(\frac{\tau^{(k_2)}_i}{l}\right).
    \end{aligned}
    $$
    Hence similar to the original proof, the total change is less than $\tilde\cO(l^2)$.
    
    \textbf{Total Complexity:} Now the problem is reduced to low rank update case, the derivation is similar to the proof of Theorem 13 of \cite{lee2015efficient}.
\end{proof}

\begin{algorithm}[H]
\caption{$\text{InverseMaintainer}(\mm)$}

\label{alg:sparseframework}

\SetAlgoLined

\textbf{Input: }Initial $\mm^{(0)}$.

Set $\mm^{(old)}:=\mm^{(0)}$ and $\gamma\defeq2000c_{s}\log r$
where $c_{s}$ defined in Lemma 5 of \cite{lee2015efficient}.

Use Lemma 6 of \cite{lee2015efficient} to find $\sigma^{(apr)}$
such that $0.99\sigma_{i}^{(apr)}\leq\sigma({\mm^{(0)}}^{1/2}\ma)_{i}\leq1.01\sigma_{i}^{(apr)}$.

For each $i\in[d/2]$ : let $ \tau_i^{(apr)}:=\max\{\sigma_{2i-1}^{(apr)}, \sigma_{2i}^{(apr)}\}$

For each $i\in[d/2]$ : let $\mh_{i}^{(0)}:=\mm_{i}/\min\{1,\gamma\cdot\tau_{i}^{(apr)}\}$
with probability $\min\{1,\gamma\cdot\tau_{i}^{(apr)}\}$ 

$\quad$and is set to $0$ otherwise.

$\mq^{(0)}\defeq\ma^{T}\mh^{(0)}\ma$.

Let $\mk^{(0)}$ be an approximate inverse of $\mq^{(0)}$ computed
using Theorem 9 of \cite{lee2015efficient}.

\textbf{Output: }A $\tilde{O}(r^{2}+\nnz(\ma))$-time linear solver
for $\ma^{T}\mm^{(0)}\ma$ (using Theorem 10 of \cite{lee2015efficient}
on $\mk^{(0)}$).

\For{each round $k\in[l]$}{

\textbf{Input:} Current $\mm^{(k)}$.

Compute $\mc^{(k)}$ as in Theorem~\ref{thm:iteration_solver}

Use Lemma 6 of \cite{lee2015efficient} and the solver $\solver^{(k-1)}$ to find $\sigma^{(apr)}$ such that 

$\quad$$0.99\sigma_{i}^{(apr)}\leq\sigma(\mc^{(k)}\ma)_{i}\leq1.01\sigma_{i}^{(apr)}$.

\For{each block $i\in[d/2]$}{

$ \tau_i^{(apr)}:=\max\{\sigma_{2i-1}^{(apr)}, \sigma_{2i}^{(apr)}\}$

\uIf{ either $0.9\tau_{i}^{(old)}\leq\tau_{i}^{(apr)}\leq1.1\tau_{i}^{(old)}$
or $0.9\mm_{i}^{(old)}\preceq \mm_{i}^{(k)}\preceq 1.1\mm_{i}^{(old)}$ is violated}{

$\mm_{i}^{(old)}:=\mm_{i}^{(k)}$.

$\tau_{i}^{(old)}:=\tau_{i}^{(apr)}$.

$\mh_{i}^{(k)}:=\mm_{i}^{(k)}/\min\{1,\gamma\cdot\tau_{i}^{(apr)}\}$
with probability $\min\{1,\gamma\cdot\tau_{i}^{(apr)}\}$ 

$\enspace\enspace\enspace$and is set to $0$ otherwise.

}\Else{

$\mh_{i}^{(k)}:=\mh_{i}^{(k-1)}.$

}

}

$\mq^{(k)}\defeq\ma^{^\top}\mh^{(k)}\ma $, satisfies $ \mq^{(k)}\approx_{O(1)}\ma^{^\top}\mm^{(k)}\ma$ using Theorem~\ref{thm:iteration_solver}.

Let $\mk^{(k)}$ be an approximate inverse of $\mq^{(k)}$ computed
using Theorem 9 of \cite{lee2015efficient}.

\textbf{Output:} A linear $\tilde{O}(r^{2}+\nnz(\ma))$-time solver $\solver^{(k)}$
for $\ma^{T}\mm^{(k)}\ma$ (using Theorem 10 of \cite{lee2015efficient}
on $\mk^{(k)}$).

}

\end{algorithm}

\section{Experiment Details and Additional Experiments}\label{app:experiment}
We conduct basic experiments to verify such a phenomenon: even though training the same linear model or neural network, different data distribution has a great effect on the convergence speed. Specifically, we conduct linear regression and neural network training on two benchmark datasets: MNIST \cite{lecun1998gradient} and CIFAR10 \cite{krizhevsky2009learning}. For linear regression, we use mean squared error loss with $L_2$ regularization (set the hyper-parameter weight decay to 5E-4) as the objective function and use full-batch Accelerated Gradient Descent to update.  For neural network training, we use a common neural network structure in solving picture classification tasks, ResNet \cite{he2016deep}. We use the most basic ResNet18 and use SGD as an optimizer.  

The results are shown in Fig.\ref{EXP1}, which indicate that for the same linear regression problem or neural network training, the convergence speed on MNIST is significantly faster than on CIFAR10. 
Moreover, we compute the eigenvalues of the Neural Tangent Kernal (NTK) \cite{jacot2018neural} of MNIST dataset. The result is shown in Fig. \ref{fig4}, indicating that the eigenvalues decrease fast. Another similar result
in \cite{sagun2016eigenvalues} shows the eigenvalues of a three-layer neural network on MNIST, which shows the same trend.
\begin{figure*}[t]
	\centering
        \vspace{-0.35cm}
        \subfigtopskip = 10pt
        \subfigcapskip = -5pt
	\subfigure[]{
        \begin{minipage}[t]{0.5\linewidth}
			\centering
			\includegraphics[width=3in]{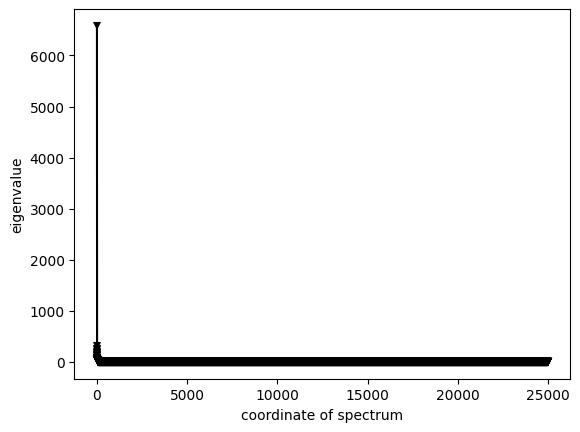}\\			\vspace{0.02cm}
		\end{minipage}%
	}%
	\subfigure[]{
		\begin{minipage}[t]{0.5\linewidth}
			\centering
			\includegraphics[width=3in]{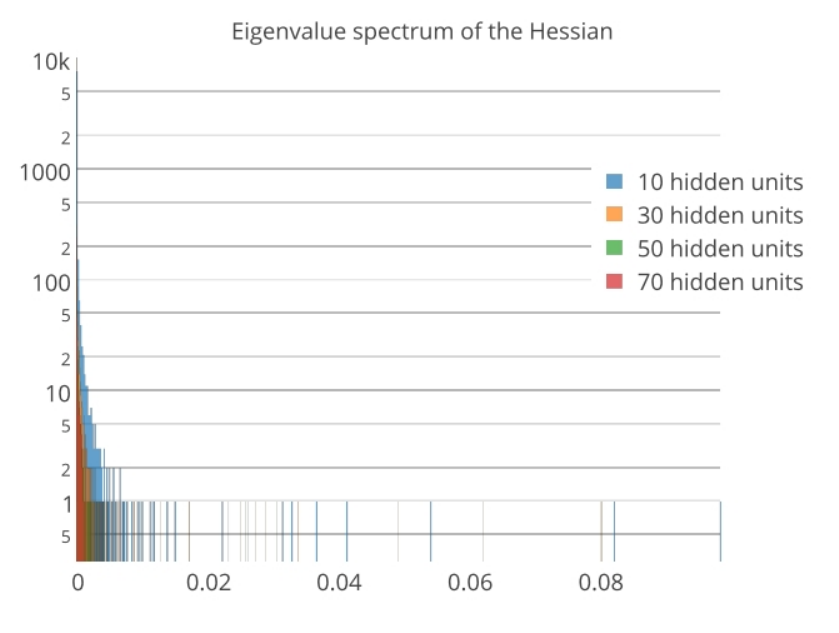}\\
			\vspace{0.02cm}
		\end{minipage}%
	}%
	\centering
	\caption{(a) The eigenvalues of the NTK matrix on MNIST. (b) The eigenvalues of a three-layer neural network on MNIST. (b) is taken directly from \cite{sagun2016eigenvalues}. }
	\vspace{-0.2cm}
	\label{fig4}
\end{figure*}
\end{document}